\documentclass[twoside,11pt]{article}

\usepackage{jmlr2e}
\usepackage[caption=false]{subfig}
\usepackage{amsmath,amssymb,amsbsy,mathtools,xfrac,nicefrac}
\usepackage{dsfont,bm,fixmath}
\usepackage{tabulary,multirow}
\usepackage[usenames]{color}
\usepackage{mathrsfs}
\usepackage{cite}
\usepackage{tikz}
\usetikzlibrary{fit,positioning,calc}
\usetikzlibrary{shapes,arrows}
\usepackage{pgfplots}
\usepackage{pgfplotstable}
\usepackage{IEEEtrantools}

\newcommand{\maximum}{\text{max}}
\newcommand{\minimum}{\text{min}}

\newcommand{\Ba}{\B{a}}
\newcommand{\Bb}{\B{b}}

\newcommand{\Be}{\B{e}}

\newcommand{\Bp}{\B{p}}
\newcommand{\Bq}{\B{q}}
\newcommand{\Br}{\B{r}}

\newcommand{\Bt}{\B{t}}
\newcommand{\Bu}{\B{u}}
\newcommand{\Bv}{\B{v}}
\newcommand{\Bw}{\B{w}}
\newcommand{\Bx}{\B{x}}
\newcommand{\By}{\B{y}}

\newcommand{\BA}{\B{A}}
\newcommand{\BB}{\B{B}}

\newcommand{\BF}{\B{F}}

\newcommand{\BW}{\B{W}}

\newcommand{\Blambda}{\bm{\lambda}}

\newcommand{\Btau}{\bm{\tau}}

\newcommand{\BDelta}{\bm{\Delta}}

\newcommand{\BLambda}{\bm{\Lambda}}

\newcommand{\BPi}{\bm{\Pi}}

\newcommand{\BPhi}{\bm{\Phi}}
\newcommand{\BPsi}{\bm{\Psi}}

\newcommand{\myid}{\mathbf{I}} 
\newcommand{\myone}{\mathbf{1}} 

\DeclareMathAlphabet{\mathbit}{OML}{cmr}{bx}{it}
\DeclareMathAlphabet{\mathgoth}{U}{ygoth}{m}{n}
\DeclareMathAlphabet{\mathfrak}{U}{yfrak}{m}{n}
\DeclareMathAlphabet{\mathswab}{U}{yswab}{m}{n}

\newcommand{\B}[1]{\mathbit{#1}}

\newtheorem{myrmk}{Remark}

\DeclareMathOperator*{\argmin}{argmin}
\DeclareMathOperator*{\argmax}{argmax}

\DeclareMathOperator{\Prob}{\mathsf{Pr}}
\DeclareMathOperator{\Exp}{\mathsf{E}}

\newcommand{\diffd}{{\operatorname{d}}}
\DeclareMathOperator{\diag}{diag}

\DeclareMathOperator{\Transpose}{T}
\DeclareMathOperator{\Hermitian}{\dagger}
\newcommand{\Tr}{{\Transpose}}
\newcommand{\He}{{\Hermitian}}

\DeclareMathOperator{\var}{var}

\let\originalleft\left
\let\originalright\right
\renewcommand{\left}{\mathopen{}\mathclose\bgroup\originalleft}
\renewcommand{\right}{\aftergroup\egroup\originalright}
\renewcommand{\mathcal}{\mathscr}

\begin{document}
\title{Locally Differentially-Private Randomized Response\\for Discrete Distribution Learning}

\author{\name Adriano Pastore \email adriano.pastore@cttc.cat \\
       \addr Department of Statistical Inference for Communications and Positioning \\
       Centre Tecnol\`{o}gic de Telecomunicacions de Catalunya (CTTC/CERCA) \\
       Castelldefels, Barcelona, Spain
       \AND
       \name Michael Gastpar \email michael.gastpar@epfl.ch \\
       \addr School for Computer and Communication Sciences \\
       Ecole Polytechnique F\'ed\'erale de Lausanne \\
       Lausanne, Switzerland}

\editor{}

\maketitle

\begin{abstract}%

We consider a setup in which confidential i.i.d.\ samples $X_1,\dotsc,X_n$ from an unknown finite-support distribution~$\Bp$ are passed through $n$ copies of a discrete privatization channel (a.k.a.~\emph{mechanism}) producing outputs $Y_1,\dotsc,Y_n$. The channel law guarantees a local differential privacy of $\epsilon$. Subject to a prescribed privacy level $\epsilon$, the optimal channel should be designed such that an estimate of the source distribution based on the channel outputs $Y_1,\dotsc,Y_n$ converges as fast as possible to the exact value $\Bp$. For this purpose we study the convergence to zero of three distribution distance metrics: $f$-divergence, mean-squared error and total variation. We derive the respective normalized first-order terms of convergence (as $n \to \infty$), which for a given target privacy $\epsilon$ represent a rule-of-thumb factor by which the sample size must be augmented so as to achieve the same estimation accuracy as that of a non-randomizing channel. We formulate the privacy--fidelity trade-off problem as being that of minimizing said first-order term under a privacy constraint~$\epsilon$. We further identify a scalar quantity that captures the essence of this trade-off, and prove bounds and data-processing inequalities on this quantity. For some specific instances of the privacy--fidelity trade-off problem, we derive inner and outer bounds on the optimal trade-off curve.
\end{abstract}

\section{Introduction}

In the statistical analysis of privacy-sensitive data, the key challenge consists in randomizing database queries or post-processing (sanitizing) the dataset so as to render inferences about the {\em data} (values or labels) as difficult as possible while at the same time preserving the usefulness of the data for estimating {\em parameters} of the underlying distribution. The inherent trade-off between the conflicting goals of privacy and utility arises in a broad variety of situations, notably in medical surveys, customer profiling, consumer studies, population census, opinion polls or surveys in social sciences.

Specifically, we will be concerned with the {\em randomized response} (RR) technique. An early inspiration for the basic setup, which is depicted in Figure~\ref{fig:system} further below, dates back to~\citet{Wa65}: a common task in social sciences is to conduct surveys in which some answers might be stigmatizing (e.g., questions on drug use, sexual behavior, etc.). To overcome the respondent's potential reticence to answering faithfully, Warner proposed to perturb the interviewee's answers by having him/her secretly randomize the answers, in such way that not even the interviewer would learn the true answer.

In more recent years, a substantial body of work has developed around the celebrated notion of {\em differential privacy} (DP) introduced by~\citet{Dw06,DwMcNiSm06}. While the original purpose of DP was to provide strong privacy guarantees against a resourceful attacker who can access queries to a central database, \citet{KaLeNiRaSm11} and \citet{DuJoWa12} considered a decentralized privatization model referred to as {\em local privacy}, in which each data point is independently perturbed by a randomizing mechanism. Warner's RR scheme can be inscribed in this local privacy framework.

Notable other proposals in the spirit of this local, non-interactive privatization mechanism include~\citet{AgSr00}, in which the authors propose a procedure to build a decision-tree classifier from perturbed training data that performs close to a classifier built from the original data. \citet{WaZh10} study the exponential rate at which certain estimates of continuous distributions from noisy samples concentrate in a small ball (in mean-square loss or Kolmogorov-Smirnov distance) centered around the true distribution. \citet{DuJoWa12} study minimax learning rates of distribution parameters from privatized samples from the viewpoint of statistics.
In the context of locally private hypothesis testing, \citet{GaRo17} study privatized chi-square tests for goodness of fit and independence testing.
Also worth mentioning as a new research direction is the work by \citet{HuKaChSaRa17}, who introduced the concept of {\em generative adversarial privacy} which is inspired by the recent invention of generative adversarial networks: the confidential dataset (in a \emph{non-local} privacy context) is used to train a generative neural network (against an adversary), which then creates plausible new data samples that emulate the original distribution without disclosing any sample from the training set.

For discrete and finitely supported distributions, the work of~\citet{KaOhVi14} has shown that a finite class of privatization mechanisms called {\em staircase mechanisms} are optimal among all $\epsilon$-private mechanisms for a constrained $f$-divergence maximization problem related to private hypothesis testing. In particular, they show that a simple mechanism, which we call the {\em step mechanism}---defined in Equation~\eqref{W_star} in the present article---which they refer to simply as {\em randomized response} (RR), is optimal for their problem in the low privacy regime.
Interestingly, this mechanism appears in other contexts as well: for instance in the distributed computation problem studied by~\citet{KaOhVi15}, this mechanism turns out to be optimal in a fairly general sense.

An alternative (and complementary) privatization scheme to RR has been more recently introduced under the name of Randomized Aggregatable Privacy-Preserving Ordinal Response (RAPPOR) scheme \citep{ErPiKo14,DuWaJo13}. While the conventional RR scheme assumes that the source variable and the privatized share a same finite alphabet, the privatized representation in a RAPPOR scheme lives on an alphabet whose size grows exponentially in the source alphabet size, thus harshly impacting on storage or bandwidth requirements. It was shown in~\citet{KaBoRa16,KaBoRa16arxiv} that under $\ell_1$ (total variation distance) and $\ell_2$ (mean-squared error) losses alike, the RAPPOR scheme is order-optimal in the high privacy regime ($\epsilon \downarrow 0$) and strictly suboptimal in the low privacy regime ($\epsilon \uparrow \infty$). Conversely, the RR scheme is order-optimal in the low privacy, and sub-optimal in the high-privacy regime.

Recently, \citet{YeBa18} introduced a novel privatization scheme which can be viewed as a generalization of both RR and RAPPOR.\footnote{Ye and Barg acknowledge in their publication that~\citet{WaHuWaNiXuYaLiQi16} independently introduced the same privatization scheme.}
They prove that under $\ell_1$ and $\ell_2$ metrics, these schemes are order-optimal in the medium to high privacy regime ($e^\epsilon \ll K$, where $K$ denotes the cardinality of the discrete source's support). Subsequently in~\citet{YeBa17}, the same authors strengthen this result for the $\ell_2$ metric, by proving asymptotic optimality in all regimes. The result was recently extended to more general metrics (including the $\ell_1$ metric) by the same authors~\citep{YeBa18b}.

In the present work, rather than studying RAPPOR and its generalization by Ye and Barg, we turn our attention to the conventional, more storage efficient RR scheme.\footnote{Note that several key results known in the literature that facilitate further analysis, such as optimality of ``extremal'' mechanisms~\citep[Lemma~IV.3]{KaOhVi14,YeBa18}, cease to hold under the restriction of pure RR that we take here.}
Specifically, we consider the situation where an {\em interviewer} or {\em curator} observes independent and identically distributed (i.i.d.)\ realizations of potentially sensitive data. Instead of publishing the records in the clear, the curator is required to process this source data in such way that inferences on the individual source {\em realizations} are rendered hard, but an accurate estimation of the source {\em distribution} is rendered easy. In addition, the curator is constrained to using a memoryless privatization strategy, sometimes referred to as {\em non-interactive mechanism} \citep{Le12,DuWaJo13,DuJoWa14}.

\begin{figure}[htb]
\begin{center}
\includegraphics{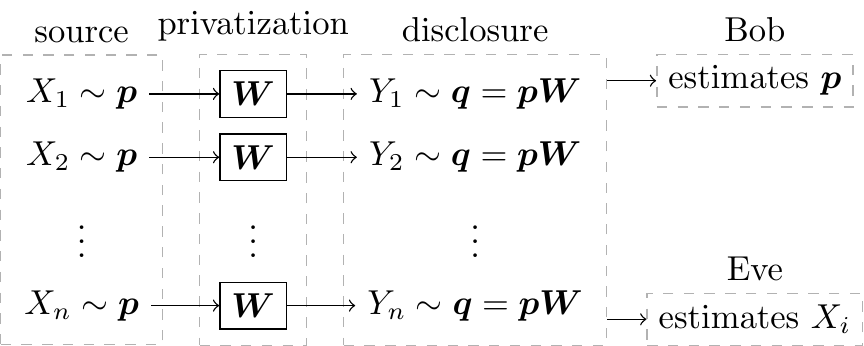}
\end{center}
\caption{Non-interactive mechanism: the curator sees $n$ samples of an i.i.d.\ discrete random source $\{X_i\}$ with distribution $\Bp$ and processes them individually through $n$ copies of a privatization channel $\BW$ prior to publication. From the outputs $\{Y_i\}$, the legitimate observer Bob tries to infer $\Bp$, whereas the adversarial observer Eve tries to infer one (or more) source sample $X_i$.}
\label{fig:system}
\end{figure}

Although non-interactive mechanisms are provably optimal in certain setups \citep{KaOhVi15}, the trade-off under study may in general benefit from more complex, interactive channel structures. However, there might be justified reasons to use a stationary and memoryless privatization channel:
\paragraph{Randomized survey}
In certain situations such as randomized response surveys \citep{Wa65}, the truthful answer to an interviewer's question might be stigmatizing (e.g., drug consumption, sexual behavior, etc.). In such setups, a transparent and non-interactive randomization of answers is necessary as a participatory incentive.
\paragraph{Timeliness}
Strict delay constraints may outrule the possibility of batch processing. Low-latency sensoring or time-critical surveys may be examples of such situations. For instance, a medical survey could be conducted over a timespan of several years, but there might be an urge to publish partial information at much more frequent intervals so as to help gain statistical insights in a timely manner.
\paragraph{Finite horizon}
In applications where the curator has no control over the eventual size $n$ of the data collection because it may be interrupted anytime, a non-interactive mechanism seems a more viable and robust approach.
\paragraph{Privacy fairness}
Applying one and the same privatization channel onto each data sample enforces full uniformity of privacy guarantees across samples in a simple and transparent way.

Our privacy requirement will be based on the notion of $\epsilon$-local differential privacy \citep{SaSa14}, which is inherited from the celebrated concept of differential privacy proposed by Dwork (see comprehensive surveys by \citet{Dw08,Le12,JiLiEl14}) by removing the adjacency relationship between datasets.
The $\epsilon$ parameter in local differential privacy conveys a sense of how uniformly hard it is to make inferences on the source realizations, regardless of the source distribution.

On the other hand, the fidelity will be linked to three alternative loss metrics---a family of Csisz\'ar $f$-divergence metrics (notably including Kullback--Leibler (KL) divergence), mean-squared error (MSE) and total variation (TV)---between the exact source distribution and an estimate thereof from the privatized samples.\footnote{Note that in some publications, the TV and MSE metrics are respectively referred to as $\ell_1$ and $\ell_2$-loss~\citep{YeBa18,YeBa17}} More specifically, the figure of merit will be the {\em speed of convergence} to zero of the expected fidelity loss (as measured by the metric of choice among the three metrics under study). For this purpose, we will derive asymptotic expressions of the expected MSE and TV losses for large sample sizes. As to the $f$-divergence metric, we will generalize\footnote{to account for an additional privatization channel} an asymptotic expansion of the expected KL divergence between the empirical distribution $\Bt(\Bx_n)$ of $n$ i.i.d.\ samples of a random variable $X \in [K]$ gathered in a vector $\Bx_n \sim \Bp^{\otimes n}$, and its exact distribution $\Bp = (p_1,\dotsc,p_K)$, as $n$ tends to infinity \citep{Ab96}:
\begin{equation}   \label{Abe_expansion}
	\Exp\bigl[D\bigl(\Bt(\Bx_n) \big\Vert \Bp\bigr)\bigr]
	= \frac{K-1}{2n} + \left(\sum_{k=1}^K \frac{1}{p_k}-1\right)\frac{1}{12 n^2} + O(n^{-3}).
\end{equation}
The first-order term of this expansion suggests that a low support set cardinality $K$ will be beneficial for the speed of convergence, the second-order term suggests that for a given cardinality, the uniform distribution is most beneficial.

The main contributions of this article can be summarized as follows:
\begin{itemize}
	\item	While previous publications have already studied large-sample size asymptotic expansions of standard loss metrics, we have sought to generalize these derivations in several ways: (i) we force the distribution estimate to be a valid probability distribution and rigorously treat the error term that arises from this projection operation\footnote{\citet{KaBoRa16,KaBoRa16arxiv}, for example, discusses projection operations to some detail and notices by simulation that a projected estimator tends to outperforms its unprojected counterpart, but provides no analytic treatment of the error term. Similarly, \citet{YeBa18} correctly point out that the impact of the projection operation is exponentially small, but omit the detailed analysis.}; (ii) in addition to TV and MSE loss metrics (elsewhere referred to as $\ell_1$ and $\ell_2$ risk) we also cover a large class of $f$-divergence metrics; (iii) we highlight that maximum-likelihood and MMSE distribution estimators acquire a similar form and argue that they yield the same asymptotic loss.
	\item	We identify the non-negative matrix $\BPhi(\BW) = \BW(\BW^{-1} \odot \BW^{-1})$ (where `$\odot$' denotes entrywise multiplication) as well as the sum of its entries $\varphi(\BW) = \sum_{ij} \Phi_{ij}(\BW)$ as representative proxies for a larger class of fidelity metrics. These quantities essentially capture the fidelity metric's dependency on the random mechanism (row-stochastic matrix) $\BW$. We study some of their properties, such as data-processing inequalities.
	\item	We derive a lower bound on $\varphi(\BW)$ which only depends on the privacy level $\epsilon$ and the source's support size $K$. This allows us to formulate new lower bounds on the optimal privacy--fidelity trade-off for a class of fidelity metrics.
\end{itemize}

\section{Notation}

By convention, all vectors are row vectors unless transposed by $(\cdot)^\Tr$.
We occasionally denote the inner product between two vectors $\Ba$ and $\Bb$ as $\langle \Ba,\Bb \rangle = \Ba\Bb^\Tr$.
The product signs `$\odot$' and `$\otimes$' stand for the Hadamard product (entrywise multiplication) and the Kronecker product, respectively.
The exponent notation $\Ba^{\otimes n}$ stands for the $n$-fold Kronecker product $\Ba \otimes \dotso \otimes \Ba$.
The bracket $[K]$ is shorthand for $\{1,2,\dotsc,K\}$.
The {\em type} or {\em empirical distribution} of a sample sequence $\Bx_n = ( X_1,\dotsc,X_n ) \in [K]^n$ shall be denoted as
$\Bt(\Bx_n) = [t_1(\Bx_n),\dotsc,t_K(\Bx_n)]$ where $t_k(\Bx_n) = \frac{1}{n} \sum_{i=1}^n \mathds{1}\{X_i = k\}$ with $\mathds{1}\{\cdot\}$ standing for the indicator function. The all-ones row vector of dimension $N$ is written as $\myone_N$, or simply $\myone$ if its dimension is clear from context.
The probability simplex is denoted as $\mathbb{P}$ (whose dimension is always clear from context).
We denote by $\delta_{k,\ell}$ the Dirac delta function, which equals one if $k=\ell$ and zero otherwise.

\section{Problem description}   \label{sec:problem_description}

A {\em privatization channel} or {\em mechanism} $\BW$ with finite source alphabet $[K]$ and finite output alphabet $[L]$, is a discrete stochastic mapping (or {\em Markov kernel}) described by a matrix of conditional probabilities $\BW$ with $(k,\ell)$-th entry
\begin{equation}
	W_{k,\ell} = \Prob\{Y=\ell|X=k\}.
\end{equation}
The matrix $\BW$ is row-stochastic, meaning that all its entries belong to the unit interval and that each row sums to one, i.e., $\BW \myone_L^\Tr = \myone_K^\Tr$. This article is only concerned with square channels, hence $L=K$ throughout.

The privatization channel acts independently upon each of the $n$ i.i.d.\ source symbols $\Bx_n = (X_1,\dotsc,X_n) \sim \Bp^{\otimes n}$ to generate a sequence of i.i.d.\ {\em privatized} observations $\By_n = (Y_1,\dotsc,Y_n) \sim \Bq^{\otimes n}$. The output distribution $\Bq$ induced by the source distribution $\Bp$ is given by right-multiplication of $\Bp$ with the channel matrix $\BW$. Moreover, we require that $\BW$ be {\em full-rank}. Hence,
\begin{align}   \label{pushforward}
	\Bq &= \Bp \BW
	&
	\Bp &= \Bq \BW^{-1}.
\end{align}
The curator is cognizant of $\BW$, has access to the output sequence $\By_n$ and seeks to generate an estimate of $\Bp$, which we denote as $\hat{\Bp}_n$. Since $\BW$ is square full-rank and the source is i.i.d., the quantity
\begin{equation}   \label{check_p}
	\check{\Bp}_n
	\triangleq \Bt(\By_n) \BW^{-1}
\end{equation}
is a complete and minimally sufficient statistic for $\Bp$. Consequently, any estimator $\hat{\Bp}_n$ can be defined as a function of $\check{\Bp}_n$ (and possibly $\BW$) without loss of optimality nor generality. Additionally, we require the estimator $\hat{\Bp}_n$ to be {\em consistent}, which means that
\begin{equation}   \label{consistency}
	\lim_{n \to \infty} \Prob\{ \lVert \hat{\Bp}_n - \Bp \rVert > \epsilon \}
	= 0
\end{equation}
for any $\epsilon > 0$. Clearly, the speed at which this convergence takes place will depend on how ``noisy'' the mechanism $\BW$ is designed to be. We would wish the convergence to be fast (for the sake of fidelity) while the mechanism should allow as little inference on $X_i$ from $Y_i$ as possible (for the sake of privacy), irrespective of the (unknown) source distribution. We now introduce the metrics for characterizing this privacy--fidelity trade-off.

\subsection{Fidelity}

To measure the accuracy of any given estimator $\hat{\Bp}_n$, we define three expected loss metrics: one based on Csisz\'ar's $f$-divergence, one based on mean-squared error (MSE) and one based on total variation distance (TV), namely\footnote{Besides $(\Bp,\BW)$, the fidelity loss metrics~\eqref{divergence_metric}--\eqref{tv_metric} also depend on the estimator function, but we choose to omit this dependency for notational brevity. The same applies to the asymptotic metrics presented further below, in~\eqref{alpha_DIV}--\eqref{alpha_TV}.}
\begin{subequations}
\begin{IEEEeqnarray}{rCl}
	\mathscr{L}^{(n)}_{f\text{-}\mathrm{DIV}}(\Bp,\BW)
	&\triangleq& \Exp\bigl[ D_f\bigl(\hat{\Bp}_n \big\Vert \Bp\bigr) \bigr]   \label{divergence_metric} \\
	\mathscr{L}^{(n)}_{\mathrm{MSE}}(\Bp,\BW)
	&\triangleq& \Exp\Bigl[\big\lVert \hat{\Bp}_n - \Bp \big\rVert_2^2\Bigr]   \label{mse_metric} \\
	\mathscr{L}^{(n)}_{\mathrm{TV}}(\Bp,\BW)
	&\triangleq& \Exp\Bigl[\big\lVert \hat{\Bp}_n - \Bp \big\rVert_1\Bigr].   \label{tv_metric}
\end{IEEEeqnarray}
\end{subequations}
Here, the $f$-divergence between two same-sized probability vectors $\Bu, \Bv \in \mathbb{P}$ and for a convex function $f$ satisfying $f(1)=0$, is defined as
\begin{equation}
	D_f(\Bp \Vert \Bq)
	= \sum_{k \in [K]} p_k f\left(\frac{q_k}{p_k}\right).
\end{equation}
Note that TV distance is actually an $f$-divergence (for the function $f(x) = |x-1|$) whereas the MSE distance is not. However, we choose to single out TV distance as a separate metric, since the class of $f$-divergences that we shall focus on requires differentiability of the function $f(x)$ at $x=1$, which excludes TV distance.

More specifically than the loss metrics~\eqref{divergence_metric}--\eqref{tv_metric} themselves, we will consider the {\em limiting ratios} between the loss metrics achieved {\em with} the privatizing mechanism $\BW$ against the corresponding value that would be obtained from a clear view on the samples, i.e., {\em without} privatization. These asymptotic normalized loss metrics are defined as
\begin{subequations}   \label{utility_metrics}
\begin{IEEEeqnarray}{rCl}
	\alpha_{f\text{-}\mathrm{DIV}}(\Bp,\BW)
	&\triangleq& \lim_{n\to\infty} \frac{\mathscr{L}^{(n)}_{f\text{-}\mathrm{DIV}}(\Bp,\BW)}{\mathscr{L}^{(n)}_{f\text{-}\mathrm{DIV}}(\Bp,\myid)}   \label{alpha_DIV} \\
	\alpha_{\mathrm{MSE}}(\Bp,\BW)
	&\triangleq& \lim_{n\to\infty} \frac{\mathscr{L}^{(n)}_{\mathrm{MSE}}(\Bp,\BW)}{\mathscr{L}^{(n)}_{\mathrm{MSE}}(\Bp,\myid)}   \label{alpha_MSE} \\
	\alpha_{\mathrm{TV}}(\Bp,\BW)
	&\triangleq& \lim_{n\to\infty} \left( \frac{\mathscr{L}^{(n)}_{\mathrm{TV}}(\Bp,\BW)}{\mathscr{L}^{(n)}_{\mathrm{TV}}(\Bp,\myid)} \right)^2,   \label{alpha_TV}
\end{IEEEeqnarray}
\end{subequations}
where $\myid$ denotes the identity (non-privatizing) channel.\footnote{The identity can be replaced by any permutation matrix, since a permutation amounts to a relabeling of symbols. The estimators $\hat{\Bp}_n$ introduced in the next Subsection are indeed consistent with this permutation invariance, in the sense that they give $\mathscr{L}^{(n)}(\Bp,\BW) = \mathscr{L}^{(n)}(\Bp,\BW\BPi)$ for any permutation matrix $\BPi$.}
Since, as we shall see, the metrics \eqref{divergence_metric}--\eqref{mse_metric} decay as $O(\frac{1}{n})$ when $n$ tends to infinity, whereas~\eqref{tv_metric} decays as $O\bigl(\frac{1}{\sqrt{n}}\bigr)$ (notice the square in~\eqref{alpha_TV} introduced to compensate for this fact), we can view the normalized quantities $\alpha_{f\text{-}\mathrm{DIV}}$, $\alpha_{\mathrm{MSE}}$ and $\alpha_{\mathrm{TV}}$ as rules of thumb for the factor by which the sample size has to be increased if we want the accuracy of the privatized estimation to approximately match that of the non-privatized case.

\subsection{Privacy}

Our definition of privacy is based on the concept of {\em local differential privacy}. A privatization channel $\BW = [W_{k,\ell}]_{k,\ell}$ as defined above is said to be $\epsilon$-locally differentially private (or $\epsilon$-private) if for all index triples $(k,k',\ell) \in [K]^3$, we have
\begin{equation}   \label{local_DP}
	W_{k,\ell} \leq e^\epsilon W_{k',\ell}.
\end{equation}
For a given channel $\BW$, we denote by $\epsilon(\BW)$ the smallest value of $\epsilon$ such that \eqref{local_DP} holds for all $(k,k',\ell)$, i.e.,
\begin{equation}   \label{privacy_metric}
	\epsilon(\BW) = \log\left( \max_{k,k',\ell} \frac{W_{k,\ell}}{W_{k',\ell}} \right).
\end{equation}
Let $\mathcal{W}$ denote the set of all $K \times K$ full-rank row-stochastic matrices and let $\mathcal{W}_\epsilon \subset \mathcal{W}$ denote the set of $\epsilon$-private mechanisms.

\subsection{Two problem formulations for the privacy--fidelity trade-off}   \label{ssec:P-U-trade-off}

Now that we have introduced the fidelity loss and privacy metrics in \eqref{utility_metrics} and \eqref{privacy_metric} respectively, we can formulate the privacy--fidelity trade-off problem.
We propose two different problem formulations, which we will refer to as the {\em feasibility problem} and the {\em minmax problem}, respectively.
In the following, the generic notation $\alpha$ may refer to either $\alpha_{f\text{-}\mathrm{DIV}}$, $\alpha_{\mathrm{MSE}}$ or $\alpha_{\mathrm{TV}}$.

\subsubsection{Feasibility problem}

Assume that the source distribution $\Bp$ is fixed. We seek to characterize the set $\mathcal{F}(\Bp)$ of all $(\epsilon,\alpha)$ pairs which are jointly feasible, i.e.,
\begin{equation}
	\mathcal{F}(\Bp)
	\triangleq \Bigl\{ \bigl(\alpha(\Bp,\BW),\epsilon(\BW)\bigr) \colon \BW \in \mathcal{W} \Bigr\}.
\end{equation}
Specifically, we seek to characterize the optimal $\epsilon$-$\alpha$ trade-off curve
\begin{equation}   \label{feasilibity_alpha_star}
	\alpha^\star(\epsilon;\Bp)
	\triangleq \min_{\BW \in \mathcal{W}_\epsilon} \alpha(\Bp,\BW)
\end{equation}
Note that, as we have stressed previously, the curator (i.e., the designer of $\BW$) has no knowledge of $\Bp$. Hence, the optimal trade-off curve $\alpha^\star(\epsilon;\Bp)$ is achieved in the event that the curator makes the best guess about the optimal $\BW$ for a given $\Bp$, as if aided by a genie who hands over the knowledge of $\Bp$. In general though, no design strategy for $\BW$ can leverage knowledge about $\Bp$, and thus will fall short of achieving $\alpha^\star(\epsilon;\Bp)$.

\subsubsection{Minmax problem}

Assume that the source distribution $\Bp$ can be any among a continuous subset $\mathcal{P}$ of the probability simplex $\mathbb{P}$. The curator seeks to optimize $\BW$ based on this knowledge of the continuous candidate set $\mathcal{P}$. Hence, we define the minmax problem as being that of determining
\begin{equation}   \label{minmax_alpha_star}
	\alpha^\star(\epsilon;\mathcal{P})
	\triangleq \min_{\BW \in \mathcal{W}_\epsilon} \sup_{\Bp \in \mathcal{P}} \alpha(\Bp,\BW).
\end{equation}
Reducing $\mathcal{P}$ to a singleton set $\{\Bp\}$ would make both problem formulations mathematically identical, in the sense that \eqref{minmax_alpha_star} would equal \eqref{feasilibity_alpha_star}. However, one has to bear in mind that the interpretations of both problem formulations are rather different. Indeed, the minmax problem makes sense mostly for a non-singleton set $\mathcal{P}$. Besides, assuming a singleton set would violate the assumption that $\mathcal{P}$ is continuous, which is important for another reason: if $\mathcal{P}$ were discrete, it would be more adequate to consider the problem within a guessing or multiple hypothesis testing framework. In this case, the guessing error would be represented by error probabilities of different types (e.g., false alarm, missed detection, etc.). For finite $\mathcal{P}$, these error probabilities would typically decay exponentially in the sample size $n$, so a natural candidate for the fidelity loss metric would be the error exponent, rather that the quantities $\alpha$ studied in this article.
Besides the feasibility and the minmax problem, one can think of other questions about fundamental limits which might be of independent interest as well, but will not be addressed in this article. Let us mention only one example:

\subsubsection{Best-case feasibility problem}

The best-case trade-off $\inf_{\Bp} \alpha^\star(\epsilon;\Bp)$ delimits the union $\bigcup_{\Bp} \mathcal{F}(\Bp)$ and characterizes the most optimistic performance limit, in the sense that the source distribution is most benevolent, and the curator $\BW$ guesses the best $\BW$. This limiting curve will only depend on the alphabet dimension $K$ (and possibly on the fidelity metric of choice, be it $\alpha_{f\text{-}\mathrm{DIV}}$, $\alpha_{\mathrm{MSE}}$ or $\alpha_{\mathrm{TV}}$).

\section{Distribution estimation}

As we have argued before, any distribution estimator can be expressed as a function of $\check{\Bp}_n$ (and possibly $\BW$) to the $K$-dimensional probability simplex. Henceforth, we shall only consider estimators that are {\em projectors} of $\check{\Bp}_n$ onto the probability simplex\footnote{Note that $\check{\Bp}_n$ [cf.~\eqref{check_p}] is not guaranteed to be a probability vector. Due to $\BW$ being row-stochastic, both $\BW$ and $\BW^{-1}$ have $1$ as an eigenvalue, with associated all-ones eigenvector $\myone^\Tr$. Thus, the rows of $\BW^{-1}$ and hence the entries of $\check{\Bp}_n = \Bt(\By_n)\BW^{-1}$ sum to one. However, some entries of $\check{\Bp}_n$ may lie outside the unit interval. Hence the necessity of a projection operation.}, namely, estimators that can be cast into the form
\begin{equation}   \label{general_estimator}
	\hat{\Bp}_n
	= \mathrm{Proj}_\mathbb{P}(\check{\Bp}_n)
\end{equation}
where $\mathrm{Proj}_\mathbb{P}(\cdot)$ stands for some idempotent function satisfying $\mathrm{Proj}_\mathbb{P}(\check{\Bp})=\check{\Bp}$ for any probability vector $\check{\Bp} \in \mathbb{P}$.

Let us denote the topological interior and closure (in $\mathbb{P}$) of a set of distributions $\mathcal{R} \subset \mathbb{P}$ as $\overline{\mathcal{R}}$ and $\mathcal{R}^\circ$, respectively, and define its boundary as $\partial\mathcal{R} = \overline{\mathcal{R}} \setminus \mathcal{R}^\circ$.
Henceforth, for a distribution $\Br \in \mathbb{P}$ and a set $\mathcal{R} \subset \mathbb{P}$, we adopt Csisz\'ar's notation~\citep{Cs84} for information projection
\begin{equation}   \label{information_projection}
	D(\mathcal{R} \Vert \Br)
	\triangleq \inf_{\Br' \in \mathcal{R}} D(\Br' \Vert \Br)
\end{equation}
where $D(\cdot \Vert \cdot)$ denotes the Kullback--Leibler divergence.
\begin{lemma}   \label{lem:concentration_inequality}
The following inequality holds:\footnote{We omit parentheses in writing $\partial\mathbb{P}\BW$ because regardless of how we set parentheses, $\partial(\mathbb{P}\BW) = (\partial\mathbb{P})\BW$.}
\begin{equation}
	\Prob\bigl\{ \check{\Bp}_n \notin \mathbb{P} \bigr\}
	\leq e^{-n D(\partial\mathbb{P}\BW \Vert \Bp\BW)}.
\end{equation}
\end{lemma}
\begin{proof}
See Appendix~\ref{app:proof:concentration_inequality}.
\end{proof}
It is easy to show with Lemma~\ref{lem:concentration_inequality} that any estimator of the form~\eqref{general_estimator} is consistent [cf.~Section~\ref{sec:problem_description}], because the cases $\check{\Bp} \neq \hat{\Bp}$ are at least exponentially rare, whereas fidelity metrics decay linearly in the sample size (as we shall see). Moreover, imposing the form~\eqref{general_estimator} outrules the possibility of trivial ``genie-aided'' estimators such as $\hat{\Bp}_n = \Bp$ by construction.

In addition, the fact that $\check{\Bp}_n$ fails to be a probability vector only in exponentially rare cases, as highlighted by Lemma~\ref{lem:concentration_inequality}, is helpful in that our asymptotic loss metrics~\eqref{alpha_DIV}--\eqref{alpha_TV} do not depend on $\mathrm{Proj}_\mathbb{P}$ (as we shall see). Therefore, all estimators of the form~\eqref{general_estimator} can be regarded as asymptotically equivalent.

Next, we will derive the maximum likelihood (ML) and the minimum mean-squared error (MMSE) distribution estimators, and verify that they are two instances of the general form~\eqref{general_estimator}.

\subsection{ML estimator}

Based on a given output sequence $\tilde{\By}_n$, the ML estimator of the source distribution is defined as the distribution that maximizes the probability of the event $\By_n = \tilde{\By}_n$. By a classic argument, one can express the ML estimator as a KL divergence minimizer:
\begin{IEEEeqnarray*}{rCl}
	\hat{\Bp}_\mathrm{ML}(\tilde{\By}_n)
	&=& \argmax_{\Bp' \in \mathbb{P}} \frac{1}{n} \log\Prob\bigl\{ \By_n = \tilde{\By}_n \bigm| \By_n \sim (\Bp' \BW)^{\otimes n} \bigr\} \\
	&=& \argmax_{\Bp' \in \mathbb{P}} \frac{1}{n} \log\prod_{i=1}^n \Prob\bigl\{ y_i = \tilde{y}_i \bigm| y_i \sim \Bp' \BW \bigr\} \\
	&=& \argmax_{\Bp' \in \mathbb{P}} \frac{1}{n} \sum_{i=1}^n \sum_{k=1}^K \mathds{1}\{\tilde{y}_i = k\} \log\bigl([\Bp'\BW]_k\bigr) \\
	&=& \argmax_{\Bp' \in \mathbb{P}} \sum_{k=1}^K t_k(\tilde{\By}_n) \log\bigl([\Bp'\BW]_k\bigr) \\
	&=& \argmin_{\Bp' \in \mathbb{P}} D\bigl(\Bt(\tilde{\By}_n) \big\Vert \Bp' \BW\bigr).   \IEEEeqnarraynumspace\IEEEyesnumber\label{ML_KL_equivalence}
\end{IEEEeqnarray*}
This optimization problem is convex, because the KL divergence is a convex functional and the probability simplex $\mathbb{P}$ is a convex set.
To see that this is an instance of~\eqref{general_estimator}, it suffices to rewrite $\Bt(\tilde{\By}_n)$ as $\Bt(\tilde{\By}_n)\BW^{-1}\BW$ in~\eqref{ML_KL_equivalence} so that the idempotence property of projection becomes evident.

\subsection{MMSE estimator}

The MMSE estimator is defined as the probability vector that minimizes the Euclidean distance between the output distribution it induces, and the empirical output distribution:
\begin{IEEEeqnarray*}{rCl}
	\hat{\Bp}_\mathrm{MMSE}(\tilde{\By}_n)
	&=& \argmin_{\Bp' \in \mathbb{P}} \, \bigl\lVert \Bt(\tilde{\By}_n) - \Bp'\BW \bigr\rVert_2.   \IEEEeqnarraynumspace\IEEEyesnumber\label{MMSE_estimator}
\end{IEEEeqnarray*}
This estimator is similar to the ML estimator, except for replacing the KL divergence in~\eqref{ML_KL_equivalence} by Euclidean distance. Similarly to the ML estimator, the minimization problem in~\eqref{MMSE_estimator} is convex and manifestly an instance of~\eqref{general_estimator}.

\section{Convergence of estimates}

Recall that $\Bp$ is supported on $[K]$, meaning that all its entries $p_k$ are positive. As a consequence, the convergence in probability~\eqref{consistency} implies the convergence to zero of all fidelity loss metrics (including $f$-divergences) as $n \to \infty$, i.e.,
\begin{equation}
	\lim_{n\to\infty} \mathscr{L}^{(n)}(\Bp,\BW)
	= 0.
\end{equation}
As a representative figure for the speed of this convergence to zero, we shall compute the leading terms in the respective asymptotic expansions (as $n\to\infty$) of the different loss metrics. Prior to providing analytical expressions for these, we need to introduce some quantities of interest. For a positive integer $\rho$, let us define
\begin{equation}   \label{def:nu}
	\nu_{\rho,k}
	\triangleq \Bp \BW\bigl( \underbrace{\BW^{-1} \odot \BW^{-1} \odot \dotso \odot \BW^{-1}}_{\text{$\rho$ factors}} \bigr) \Be_k^\Tr
\end{equation}
where `$\odot$' denotes entrywise multiplication. In particular, the matrix
\begin{equation}   \label{def:Phi}
	\BPhi(\BW) \triangleq \BW(\BW^{-1} \odot \BW^{-1})
\end{equation}
involved in the expression of $\nu_{2,k}$ will play a prominent role in the fidelity loss metrics and will be shown to satisfy data-processing inequalities. Finally, the sum of all entries of $\BPhi(\BW)$, i.e.,
\begin{equation}   \label{def:phi}
	\varphi(\BW)
	\triangleq \myone \BPhi(\BW) \myone^\Tr
	= \sum_{k=1}^K \sum_{\ell=1}^K \Phi_{k,\ell}(\BW)
\end{equation}
will be repeatedly used.

The following theorem gives an asymptotic expansion applicable to a large class of $f$-divergences, including KL divergence:

\begin{theorem}[Expansion of $f$-divergence loss]   \label{thm:asymptotic_expansion}
Assume that
\begin{enumerate}
	\item	$f(1)=0$ and $f(x)$ is four times differentiable at $x=1$,
	\item	$f(0)$ is finite\footnote{The requirement that $f(0)$ be finite ensures that the expected value $\Exp\bigl[ D_f\bigl(\hat{\Bp}_n \big\Vert \Bp\bigr) \bigr]$ is finite, for otherwise there would be a positive probability of $D_f\bigl(\hat{\Bp}_n \big\Vert \Bp\bigr)$ being infinite for any $n$, and thus its expectation would be infinite.}
	\item	$f(x)$ can be expanded as
		\begin{equation*}
			f(x) = \sum_{\rho=1}^4 \frac{f^{(\rho)}(1)}{\rho!}(x-1)^\rho + O\bigl( \left| x-1 \right|^{4+\gamma} \bigr)
		\end{equation*}
		for some $\gamma>0$, where $f^{(\rho)}(x)$ denotes the $\rho$-th derivative of $f(x)$.
\end{enumerate}
Furthermore, assume that $\BW$ is full-rank (invertible). Then the following asymptotic expansion holds:
\begin{equation}   \label{asymptotic_expansion}
	\mathscr{L}^{(n)}_{f\text{-}\mathrm{DIV}}(\Bp,\BW)
	= \frac{A f''(1)}{2n} + \left( \frac{B f^{(3)}(1)}{6} + \frac{C f^{(4)}(1)}{8} \right) \frac{1}{n^2} + O(n^{-3}) 
\end{equation}
with coefficients $A$, $B$ and $C$ given by
\begin{subequations}
\begin{IEEEeqnarray}{rCl}
	A
	&=& -1 + \sum_{k=1}^K \frac{\nu_{2,k}}{\nu_{1,k}} \\
	B
	&=& 2 + \sum_{k=1}^K \left(\frac{\nu_{3,k}}{\nu_{1,k}^2} - 3\frac{\nu_{2,k}}{\nu_{3,k}}\right) \\
	C
	&=& 1 + \sum_{k=1}^K \left(\frac{\nu_{2,k}^2}{\nu_{1,k}^3}-2\frac{\nu_{2,k}}{\nu_{1,k}}\right)   \label{gamma}
\end{IEEEeqnarray}
\end{subequations}
where $\nu_{\rho,k}$ is defined in~\eqref{def:nu}. (Note that $\nu_{1,k}$ is simply $p_k$)
\end{theorem}

\begin{table}[htb]
\centering\footnotesize
{\renewcommand{\arraystretch}{1.4}
\begin{tabular}{|m{2.5cm}|m{2cm}|m{1cm}|m{1cm}|m{1cm}|m{1cm}|}
	\cline{2-6}
	\multicolumn{1}{c|}{}
	& $f(x)$ & $f^{(1)}(1)$ & $f^{(2)}(1)$ & $f^{(3)}(1)$ & $f^{(4)}(1)$ \\
	\hline
	KL divergence & $x\ln(x)$ & $1$ & $1$ & $-1$ & $2$ \\
	\hline
	\multirow{2}{*}{\parbox[c]{0cm}{Hellinger\\distance}} & $(\sqrt{x}-1)^2$ & 0 & $\nicefrac{1}{2}$ & $-\nicefrac{3}{4}$ & $\nicefrac{15}{8}$ \\
	\cline{2-6}
	& $1-\sqrt{x}$ & $-\nicefrac{1}{2}$ & $\nicefrac{1}{4}$ & $-\nicefrac{3}{8}$ & $\nicefrac{15}{16}$ \\
	\hline
	\multirow{2}{*}{\parbox[c]{2cm}{Pearson\\$\chi^2$ divergence}} & $(x-1)^2$ & $0$ & $2$ & $0$ & $0$ \\
	\cline{2-6}
	& $x^2-1$ & $2$ & $2$ & $0$ & $0$ \\
	\hline
	\vspace{1mm}\parbox[]{0cm}{Triangular\\discrimination}\vspace{1mm}  & $\frac{(x-1)^2}{x+1}$ & $0$ & $1$ & $-\nicefrac{3}{2}$ & $3$ \\
	\hline
	TV distance & $|x-1|$ & -- & -- & -- & -- \\
	\hline
\end{tabular}
}
\caption{First four derivatives at $1$ of functions $f(x)$ associated to different $f$-divergences. Some definitions vary across the literature.}
\end{table}

If one assumes that $f$ is only twice differentiable and that $f(x) = f^{(1)}(1)(x-1) + \tfrac{1}{2} f^{(2)}(1)(x-1)^2 + O\bigl(|x-1|^{2+\gamma}\bigr)$, one can also prove a simpler version of Theorem~\ref{thm:asymptotic_expansion}, namely that $\mathscr{L}^{(n)}_{f\text{-}\mathrm{DIV}}(\Bp,\BW) = \frac{A f''(1)}{2n} + O(n^{-2})$.

\begin{theorem}[Expansion of MSE loss]   \label{thm:MSE}
It holds that\footnote{Unlike the differentiable $f$-divergence metrics which can be Taylor-expanded [cf.~\eqref{asymptotic_expansion}] to, in general, an infinity of terms, the MSE metric expansion has only a single $O(n^{-1})$ term. That is, it is exact up to an exponential remainder term which is attributable to the projection operation.}
\begin{equation}
	\mathscr{L}^{(n)}_{\mathrm{MSE}}(\Bp,\BW)
	= \frac{1}{n}\sum_{k=1}^K \bigl( \nu_{2,k} - \nu_{1,k}^2 \bigr) + O\bigl(e^{-n D(\partial\mathbb{P}\BW \Vert \Bp\BW)}\bigr).   \label{asymptotic_MSE_expansion}
\end{equation}
\end{theorem}

\begin{theorem}[Expansion of TV loss]   \label{thm:TV}
It holds that
\begin{equation}
	\mathscr{L}^{(n)}_{\mathrm{TV}}(\Bp,\BW)
	= \sqrt{\frac{2}{\pi n}} \sum_{k=1}^K \sqrt{\nu_{2,k} - \nu_{1,k}^2} + o\left(\frac{1}{\sqrt{n}}\right).   \label{asymptotic_TV_expansion}
\end{equation}
\end{theorem}

The proofs of Theorems~\ref{thm:asymptotic_expansion}, \ref{thm:MSE} and \ref{thm:TV}  are given in Appendices~\ref{app:proof:asymptotic_expansion}, \ref{app:proof:MSE} and \ref{app:proof:TV}, respectively. By applying them on the definition of the normalized first-order terms \eqref{utility_metrics}, while calling to mind that $\nu_{1,k} = p_k$ and that $\nu_{2,k} = \Bp\BPhi(\BW)\Be_k^\Tr$, we obtain the explicit expressions
\begin{subequations}
\begin{IEEEeqnarray}{rCl}
	\alpha_{f\text{-}\mathrm{DIV}}(\Bp,\BW)
	&=& \frac{\Bp\BPhi(\BW)\Bp^{-\Tr}-1}{K-1}   \IEEEeqnarraynumspace\label{alpha_KL_in_full} \\
	\alpha_{\mathrm{MSE}}(\Bp,\BW)
	&=& \frac{ \Bp \BPhi(\BW) \myone^\Tr - \lVert \Bp \rVert_2^2 }{1-\lVert \Bp \rVert_2^2}   \IEEEeqnarraynumspace\label{alpha_MSE_in_full} \\
	\alpha_{\mathrm{TV}}(\Bp,\BW)
	&=& \left( \frac{ \sqrt{ \Bp \BPhi(\BW) - \Bp \odot \Bp } \, \myone^\Tr }{ \sqrt{ \Bp - \Bp \odot \Bp } \, \myone^\Tr } \right)^2   \IEEEeqnarraynumspace\label{alpha_TV_in_full}
\end{IEEEeqnarray}
\end{subequations}
where in~\eqref{alpha_TV_in_full}, square roots on vectors are applied entrywise.
Notice how $\BPhi(\BW)$ plays a central role, in that it fully captures the dependency in $\BW$ of all three loss metrics. To better appreciate the significance and behavior of these metrics, a few remarks are in order.

\begin{myrmk}
Interestingly, for the uniform source $\Bp = \myone/K$, the $f$-divergence and MSE metrics happen to coincide, regardless of the choice of mechanism $\BW \in \mathcal{W}$:
\begin{IEEEeqnarray}{rCl}
	\alpha_{f\text{-}\mathrm{DIV}}\bigl(\tfrac{\myone}{K},\BW\bigr)
	= \alpha_{\mathrm{MSE}}\bigl(\tfrac{\myone}{K},\BW\bigr)
	&=& \frac{\varphi(\BW)-1}{K-1}.   \IEEEeqnarraynumspace\label{alpha_for_uniform}
\end{IEEEeqnarray}
By contrast, the TV metric is generally smaller, which can be shown by Jensen's inequality:
\begin{equation}
	\alpha_{\mathrm{TV}}\bigl(\tfrac{\myone}{K},\BW\bigr)
	= \left( \frac{\textstyle\sum_k \sqrt{ K \sum_\ell \Phi_{k,\ell}(\BW) - 1}}{K\sqrt{K-1}} \right)^2
	\leq \frac{\varphi(\BW)-1}{K-1}.   \label{Jensen}
\end{equation}
As we shall see in Lemma~\ref{lem:matching_metrics} of Section~\ref{sec:upper_bounds}, this inequality becomes tight when the mechanism is circulant, thus making all three metrics match in such case.
\end{myrmk}
\begin{myrmk}
In the noiseless case, $\BW$ and its inverse $\BW^{-1}$ are permutation matrices ($\BPi$ and $\BPi^\Tr$, respectively), whose entries are equal to $0$ or $1$. Entrywise squaring leaves them unchanged, hence $\BPhi(\BPi) = \BPi(\BPi^\Tr \odot \BPi^\Tr) = \BPi \BPi^\Tr = \myid$ and we see that the metrics~\eqref{alpha_KL_in_full}--\eqref{alpha_TV_in_full} all become equal to one, which is consistent with our expectation based on the definitions~\eqref{alpha_DIV}--\eqref{alpha_TV}. Furthermore we can verify that for $\BW$ a permutation, specializing Theorem~\ref{thm:asymptotic_expansion} to the KL divergence recovers the asymptotic expansion \eqref{Abe_expansion} by evaluation of \eqref{asymptotic_expansion}--\eqref{gamma}.
\end{myrmk}
\begin{myrmk}   \label{rmk3}
As we shall see in the next section [cf.~\eqref{DPI_particularized}], it holds that $[\BPhi(\BW)]_{k,\ell} \geq \delta_{k,\ell}$ (where $\delta_{k,\ell}$ stands for the Dirac delta), thanks to which it becomes manifest that the metrics~\eqref{alpha_KL_in_full}--\eqref{alpha_TV_in_full} are larger or equal to one.
\end{myrmk}
\begin{myrmk}
Concerning the dependency on $\Bp$, we notice from inspecting \eqref{alpha_KL_in_full}--\eqref{alpha_TV_in_full} that $\alpha_{\mathrm{MSE}}(\Bp,\BW)$ and $\alpha_{\mathrm{TV}}(\Bp,\BW)$ become large if $\Bp$ tends to a canonical base vector (i.e., when $\lVert \Bp \rVert_2^2 \to 1$ or equivalently $\sqrt{\Bp - \Bp \odot \Bp}\,\myone^\Tr \to 1$) whereas for the divergence metric $\alpha_{f\text{-}\mathrm{DIV}}(\Bp,\BW)$, it already suffices to have at least one low-probability symbol ($p_k \to 0$ for some $k \in [K]$) for $\alpha_{f\text{-}\mathrm{DIV}}(\Bp,\BW)$ to become large. Loosely speaking, low-probability symbols are more ``penalizing'' for the $f$-divergence loss metric than they are for the MSE or TV loss metrics.
\end{myrmk}

To illustrate Theorems~\ref{thm:asymptotic_expansion}, \ref{thm:MSE} and \ref{thm:TV}, we provide numerical evaluations of $\mathscr{L}_{f\text{-}\mathrm{DIV}}^{(n)}$, $\mathscr{L}_{\mathrm{MSE}}^{(n)}$ and $\mathscr{L}_{\mathrm{TV}}^{(n)}$ for both the ML and MSE estimators~\eqref{ML_KL_equivalence} and \eqref{MMSE_estimator}, respectively, and compare them against the non-privatized performance. For the $\epsilon$-private mechanism, we choose the matrix
\begin{equation}   \label{W_star}
	\BW_{\epsilon,\star}
	= \frac{1}{e^\epsilon+K-1}
	\begin{bmatrix}
		e^\epsilon & 1 & \hdots & 1 \\
		1 & e^\epsilon & \ddots & \vdots \\
		\vdots & \ddots & \ddots & 1 \\
		1 & \hdots & 1 & e^\epsilon
	\end{bmatrix}
\end{equation}
which we shall call the {\em step mechanism}, and for which we will prove optimality results further on.
With this choice of mechanism, the ML and MMSE estimators~\eqref{ML_KL_equivalence} and \eqref{MMSE_estimator} are expressible by waterfilling-type closed forms 
\begin{subequations}
\begin{IEEEeqnarray}{rCl}
	\hat{p}_{\mathrm{ML},k}
	&=& \frac{1}{e^\epsilon-1} \max\bigl\{ 0 , (e^\epsilon+K-1) t_k + \eta_{\mathrm{ML}} \bigr\} \IEEEeqnarraynumspace\label{ML_projection} \\
	\hat{p}_{\mathrm{MMSE},k}
	&=& \frac{1}{e^\epsilon-1} \max\bigl\{ 0 , \eta_{\mathrm{MMSE}} t_k - 1 \bigr\}   \IEEEeqnarraynumspace\label{MMSE_projection}
\end{IEEEeqnarray}
\end{subequations}
where the scalars $\eta_{\mathrm{ML}}$ and $\eta_{\mathrm{MMSE}}$ are chosen such that the sum constraints $\sum_k \hat{p}_{\mathrm{ML},k} = \sum_k \hat{p}_{\mathrm{MMSE},k} = 1$ are met.

Evaluating the quantities $\nu_{\rho,k}$ defined in~\eqref{def:nu} for the step mechanism~\eqref{W_star} yields
\begin{subequations}
\begin{IEEEeqnarray*}{rCl}
	\nu_{1,k}
	&=& p_k   \IEEEeqnarraynumspace\IEEEyesnumber\\
	\nu_{2,k}
	&=& \frac{1}{(e^\epsilon-1)^2} \bigl( (e^\epsilon-1)(e^\epsilon+K-3) p_k + e^\epsilon+K-2 \bigr) \IEEEeqnarraynumspace\IEEEyesnumber\\
	\nu_{3,k}
	&=& \frac{1}{(e^\epsilon-1)^3(e^\epsilon+K-1)} \bigl( (e^\epsilon-1)((e^\epsilon+K-2)^3+1) p_k \\
	&& \qquad {} + (e^\epsilon+K-1)(e^\epsilon+K-2)(e^\epsilon+K-3) \bigr).   \IEEEeqnarraynumspace\IEEEyesnumber
\end{IEEEeqnarray*}
\end{subequations}

Figure~\ref{fig:loss_metrics_vs_sample_size} shows how the fidelity loss metrics decay with $n$. The exact values of $\mathscr{L}_{f\text{-}\mathrm{DIV}}^{(n)}$ (for KL divergence $f(x) = x \ln(x)$), $\mathscr{L}_{\mathrm{MSE}}^{(n)}$ and $\mathscr{L}_{\mathrm{TV}}^{(n)}$ are plotted against the second-order approximation~\eqref{asymptotic_expansion} and the first-order approximations~\eqref{asymptotic_MSE_expansion} and~\eqref{asymptotic_TV_expansion}, respectively. The same plot is exhibited twice, once for a small (Figure~\ref{fig:loss_metrics_vs_sample_size_zoomed_in}) and once for a large (Figure~\ref{fig:loss_metrics_vs_sample_size_zoomed_out}) range of values of $n$. Though it is not the focus of this article, it is worth mentioning that the pseudo-periodic fluctuations visible on Figure~\ref{fig:loss_metrics_vs_sample_size_zoomed_in} are traceable to how empirical distributions are better approximations of the limiting distribution for certain values of $n$ than for others, an effect related to Diophantine approximations. Note that the periodicity and magnitude of fluctuations crucially depend on the alphabet cardinality $K$ and on the values of $p_k, k = 1,\dotsc,K$. 
\begin{figure}[ht]
\begin{center}
\subfloat[For $n$ ranging from $1$ to $200$]{
	\includegraphics{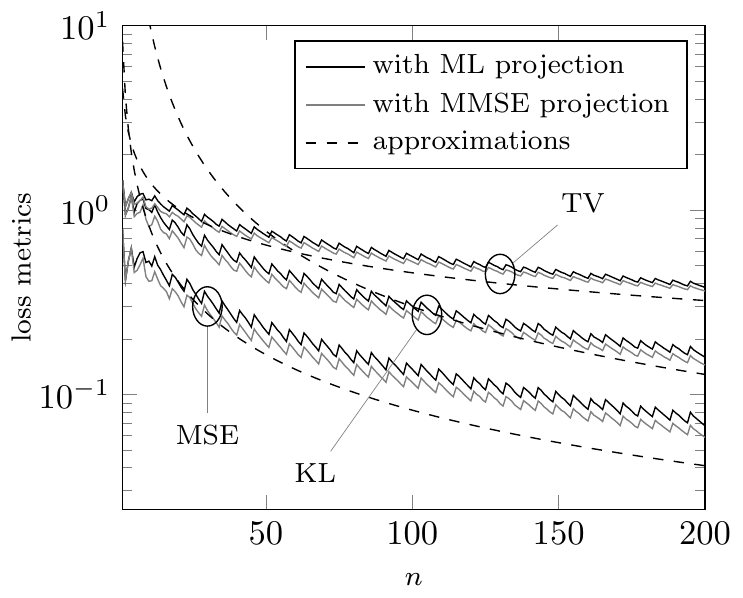}
	\label{fig:loss_metrics_vs_sample_size_zoomed_in}
}
\subfloat[For $n$ ranging from $1$ to $2000$]{
	\includegraphics{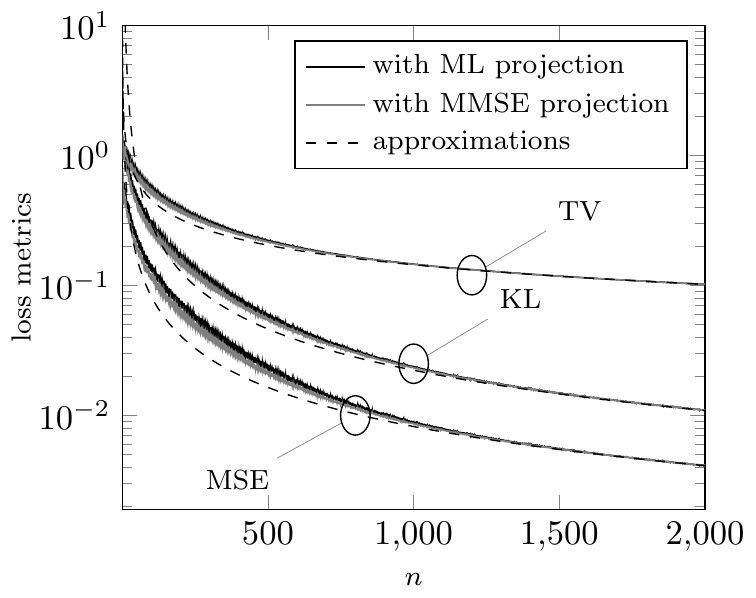}
	\label{fig:loss_metrics_vs_sample_size_zoomed_out}
}
\end{center}
\caption{Unnormalized loss metrics $\mathscr{L}^{(n)}(\Bp,\BW_{\epsilon,\star})$ (for KL, MSE and TV metrics) plotted as functions of the sample size $n$. We have chosen $K = 4$, $\epsilon = 1$, $\Bp = [0.5\ 0.25\ 0.125\ 0.125]$ and the step mechanism $\BW_{\epsilon,\star}$. The dashed curves show the approximations via truncated expansion, obtained from omitting the big-$O$ or little-$o$ remainder terms in~\eqref{asymptotic_expansion}, \eqref{asymptotic_MSE_expansion}, \eqref{asymptotic_TV_expansion}, respectively.   \label{fig:loss_metrics_vs_sample_size}
}
\end{figure}

\begin{figure}[ht]
\begin{center}
\includegraphics{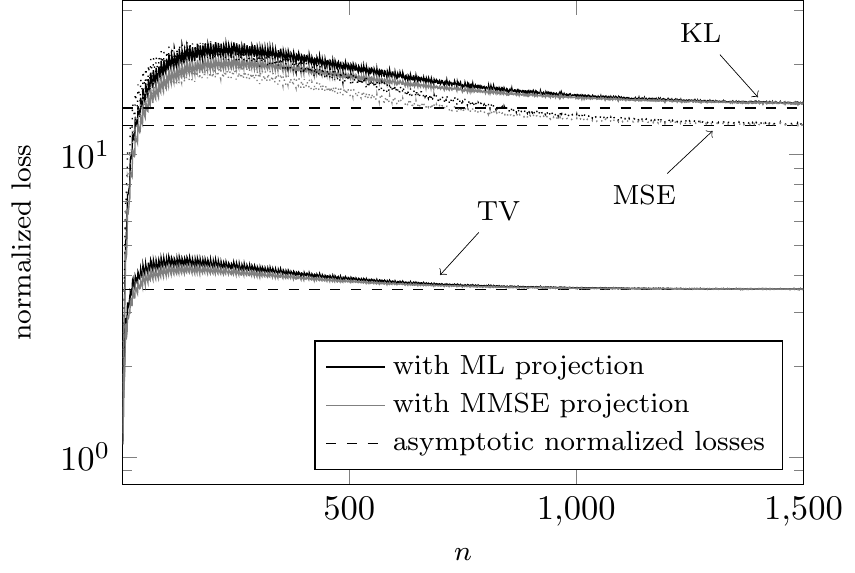}
\caption{Normalized loss metrics $\mathscr{L}^{(n)}(\Bp,\BW_{\epsilon,\star})/\mathscr{L}^{(n)}(\Bp,\myid)$ plotted as functions of the sample size $n$, with same parameters as in Figure~\ref{fig:loss_metrics_vs_sample_size}. Note that curves corresponding to ML projection are black, whereas curves corresponding to MMSE projection are gray. For better visualization in regions where curves overlap, we have chosen dotted curves for the MSE metric. The dashed curves show the limits as $n \to \infty$, which correspond to the quantities~\eqref{alpha_KL_in_full}--\eqref{alpha_TV_in_full}.}
\end{center}
\end{figure}

\section{Data processing theorems}

Intuitively, it is clear that any estimator of the source distribution $\Bp$ based on the privatized observations $\By_n$ will perform worse, in terms of metrics \eqref{divergence_metric}--\eqref{tv_metric}, than the empirical distribution estimator $\Bt(\Bx_n)$ based on a clear view of the source symbol vector $\Bx_n$. This idea is formalized by the data-processing theorems stated below.

\begin{figure}[htb]
\begin{center}
\begin{tikzpicture}
	\matrix[column sep=4mm] {
		\node (X) {$X \sim \Bp$};
		& \node[draw] (W) {$\BW$};
		& \node (Y) {$Y \sim \Bp\BW$};
		& \node[draw] (WW) {$\BW'$};
		& \node (YY) {$Y' \sim \Bp\BW\BW'$}; \\
	};
	\node[draw,dashed,opacity=0.3,fit={(X)},label=above:source] {};
	\node[draw,dashed,opacity=0.3,fit={(Y)},label=above:first disclosure] {};
	\node[draw,dashed,opacity=0.3,fit={(YY)},label=above:second disclosure] {};
	\draw[->] (X)--(W);
	\draw[->] (W)--(Y);
	\draw[->] (Y)--(WW);
	\draw[->] (WW)--(YY);
\end{tikzpicture}
\end{center}
\caption{Data processing theorems.}
\label{fig:DPI}
\end{figure}
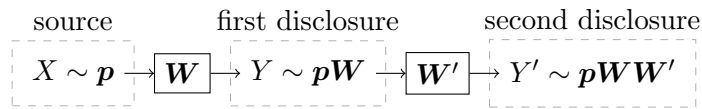

\begin{theorem}[General form]   \label{thm:general_DPI}
Consider $n$ copies of the setup as depicted in Figure~\ref{fig:DPI}. That is, assume that $\By_n \sim (\Bp\BW)^{\otimes n}$ and $\By'_n \sim (\Bp\BW\BW')^{\otimes n}$ are obtained from passing the samples $\Bx_n$ through $n$ copies of the channels $\BW'$ and $\BW\BW'$, respectively. In addition, assume that
\begin{enumerate}
	\item	for $f$-divergence metrics, $f(x)$ is strictly convex\footnote{Note that strict local convexity is satisfied by all $f$-divergences of interest which also satisfy the conditions of Theorem~\ref{thm:asymptotic_expansion}. It is also satisfied by TV distance, in the sense that for $f(x) = |x-1|$, we have $f(\lambda(1-\epsilon) + (1-\lambda)(1+\epsilon')) < \lambda f(1-\epsilon) + (1-\lambda) f(1+\epsilon')$ for all $\epsilon,\epsilon' > 0$.} at $x=1$;
	\item	$\BW'$ is not a permutation.
\end{enumerate}
Then, for any estimator of the form~\eqref{general_estimator}, for any source distribution $\Bp$ supported on $[K]$ and for sufficiently large $n$, we have
\begin{subequations}
\begin{IEEEeqnarray}{rCl}
	\mathscr{L}^{(n)}(\Bp,\BW\BW')
	&>& \mathscr{L}^{(n)}(\Bp,\BW).
\end{IEEEeqnarray}
\end{subequations}
\end{theorem}
\begin{proof}
	See Appendix~\ref{app:proof:general_DPI}.
\end{proof}

Unsurprisingly, this data-processing relationship carries over directly to the respective leading terms in the asymptotic expansions of the three loss metrics, as stated in the following corollary.
\begin{corollary}[Data-processing inequality for $\alpha$]   \label{cor:DPI_alpha}
For any source distribution $\Bp$ and channels $\BW$ and $\BW'$, it holds that
\begin{subequations}
\begin{IEEEeqnarray}{rCl}
	\alpha_{f\text{-}\mathrm{DIV}}(\Bp,\BW\BW') &\geq& \alpha_{f\text{-}\mathrm{DIV}}(\Bp,\BW)   \label{alpha_DIV_DPI} \\
	\alpha_{\mathrm{MSE}}(\Bp,\BW\BW') &\geq& \alpha_{\mathrm{MSE}}(\Bp,\BW)   \label{alpha_MSE_DPI} \\
	\alpha_{\mathrm{TV}}(\Bp,\BW\BW') &\geq& \alpha_{\mathrm{TV}}(\Bp,\BW)   \label{alpha_TV_DPI}
\end{IEEEeqnarray}
\end{subequations}
with equality if and only if $\BW'$ is a permutation.
\end{corollary}
Recall that $\alpha(\Bp,\myid) = 1$. By setting $\BW$ to the identity in~\eqref{alpha_DIV_DPI}--\eqref{alpha_TV_DPI}, we recover the already known fact that $\alpha(\Bp,\BW') \geq 1$ holds for any mechanism $\BW' \in \mathcal{W}$.
\begin{proof}
Suppose that there exists a source distribution $\Bp$ and a pair of channels $(\BW,\BW')$ such that $\alpha(\Bp,\BW\BW') < \alpha(\Bp,\BW)$ (where $\alpha$ stands for either metric). Then, by Definitions~\eqref{alpha_DIV}--\eqref{alpha_TV}, there exist arbitrarily large sample sizes $n$ such that $\mathscr{L}^{(n)}(\Bp,\BW\BW') < \mathscr{L}^{(n)}(\Bp,\BW)$. This would contradict Theorem~\ref{thm:general_DPI}.
\end{proof}

\begin{theorem}[Data-processing inequality for $\BPhi$]   \label{thm:DPI}
For any square full-rank channels $\BW$ and $\BW'$ of same size, it holds that\footnote{For matrix-valued $\BA$ and $\BB$, an inequality like $\BA \geq \BB$ should be read entrywise. Hence \eqref{Phi_DPI} denotes an array of $K \times K$ simultaneously holding inequalities.}
\begin{equation}   \label{Phi_DPI}
	\BPhi(\BW\BW') \geq \BPhi(\BW).
\end{equation}
with equality\footnote{Equality means that all $K^2$ inequalities are satisfied with equality.} if and only if $\BW'$ is a permutation.
\end{theorem}
\begin{proof}
See Appendix~\ref{app:proof:DPI}.
\end{proof}
Note that the above-referenced proof of Theorem~\ref{thm:DPI} in Appendix~\ref{app:proof:DPI} hinges on Corollary~\ref{cor:DPI_alpha} (of Theorem~\ref{thm:general_DPI}). On the other hand, this Corollary~\ref{cor:DPI_alpha} could clearly also be recovered as a corollary, in fact, of Theorem~\ref{thm:DPI}, since $\alpha_{f\text{-}\mathrm{DIV}}$, $\alpha_{\mathrm{MSE}}$ and $\alpha_{\mathrm{TV}}$ are all monotone functions of the entries of $\BPhi$ [cf.~\eqref{alpha_KL_in_full}--\eqref{alpha_TV_in_full}]. This means that Theorem~\ref{thm:DPI} and Corollary~\ref{cor:DPI_alpha} (to be precise, any one of the three inequalities in Corollary~\ref{cor:DPI_alpha}) are in fact equivalent statements. A self-contained proof of Theorem~\ref{thm:DPI} (which we do not provide) would consist, by contrast, in proving~\eqref{Phi_DPI} based on the assumption of row-stochasticity of $\BW$ and $\BW'$ alone.

It is instructive to particularize~\eqref{Phi_DPI} by setting $\BW=\myid$ and using the fact that $\BPhi(\myid)=\myid$. We obtain $\BPhi(\BW') \geq \myid$ (which holds for any row-stochastic $\BW'$), or equivalently,
\begin{equation}   \label{DPI_particularized}
	\Phi_{k,\ell}(\BW')
	\geq \begin{cases}
		1 & \text{if $k = \ell$} \\
		0 & \text{if $k \neq \ell$.}
	\end{cases}
\end{equation}
This inequality allows to immediately grasp why the normalized metrics $\alpha_{f\text{-}\mathrm{DIV}}$, $\alpha_{\mathrm{MSE}}$ and $\alpha_{\mathrm{TV}}$, as written out in~\eqref{alpha_KL_in_full}--\eqref{alpha_TV_in_full}, are quantities larger or equal to one (cf.~Remark~\ref{rmk3}).

\section{Upper bounds on privacy--fidelity trade-off}   \label{sec:upper_bounds}

Consider the class of circulant mechanisms, which we shall denote as the set $\mathcal{W}_\circ$, and which contains all invertible matrices of the form
\begin{equation}   \label{circulant_mechanism}
	\BW
	=
	\begin{bmatrix}
		w_1 & \hdots & \hdots & w_K \\
		w_K & w_1 & \hdots & w_{K-1} \\
		\vdots & \ddots & \ddots & \vdots \\
		w_2 & \hdots & \hdots & w_1
	\end{bmatrix}.
\end{equation}
Note that circulant mechanisms are fully described by their first row $\Bw = [w_1, \dotsc, w_K] \in \mathbb{P}$. If $w_k/w_{k'} \leq e^\epsilon$ for all $(k,k')$, then this matrix constitutes an $\epsilon$-private mechanism and thus yields an upper (achievable) bound on the privacy--fidelity trade-off curve. We choose this class of mechanisms for producing upper bounds due to them
\begin{enumerate}
	\item	appearing as a natural choice, given that all columns have the same composition and thus the same maximum ratio $\max_{k,k'} w_k/w_{k'}$, thereby satisfying an intuitive notion (though not presently backed by a theorem) that for an optimal mechanism, the maximum intra-column ratio should be equal on all columns.
	\item	yielding simple expressions for the relevant fidelity loss metrics, with the added benefit of making all three normalized metrics $\alpha_{f\text{-}\mathrm{DIV}}$, $\alpha_{\mathrm{MSE}}$ and $\alpha_{f\mathrm{TV}}$ match exactly (so long as the source is uniform), as we shall see in Lemma~\ref{lem:matching_metrics} stated below.
\end{enumerate}
While a universal characterization of the privacy--fidelity trade-off in the context of local differential privacy seems elusive due to the fact that the trade-off critically depends on the specific choice of fidelity metric ($f$-divergence, MSE, TV, etc.), the following lemma, however, legitimizes $\varphi(\BW)$ as a useful quantity in that it ``universally'' captures the privacy-preserving properties of the mechanism $\BW$, and thus may be subjected to optimization.
\begin{lemma}   \label{lem:matching_metrics}
For a uniform source distribution $\Bp = \myone/K$ and a circulant mechanism $\BW \in \mathcal{W}_\circ$, the three normalized fidelity metrics are equal. That is, for every $\BW \in \mathcal{W}_\circ$,
\begin{equation}
	\alpha_{f\text{-}\mathrm{DIV}}\bigl(\tfrac{\myone}{K},\BW\bigr)
	= \alpha_{\mathrm{MSE}}\bigl(\tfrac{\myone}{K},\BW\bigr)
	= \alpha_{\mathrm{TV}}\bigl(\tfrac{\myone}{K},\BW\bigr)
	= \frac{\varphi(\BW)-1}{K-1}.
\end{equation}
\end{lemma}
\begin{proof}
The first equality was already mentioned previously (see~\eqref{alpha_for_uniform}) and holds more generally for~\emph{any} (not necessarily circulant) mechanism $\BW$. The second equality (for the TV metric) can be verified by direct evaluation. For this purpose, it suffices to inspect~\eqref{alpha_TV_in_full} while realizing that for a circulant $\BW$, we have $K\myone\BPhi = \varphi(\BW)\myone$, meaning that $\varphi(\BW)$ becomes the eigenvalue of $\BPhi(\BW)$ (up to a factor $K$) associated to the all-ones eigenvector. Alternatively, we can leverage the fact that all rows and columns of a circulant matrix have the same composition to infer that the Jensen inequality~\eqref{Jensen} becomes tight. The reader is also referred to detailed derivations for circulant mechanisms in the proof of Theorem~\ref{thm:optimal_mechanism}, in Appendix~\ref{app:proof:optimal_mechanism}. 
\end{proof}
The following theorem determines the optimal mechanism within the class of circulant mechanisms, in the sense that it minimizes $\varphi(\BW)$.
\begin{theorem}   \label{thm:optimal_mechanism}
The step mechanism $\BW_{\epsilon,\star}$ as defined in~\eqref{W_star} is optimal (up to row and column permutations) among all circulant $\epsilon$-private mechanisms in terms of minimizing the quantity $\varphi(\BW)$, i.e.,
\begin{equation}   \label{argmin}
	\BW_{\epsilon,\star}
	= \argmin_{\BW \in \mathcal{W}_\epsilon \cap \mathcal{W}_\circ} \varphi(\BW).
\end{equation}
\end{theorem}
\begin{proof}
See Appendix~\ref{app:proof:optimal_mechanism}.
\end{proof}

\begin{figure}
\begin{center}
\includegraphics{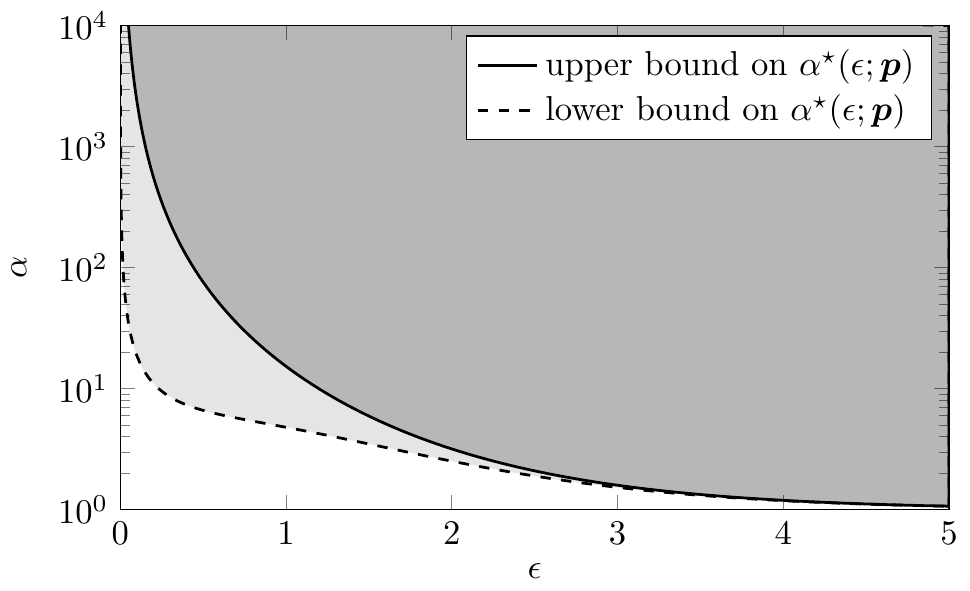}
\end{center}
\caption{$(\epsilon,\alpha)$-trade-off curves for $K=5$ and a uniform source $\Bp = \myone/K$. The feasible set $\mathcal{F}(\myone/K)$ contains the dark shaded region and is contained in the union of both shaded regions. In other terms, the lower boundary of $\mathcal{F}(\myone/K)$, described by $\alpha^\star(\epsilon;\myone/K)$, lies in the light shaded region.}
\label{fig:trade-off_curve}
\end{figure}

By plugging the minimizer $\BW_{\epsilon,\star}$ of the above problem~\eqref{argmin} into the expressions~\eqref{def:Phi}--\eqref{def:phi}, we obtain
\begin{subequations}
\begin{IEEEeqnarray}{rCl}
	\Phi_{k,\ell}(\BW_{\epsilon,\star})
	&=& \frac{1}{(e^\epsilon-1)^2} \Bigl[ \bigl(e^\epsilon(e^\epsilon+K-2)+1-e^\epsilon\bigr) \delta_{k,\ell} + (e^\epsilon+K-2) (1-\delta_{k,\ell}) \Bigr]   \IEEEeqnarraynumspace\IEEEyesnumber\label{Phi_star} \\
	\varphi(\BW_{\epsilon,\star})
	&=& \frac{K}{(e^\epsilon-1)^2} \left[ (e^\epsilon+K-1)(e^\epsilon+K-2) + 1 - e^\epsilon \right].   \IEEEeqnarraynumspace\IEEEyesnumber\label{phi_star}
\end{IEEEeqnarray}
\end{subequations}
The optimality property of the step mechanism elicited by Theorem~\ref{thm:optimal_mechanism} gives additional support to its being widely regarded as a natural choice of RR mechanism in any context. Some publications even use the term ``randomized response'' to refer to the step mechanism itself. Said mechanism can be used to obtain any upper (i.e., ``achievable'') bound on the fundamental privacy--fidelity trade-off curves $\alpha^\star(\epsilon;\Bp)$ (for the feasbility problem) or $\alpha^\star(\epsilon;\mathcal{P})$ (for the minmax problem), simply by inserting~\eqref{Phi_star} or \eqref{phi_star} into~\eqref{feasilibity_alpha_star} or \eqref{minmax_alpha_star}, in combination with loss metric expressions~\eqref{alpha_KL_in_full}--\eqref{alpha_TV_in_full}. For the feasibility problem, with uniform source, for example, using Lemma~\ref{lem:matching_metrics} we obtain, for any of the three fidelity loss metrics,
\begin{equation}   \label{alpha_star_UB}
	\alpha^\star\bigl(\epsilon;\tfrac{\myone}{K}\bigr)
	\leq \frac{\varphi(\BW_{\epsilon,\star})-1}{K-1}
\end{equation}
with $\varphi(\BW_{\epsilon,\star})$ as given in~\eqref{phi_star}.
The latter quantity is plotted as a function of $\epsilon$ in Figure~\ref{fig:trade-off_curve}, along with a corresponding lower bound derived in the next Section.

Furthermore, we believe that $\BW_{\epsilon,\star}$ is the minimizer of $\varphi(\BW)$ not only within the class of circulant mechanisms, but over \emph{all} $\epsilon$-private mechanisms. That is, we conjecture that $\BW_{\epsilon,\star} = \argmin_{\BW \in \mathcal{W}_\epsilon} \varphi(\BW)$. As long as this conjecture remains unproven, we rely on complementing the upper bounds on privacy--fidelity tradeoffs yielded by the step mechanism, with the (generally non-matching) lower bounds developed in the next Section.

\section{Lower bounds on privacy--fidelity trade-off}

The lower bounds presented in this section rely on the following key lemma.
\begin{lemma}   \label{lem:Phi_weighted_sum}
It holds that
\begin{equation}   \label{Phi_sum}
	\min_{\BW \in \mathcal{W}_\epsilon} \varphi(\BW)
	\geq \frac{K}{1-e^{-4\epsilon}} \frac{(e^{\epsilon} + K-1)^2}{e^{2\epsilon} + K-1}.
\end{equation}
\end{lemma}
\begin{proof}
The proof is deferred to Appendix~\ref{app:proof:Phi_weighted_sum}.
\end{proof}
Note that for $\epsilon$ fixed and $K$ tending to infinity, the lower bound on the right-hand side of~\eqref{Phi_sum} behaves as $O(K^2)$, whereas $\varphi(\BW_{\epsilon,\star})$ as given on the right-hand side of~\eqref{phi_star} behaves as $O(K^3)$. Devising a lower bound sharper than~\eqref{Phi_sum} that behaves as $O(K^3)$ is an open problem.

In the next subsections, we will present lower bounds for the feasibility problem and an instance of the minmax problem. In each case, we will need to address the three metrics~\eqref{alpha_DIV}--\eqref{alpha_TV} separately.

\subsection{Lower bounds for the feasibility problem}

In order to leverage Lemma~\ref{lem:Phi_weighted_sum} for lower-bounding the fundamental trade-off curve $\alpha^\star\left(\epsilon;\Bp\right)$, we need to derive lower bounds on $\alpha^\star\left(\epsilon;\Bp\right)$ that depend on $\BW$ only via $\varphi(\BW)$. For the $f$-divergence and MSE metric, such bounds can be obtained via $p_{\minimum}\myone \leq \Bp \leq p_{\maximum}\myone$ and exploiting the fact that the metrics $\alpha_{f\text{-}\mathrm{DIV}}$ and $\alpha_{\mathrm{MSE}}$ are not smaller than one:
\begin{subequations}
\begin{IEEEeqnarray*}{rCl}
	\alpha_{f\text{-}\mathrm{DIV}}^\star\left(\epsilon;\Bp\right)
	&\geq& \min_{\BW \in \mathcal{W}_\epsilon} \frac{ \max\bigr\{K,\tfrac{p_{\minimum}}{p_{\maximum}}\varphi(\BW)\bigr\}-1}{K-1}   \IEEEeqnarraynumspace\IEEEyesnumber\label{alpha_div_feasibility_bound} \\
	\alpha_{\mathrm{MSE}}^\star\left(\epsilon;\Bp\right)
	&\geq& \min_{\BW \in \mathcal{W}_\epsilon} \frac{ \max\bigl\{1,p_{\minimum}\varphi(\BW)\bigr\} - \lVert \Bp \rVert_2^2 }{1 - \lVert \Bp \rVert_2^2}   \IEEEeqnarraynumspace\IEEEyesnumber\label{alpha_mse_feasibility_bound} \\
	\alpha_{\mathrm{TV}}^\star\left(\epsilon;\Bp\right)
	&\geq& \min_{\BW \in \mathcal{W}_\epsilon} \left( \frac{ \sqrt{p_{k_0}(1-p_{k_0}) + (\varphi(\BW) p_{\minimum} - 1)^+} + \sum_{k \in [K] \setminus \{k_0\}} \sqrt{p_k(1-p_k)} }{ \sum_{k \in [K]} \sqrt{p_k(1-p_k)} } \right)^2   \IEEEeqnarraynumspace\IEEEyesnumber\label{alpha_tv_feasibility_bound}
\end{IEEEeqnarray*}
\end{subequations}
where $k_0 = \argmin_{k \in [K]} \bigl| p_k - \frac{1}{2} \bigr|$ and where $(\cdot)^+$ denotes $\max\{\cdot,0\}$.
While~\eqref{alpha_div_feasibility_bound} and \eqref{alpha_mse_feasibility_bound} are straightforward, deriving~\eqref{alpha_tv_feasibility_bound} can be done by lower-bounding the numerator of~\eqref{alpha_TV_in_full} as follows:
\begin{IEEEeqnarray*}{rCl}
	\sqrt{ \Bp \BPhi(\BW) - \Bp \odot \Bp } \, \myone^\Tr
	&\geq& \sqrt{ \max\{\Bp,p_{\minimum}\myone\BPhi(\BW)\} - \Bp \odot \Bp } \\
	&\stackrel{\text{(a)}}{\geq}&  \min_{\substack{\BPsi \geq \myone \colon \\ \BPsi\myone^\Tr = \varphi(\BW)}} \sqrt{ \max\{\Bp,p_{\minimum}\BPsi\} - \Bp \odot \Bp } \\
	&=&
		\begin{cases}
			\sqrt{ \Bp - \Bp \odot \Bp } & \text{if $\varphi(\BW) \leq \frac{1}{p_{\minimum}}$} \\
			\displaystyle\min_{\substack{\bar{\psi}_k \geq 0 \colon \\ \sum_k \psi_k = \varphi(\BW) - \frac{1}{p_{\minimum}}}} \sum_{k \in [K]} \sqrt{ p_k(1-p_k) + p_{\minimum}\bar{\psi}_k } & \text{if $\varphi(\BW) > \frac{1}{p_{\minimum}}$.}
		\end{cases}
\end{IEEEeqnarray*}
For bounding step $\text{(a)}$, we make use of the fact that $\myone\BPhi(\BW) \geq \myone$, which is a consequence of $\BPhi(\BW) \geq \myid$.
Note that~\eqref{alpha_div_feasibility_bound} and \eqref{alpha_mse_feasibility_bound} have the merit of being tight when $\Bp = \myone/K$, whereas~\eqref{alpha_tv_feasibility_bound} is generally not tight even for the uniform source distribution.

Now, by replacing $\varphi(\BW)$ on the right-hand sides of~\eqref{alpha_div_feasibility_bound}--\eqref{alpha_tv_feasibility_bound} with the right-hand side of~\eqref{Phi_sum}, we eventually obtain lower bounds on the fundamental trade-off curves $\alpha^\star(\epsilon;\Bp)$ that only depend on $K$, $\epsilon$ and $\Bp$, but not on the mechanism $\BW$. Although the so-obtained bounds may be loose---especially when $\tfrac{p_{\minimum}}{p_{\maximum}} \ll 1$ or $p_{\minimum} \ll \tfrac{1}{K}$, that is, when $\Bp$ is highly non-uniform---they nonetheless highlight the relevance of the quantity $\varphi(\BW)$ as a proxy for characterizing the privacy--fidelity trade-off. This corroborates the relevance of lower bounds on $\varphi(\BW)$ such as the one given by Lemma~\ref{lem:Phi_weighted_sum}.

\subsection{Lower bounds for the minmax problem}

Recall that the minmax trade-off curve for a continuous set $\mathcal{P} \subseteq \mathbb{P}$ is given by [cf.~\eqref{minmax_alpha_star}]
\begin{equation}   \label{minmax_alpha_star_bis}
	\alpha^\star(\epsilon;\mathcal{P})
	\triangleq \min_{\BW \in \mathcal{W}_\epsilon} \sup_{\Bp \in \mathcal{P}} \alpha(\Bp,\BW).
\end{equation}
In the context of the minmax problem formulation, we shall focus on sets $\mathcal{P}$ that are symmetric in such way that all symbol probabilities are larger or equal to some $p_{0} > 0$. That is, we set\footnote{Note that $p_{0}$ must be smaller than $1/K$ for $\mathcal{P}$ to be non-empty.}
\begin{equation}   \label{min_probability_source_set}
	\mathcal{P}
	= \Bigl\{ (p_1,\dotsc,p_K) \in [p_{0};1]^K \colon \textstyle\sum_{k=1}^K p_k = 1 \Bigr\}.
\end{equation}
Since this set if closed, we can replace the supremum in~\eqref{minmax_alpha_star_bis} with a maximum.
In the next subsections, we will compute lower bounds on $\alpha^\star(\epsilon;\mathcal{P})$ for each of the three loss metrics.

\subsubsection{$f$-divergence metric}

In~\eqref{minmax_alpha_star_bis}, let us set the $f$-divergence metric $\alpha_{f\text{-}\mathrm{DIV}}$ and start by focusing on the inner supremization over $\Bp \in \mathcal{P}$. We need the following lemma.
\begin{lemma}   \label{lem:bivariate_convexity}
Let $\BPi_{\{i,j\}}$ denote the transposition of $i$ and $j$, that is, the elementary permutation matrix that swaps the position of the $i$-th and $j$-th entries. Then the function
\begin{IEEEeqnarray*}{rCl}
	[0;1] &\to& \mathbb{R}_+ \\
	\lambda &\mapsto& \alpha_{f\text{-}\mathrm{DIV}}(\lambda\Bp + (1-\lambda)\Bp\BPi_{\{i,j\}},\BW)
\end{IEEEeqnarray*}
is convex.
\end{lemma}
\begin{proof}
The proof is deferred to Appendix~\ref{app:proof:bivariate_convexity}.
\end{proof}
Using Lemma~\ref{lem:bivariate_convexity}, we can infer that for any fixed $\BW$, the maximizing source distribution
\begin{equation}
	\Bp^\star(\BW)
	= \argmax_{\Bp \in \mathcal{P}} \alpha_{f\text{-}\mathrm{DIV}}(\Bp,\BW)
\end{equation}
has at most one entry different from $p_{0}$. In fact, if it had two entries $p^\star_i$ and $p^\star_j$ both distinct from (larger than) $p_{0}$, then one could argue with Lemma~\ref{lem:bivariate_convexity} that replacing the $(i,j)$-th entry pair of $\Bp^\star$ either by $(p^\star_j + p^\star_i - p_{0},p_{0})$ or by $(p_{0},p^\star_j + p^\star_i - p_{0})$ would yield a larger value of $\alpha_{f\text{-}\mathrm{DIV}}(\Bp^\star,\BW)$, thus leading to a contradiction.
We conclude that $\Bp^\star$ must have $K-1$ entries equal to $p_{0}$ and one entry equal to $1 - (K-1)p_{0}$. Suppose that this entry is denoted by index $k'$, then
\begin{IEEEeqnarray*}{rCl}
	\max_{\Bp \in \mathcal{P}} \Bp\BPhi\Bp^{-\Tr}
	&=& \Phi_{k',k'}(\BW) + \sum_{k \neq k'} \sum_{\ell \neq k'} \Phi_{k,\ell}(\BW) + \frac{p_0}{1 - (K-1)p_0} \sum_{k \neq k'} \Phi_{k,k'}(\BW) \\
	\IEEEeqnarraymulticol{3}{r}{
		{} + \frac{1 - (K-1)p_0}{p_0} \sum_{\ell \neq k'} \Phi_{k',\ell}(\BW)
	} \\
	&\geq& \max\left\{ K , \frac{p_0}{1 - (K-1)p_0} \varphi(\BW) \right\}
\end{IEEEeqnarray*}
where the inequality follows from $\frac{p_0}{1 - (K-1)p_0} \leq 1$ and the fact that $\Bp\BPhi\Bp^{-\Tr} \geq K$ due to $\BPhi \geq \myid$.
Consequently, with Lemma~\ref{lem:Phi_weighted_sum}, a lower bound on $\alpha_{f\text{-}\mathrm{DIV}}^\star(\epsilon;\mathcal{P})$ is given by
\begin{equation}
	\alpha_{f\text{-}\mathrm{DIV}}^\star(\epsilon;\mathcal{P})
	\geq \frac{1}{K-1} \left( \max\left\{ K , \frac{p_0}{1 - (K-1)p_0} \frac{K}{1-e^{-4\epsilon}} \frac{(e^{\epsilon} + K-1)^2}{e^{2\epsilon} + K-1} \right\} - 1 \right).
\end{equation}

\subsubsection{MSE metric}

The MSE metric can be expressed as [cf.~\eqref{alpha_MSE_in_full}]
\begin{equation}
	\alpha_{\mathrm{MSE}}(\Bp,\BW)
	= \tilde{\alpha}(\Bp \BPhi(\BW) \myone^\Tr,\lVert \Bp \rVert_2^2)
\end{equation}
where $\tilde{\alpha}(u,v) \triangleq \frac{u-v}{1-v}$. On the domain $(u,v) \in [1,+\infty) \times [0;1)$, the function $\tilde{\alpha}$ is marginally non-decreasing in $u$ (for fixed $v$) and in $v$ (for fixed $u$). Therefore, one obtains an upper bound on $\max_{\Bp \in \mathcal{P}} \alpha_{\mathrm{MSE}}(\Bp,\BW)$ by separately maximizing both arguments, namely
\begin{equation}   \label{relaxed_worst_case}
	\max_{\Bp \in \mathcal{P}} \alpha_{\mathrm{MSE}}(\Bp,\BW)
	\leq \tilde{\alpha}\left( \max_{\Bp \in \mathcal{P}} \Bp \BPhi(\BW) \myone^\Tr, \max_{\Bp \in \mathcal{P}} \lVert \Bp \rVert_2^2 \right).
\end{equation}
For any fixed $\BW$, the maximizing $\Bp$ in the first argument of the function $\tilde{\alpha}(\cdot,\cdot)$ on the right-hand side of~\eqref{relaxed_worst_case} is
\begin{IEEEeqnarray*}{rCl}
	\Bp^\star(\BW)
	&=& \argmax_{\Bp \in \mathcal{P}} \Bp \BPhi(\BW) \myone^\Tr \\
	&=& p_{0} \myone + (1 - K p_{0}) \Be_{k^\star(\BW)}   \IEEEyesnumber\IEEEeqnarraynumspace\label{worst_case_p}
\end{IEEEeqnarray*}
where $\Be_k = (0,\dotsc,0,1,0,\dotsc,0)$ denotes the $k$-th canonical basis vector (with the $k$-th entry equal to one), and the index $k^\star(\BW)$ corresponds to the column of $\BPhi(\BW)$ with maximum column-sum:
\begin{equation}
	k^\star(\BW)
	= \argmax_{k \in [K]} \Be_k\BPhi(\BW)\myone^\Tr
	= \argmax_{k \in [K]} \sum_{k' \in [K]} \Phi_{k,k'}(\BW).
\end{equation}
As to the second argument of $\tilde{\alpha}(\cdot,\cdot)$ in~\eqref{relaxed_worst_case}, its maximizer is any vector of the form~\eqref{worst_case_p} as well, yet with arbitrary $k^\star$. Hence, $\Bp^\star(\BW)$ as defined in \eqref{worst_case_p} is a common maximizer for both arguments of $\tilde{\alpha}(\cdot,\cdot)$, and consequently, the upper bound~\eqref{relaxed_worst_case} is tight, i.e.,
\begin{equation}
	\max_{\Bp \in \mathcal{P}} \alpha_{\mathrm{MSE}}(\Bp,\BW)
	= \tilde{\alpha}\left( \Bp^\star(\BW) \BPhi(\BW) \myone^\Tr, \lVert \Bp^\star(\BW) \rVert_2^2 \right).
\end{equation}
Here, the two arguments of the function $\tilde{\alpha}(\cdot,\cdot)$ can be evaluated in closed form, respectively, as
\begin{IEEEeqnarray*}{rCl}
	\lVert \Bp^\star(\BW) \rVert_2^2
	&=& (K-1) p_{0}^2 + (1-(K-1)p_{0})^2 \\
	&=& 1 - p_{0} (K-1) (2 - K p_{0}).   \IEEEeqnarraynumspace\IEEEyesnumber\label{minimum_squared_norm}
\end{IEEEeqnarray*}
and
\begin{equation}
	\Bp^\star(\BW) \BPhi(\BW) \myone^\Tr
	= p_{0} \varphi(\BW) + (1 - K p_{0}) \sum_{\ell \in [K]} \Phi_{k^\star(\BW),\ell}(\BW).
\end{equation}
The latter expression still depends on $\BW$, so we further lower-bound it as follows:
\begin{IEEEeqnarray*}{rCl}
	\Bp^\star(\BW) \BPhi(\BW) \myone^\Tr
	&\geq& \min_{\BW \in \mathcal{W}_\epsilon} \Bp^\star(\BW) \BPhi(\BW) \myone^\Tr \\
	&=& \min_{\BW \in \mathcal{W}_\epsilon} \Bigl\{ p_{0} \varphi(\BW) + (1 - K p_{0}) \sum_{\ell \in [K]} \Phi_{k^\star(\BW),\ell}(\BW) \Bigr\} \\
	&\stackrel{\text{(a)}}{\geq}& p_{0} \min_{\BW \in \mathcal{W}_\epsilon} \varphi(\BW) + (1 - K p_{0}) \min_{\BW \in \mathcal{W}_\epsilon} \sum_{\ell \in [K]} \Phi_{k^\star(\BW),\ell}(\BW) \\
	&\stackrel{\text{(b)}}{\geq}& \frac{1}{K} \min_{\BW \in \mathcal{W}_\epsilon} \varphi(\BW) \\
	&\stackrel{\text{(c)}}{\geq}& \frac{1}{1-e^{-4\epsilon}} \frac{(e^{\epsilon} + K-1)^2}{e^{2\epsilon} + K-1}.   \IEEEeqnarraynumspace\IEEEyesnumber\label{weighted_rowsum_LB}
\end{IEEEeqnarray*}
Here, inequality $\text{(a)}$ results from splitting the minimum of a sum into the sum of two minima; step $\text{(b)}$ follows from lower-bounding the maximum column-sum of $\BPhi(\BW)$ by the average column-sum; step $\text{(c)}$ is the application of Lemma~\ref{lem:Phi_weighted_sum}. Combining~\eqref{relaxed_worst_case}, \eqref{worst_case_p} and~\eqref{weighted_rowsum_LB}, we obtain
\begin{equation}
	\alpha_{\mathrm{MSE}}^\star(\epsilon;\mathcal{P})
	\geq \frac{\frac{1}{1-e^{-4\epsilon}} \frac{(e^{\epsilon} + K-1)^2}{e^{2\epsilon} + K-1} - 1 + p_{0} (K-1) (2 - K p_{0})}{p_{0} (K-1) (2 - K p_{0})}.
\end{equation}

\subsubsection{TV metric}

Let us first focus on the inner supremization over $\Bp$. We can lower-bound the supremum as
\begin{equation}
	\sup_{\Bp \in \mathcal{P}} \alpha_{\mathrm{TV}}(\Bp,\BW)
	\geq \left( \frac{ \min_{\Bp \in \mathcal{P}} \sqrt{ \Bp \BPhi(\BW) - \Bp \odot \Bp } \, \myone^\Tr }{ \max_{\Bp \in \mathcal{P}} \sqrt{ \Bp - \Bp \odot \Bp } \, \myone^\Tr } \right)^2.
\end{equation}
The maximization in the denominator is a convex problem and can be solved by a symmetry argument, the maximum being equal to $\sqrt{K-1}$, which is achieved by a uniform distribution vector $\Bp=\myone/K$. As to the numerator, it can be bounded as
\begin{IEEEeqnarray*}{rCl}
	\min_{\Bp \in \mathcal{P}} \sqrt{ \Bp \BPhi(\BW) - \Bp \odot \Bp } \, \myone^\Tr
	&=& \min_{\Bp \in \mathcal{P}} \sum_{k=1}^K \sqrt{ \sum_{\ell=1}^K p_\ell \Phi_{\ell,k}(\BW) - p_k^2 } \\
	&\stackrel{\text{(a)}}{\geq}& \min_{\Bp \in \mathcal{P}} \sum_{k=1}^K \sqrt{ \left( p_{0} \sum_{\ell=1}^K \Phi_{\ell,k}(\BW) - p_k^2 \right)^+ } \\
	&\stackrel{\text{(b)}}{\geq}& \min_{\Bp \in \mathcal{P}} \sqrt{ \left( p_{0} \varphi(\BW) - \lVert \Bp \rVert_2^2 \right)^+ } \\
	&\stackrel{\text{(c)}}{=}& \sqrt{ \left( p_{0} \varphi(\BW) - 1 + p_{0}(K-1)(2-K p_{0}) \right)^+ }.   \IEEEeqnarraynumspace\IEEEyesnumber
\end{IEEEeqnarray*}
Here, the inequality $\text{(a)}$ results from $p_\ell \geq p_0$, whereas $\text{(b)}$ follows from $\sqrt{a^+}+\sqrt{b^+} \geq \sqrt{a^++b^+} \geq \sqrt{(a+b)^+}$. In step $\text{(c)}$ we have reused~\eqref{minimum_squared_norm}: as we have already argued in the previous Subsection, $p_{0} \myone + (1 - K p_{0}) \Be_k$ (with arbitrary $k$) is a maximizer over $\mathcal{P}$ of the Euclidean norm $\lVert \Bp \rVert_2$. All in all, we obtain
\begin{equation}
	\sup_{\Bp \in \mathcal{P}} \alpha_{\mathrm{TV}}(\Bp,\BW)
	\geq \frac{ \left( p_{0} \varphi(\BW) - 1 + p_{0}(K-1)(2-K p_{0}) \right)^+ }{K-1}.
\end{equation}
Hence
\begin{equation}
	\alpha_{\mathrm{TV}}^\star(\epsilon;\mathcal{P})
	\geq \frac{1}{K-1} \left( \frac{K p_0}{1-e^{-4\epsilon}} \frac{(e^{\epsilon} + K-1)^2}{e^{2\epsilon} + K-1} - 1 + p_{0}(K-1)(2-K p_{0}) \right)^+.
\end{equation}

\section{Conclusion}

We have proposed a framework for the study of privacy--fidelity trade-off problems in the context of randomized response mechanisms, in which the privatization channel is supposed to facilitate the estimation of the unknown source distribution while obfuscating the source realizations. A privacy metric based on the concept of local differential privacy, and fidelity loss metrics based on $f$-divergence, MSE and TV distance have been proposed as figures of merit. We have identified the quantities $\BPhi(\BW)$ and $\varphi(\BW)$, which capture the essence of the dependency of fidelity loss metrics on the random mechanism $\BW$, and studied some of its properties, including data-processing inequalities. Finally, we have derived inner and outer bounds to some specific instances of the fundamental privacy--fidelity trade-off curve, all of which depend on the random mechanism via $\BPhi(\BW)$ or $\varphi(\BW)$.

For a better understanding of the fundamental privacy--fidelity trade-off problems, it would be desirable to tighten the gap between inner and outer bounds much further. There is some indication that the lower bounds are loose, so that the step mechanism $\BW_{\epsilon,\star}$ stands as an optimality candidate among all $\epsilon$-private mechanisms (in terms of minimizing $\varphi(\BW)$). A proof or counterexample to this claim is left as an open problem. Other interesting research directions include, for instance, broad channels $(L \geq K)$ (which encompass the RAPPOR mechanism and its generalization by \citet{YeBa18}), sources and/or privatization channels with memory, interactive mechanisms and batch processing, extensions to other types of statistical tests or queries (beyond distribution estimation), and to fidelity loss metrics based on tail probabilities rather than expectations.

\acks{
This work has been supported by the Catalan Government under grant 2017 SGR 1479 and by the Spanish Ministery of Economy and Competitiveness through project TEC2014-59255-C3-1-R (ELISA) and in part by the European ERC Starting Grant 259530-ComCom.}

\appendix

\section{Proof of Lemma~\ref{lem:concentration_inequality}}   \label{app:proof:concentration_inequality}

By Csisz\'ar's concentration inequality \citep[Theorem~1]{Cs84},
\begin{IEEEeqnarray*}{rCl}
	\Prob\bigl\{ \Bt(\By_n) \BW^{-1} \notin \mathbb{P} \bigr\}
	&=& \Prob\bigl\{ \Bt(\By_n) \notin \mathbb{P}\BW \bigr\} \\
	&=& \Prob\bigl\{ \Bt(\By_n) \in \mathbb{P}\setminus\mathbb{P}\BW \bigr\} \\
	&\leq& e^{-n D(\mathbb{P}\setminus\mathbb{P}\BW \Vert \Bp\BW)}.
\end{IEEEeqnarray*}
Since $\Bp\BW$ has only positive entries, we can express $D(\mathbb{P}\setminus\mathbb{P}\BW \Vert \Bp\BW)$ as a minimum over the topological closure of $\mathbb{P}\setminus\mathbb{P}\BW$ (instead of an infimum):
\begin{equation}
	D(\mathbb{P}\setminus\mathbb{P}\BW \Vert \Bp\BW)
	= \min_{\Br' \in \overline{\mathbb{P}\setminus\mathbb{P}\BW}} D(\Br' \Vert \Bp\BW).
\end{equation}
Furthermore, since $\Bp\BW$ belongs to the compact set $\mathbb{P}\BW$, the minimizing distribution is located on its boundary $\partial\mathbb{P}\BW$. In fact, for any distribution $\Br' \in \overline{\mathbb{P}\setminus\mathbb{P}\BW}$, there exists a distribution $\Br''$ on the intersection between the boundary $\partial\mathbb{P}\BW$ and the segment connecting $\Br'$ and $\Bp\BW$, such that $D(\Br'' \Vert \Bp\BW) \leq D(\Br' \Vert \Bp\BW)$. Hence,
\begin{IEEEeqnarray*}{rCl}
	D(\mathbb{P}\setminus\mathbb{P}\BW \Vert \Bp\BW)
	&=& D(\partial\mathbb{P}\BW \Vert \Bp\BW) \\
	&=& \min_{\Br \in \partial\mathbb{P}} D(\Br\BW \Vert \Bp\BW).   \IEEEeqnarraynumspace\IEEEyesnumber\label{information_projection_3}
\end{IEEEeqnarray*}
In addition, observing that
\begin{equation}
	\partial\mathbb{P}
	= \bigl\{ \Bp' \in \mathbb{P} \colon p_i' = 0 \ \text{for some $i$} \bigr\}
\end{equation}
it becomes evident from the last line of~\eqref{information_projection_3} that $D(\partial\mathbb{P}\BW \Vert \Bp\BW) > 0$, since $\Br\BW \neq \Bp\BW$ for all $\Br$ in the compact set $\partial\mathbb{P}$.

\section{Proof of Theorem~\ref{thm:asymptotic_expansion}}   \label{app:proof:asymptotic_expansion}

Let us define the random variables [cf.~\eqref{general_estimator}]
\begin{subequations}
\begin{IEEEeqnarray}{rCcCl}
	\hat{\Bq}_n
	&=& [\hat{q}_{n,1}, \dotsc, \hat{q}_{n,K}]
	&=& \Bt(\By_n)   \IEEEeqnarraynumspace\\
	\check{\Bp}_n
	&=& [\check{p}_{n,1}, \dotsc, \check{p}_{n,K}]
	&=& \Bt(\By_n) \BW^{-1}   \IEEEeqnarraynumspace\\
	\hat{\Bp}_n
	&=& [\hat{p}_{n,1}, \dotsc, \hat{p}_{n,K}]
	&=& \mathrm{Proj}_\mathbb{P}(\Bt(\By_n) \BW^{-1}).   \IEEEeqnarraynumspace
\end{IEEEeqnarray}
\end{subequations}
The histogram $n \hat{\Bq}_n$ is a $K$-variate random variable which follows the multinomial distribution with support set
\begin{equation*}
	\mathcal{S}_n
	\triangleq \bigl\{ (n_1,\dotsc,n_K) \in \mathbb{N}^K \colon n_1 + \dotso + n_K = n \bigr\}
\end{equation*}
(where $\mathbb{N}$ are the non-negative integers) and with a conditional probability mass function given by
\begin{equation}
	\Prob\bigl\{ n \hat{\Bq}_n = (n_1,\dotsc,n_K) \bigr\}
	=
	\begin{cases}
		\frac{n!}{n_1! \dotso n_K!} q_1^{n_1} \dotso q_K^{n_K} & \text{for $(n_1,\dotsc,n_K) \in \mathcal{S}_n$} \\
		0 & \text{for $(n_1,\dotsc,n_K) \notin \mathcal{S}_n$.}
	\end{cases}
\end{equation}
By assumption,
\begin{equation}   \label{Taylor_expansion}
	f(x) = \sum_{\rho=1}^4 \frac{f^{(\rho)}(1)}{\rho!}(x-1)^\rho + \varrho\bigl( \left| x-1 \right|^{4+\gamma} \bigr)
\end{equation}
where the remainder function $\varrho(x)$ satisfies that $\limsup_{x \to 0} |\varrho(x)|/x$ is finite.
By combining the definition of $f$-divergence~\eqref{alpha_DIV} with the above Taylor expansion, we obtain
\begin{equation}
	\Exp\bigl[ D_f\bigl(\hat{\Bp}_n \big\Vert \Bp\bigr) \bigr]
	= \sum_{k=1}^K p_k \Biggl( \sum_{\rho=1}^4 \frac{f^{(\rho)}(1)}{\rho!} \frac{\hat{\mu}_{n,k}^{(\rho)}}{n^\rho p_k^\rho} + R_{n,k} \Biggr)
\end{equation}
where
\begin{subequations}
\begin{IEEEeqnarray}{rCl}
	\hat{\mu}_{n,k}^{(\rho)}
	&\triangleq& n^\rho \Exp\left[ \bigl( \hat{p}_{n,k} - p_k \bigr)^\rho \right] \\
	R_{n,k}
	&\triangleq& \Exp\left[ \varrho\Bigl( \Bigl| \tfrac{\hat{p}_{n,k}}{p_k} - 1 \Bigr|^{4+\gamma} \Bigr) \right].
\end{IEEEeqnarray}
\end{subequations}
Let $\tilde{\Bw}_k = [\tilde{W}_{1,k},\dotsc,\tilde{W}_{K,k}]$ denote the $k$-th column of $\BW^{-1}$ (transposed into a row vector) and let us define the $\rho$-th central moment of $n\langle \tilde{\Bw}_k,\hat{\Bq}_n \rangle$ as
\begin{IEEEeqnarray*}{rCl}
	\check{\mu}_{n,k}^{(\rho)}
	&\triangleq& n^\rho \Exp\left[ \bigl\langle \tilde{\Bw}_k,\hat{\Bq}_n - \Bq \bigr\rangle^\rho \right] \\
	&=& n^\rho \Exp\left[ \bigl( \check{p}_{n,k} - p_k \bigr)^\rho \right].   \IEEEyesnumber\label{def:mu_check}
\end{IEEEeqnarray*}

\begin{lemma}   \label{lem:exponential_decay}
The difference between $\hat{\mu}_{n,k}^{(\rho)}$ and $\check{\mu}_{n,k}^{(\rho)}$ decays at least exponentially in $n$, in that we can upper-bound its absolute value as follows:
\begin{equation}
	\left| \hat{\mu}_{n,k}^{(\rho)} - \check{\mu}_{n,k}^{(\rho)} \right| \leq C n^\rho e^{-n D(\partial\mathbb{P}\BW \Vert \Bp\BW)}
\end{equation}
for some positive constant $C>0$.\footnote{For an explicit bound on $C$, see~\eqref{C_UB}.}
\end{lemma}
\begin{proof}
See Appendix~\ref{app:proof:exponential_decay}.
\end{proof}

Lemma~\ref{lem:exponential_decay} is a consequence of $\hat{\Bp}_n$ and $\check{\Bp}_n$ being equal except in exponentially rare cases. This allows us to exchange $\hat{\mu}_{n,k}^{(\rho)}$ for the easier-to-analyze $\check{\mu}_{n,k}^{(\rho)}$ in the study of asymptotic expansions. Next, we will turn our attention to evaluating the $\rho$-th moment $\check{\mu}_{n,k}^{(\rho)}$.

The multivariate moment-generating function of the multinomially distributed $n(\hat{\Bq}_n-\Bq)$ being
\begin{equation}
	\mathcal{M}(\Blambda)
	= \Exp\left[e^{n\langle\Blambda,\hat{\Bq}_n-\Bq\rangle}\right]
	= \left( \sum_{\ell=1}^K q_\ell e^{\lambda_\ell} \right)^n e^{-n\langle\Blambda,\Bq\rangle}   \label{M_ell}
\end{equation}
with $\Blambda \in \mathbb{R}^K$, it is immediate to obtain the moment-generating function of a weighted sum of the variates of $n(\hat{\Bq}_n-\Bq)$. For a weight vector $\tilde{\Bw}_k$, it suffices to replace $\Blambda$ with $\lambda \tilde{\Bw}_k$ in \eqref{M_ell} to obtain the moment-generating function of $n\langle \tilde{\Bw}_k,\hat{\Bq}_n-\Bq \rangle = n(\check{p}_{n,k}-p_k)$, i.e.,
\begin{IEEEeqnarray*}{rCl}
	\mathcal{M}_{k}(\lambda)
	&\triangleq& \mathcal{M}(\lambda \tilde{\Bw}_k) \\
	&=& \left( \sum_{\ell=1}^K q_\ell e^{\lambda \tilde{W}_{\ell,k}} \right)^n e^{-\lambda n \langle \tilde{\Bw}_k,\Bq \rangle}   \IEEEeqnarraynumspace\IEEEyesnumber\label{MGF_k}
\end{IEEEeqnarray*}
with scalar argument $\lambda \in \mathbb{R}$. The $\rho$-th central moment of $n(\check{p}_{n,k}-p_k)$ can now be expressed as the $\rho$-th derivative at $\lambda = 0$ of the corresponding moment-generating function:
\begin{equation}   \label{moments_from_MGF}
	\check{\mu}_{n,k}^{(\rho)}
	= \mathcal{M}_k^{(\rho)}(0).
\end{equation}
The $\rho$-th derivative of $\mathcal{M}_k$ can be expressed as
\begin{IEEEeqnarray*}{rCl}
	\mathcal{M}_k^{(\rho)}(\lambda)
	&=& \sum_{i=0}^\rho {\rho \choose i} f_k^{(i)}(\lambda) g_k^{(\rho-i)}(\lambda) \\
	&=& f_k(\lambda) g_k^{(\rho)}(\lambda) + \sum_{i=1}^\rho {\rho \choose i} f_k^{(i)}(\lambda) g_k^{(\rho-i)}(\lambda)   \IEEEeqnarraynumspace\IEEEyesnumber \label{M_ell_nu}
\end{IEEEeqnarray*}
where the functions $f_k$ and $g$ are defined as
\begin{subequations}
\begin{IEEEeqnarray}{rCl}
	f_k(\lambda)
	&\triangleq& \left( \sum_{\ell=1}^K q_\ell e^{\lambda \tilde{W}_{\ell,k}} \right)^n \\
	g_k(\lambda)
	&\triangleq& e^{-\lambda n \langle \tilde{\Bw}_k,\Bq \rangle}.
\end{IEEEeqnarray}
\end{subequations}
We now seek to derive an explicit expression for $\mathcal{M}_k^{(\rho)}(\lambda)$. The derivative $g_k^{(\rho-i)}(\lambda)$ can be easily evaluated as
\begin{equation}
	g_k^{(\rho-i)}(\lambda)
	= \left(-n \langle \tilde{\Bw}_k,\Bq \rangle\right)^{\rho-i} g_k(\lambda).
\end{equation}
As to the derivative $f_k^{(i)}$, it can be evaluated for $i \geq 1$ using Fa\`a di Bruno's formula for derivatives of concatenated functions:
\begin{IEEEeqnarray*}{rCl}
	f_k^{(i)}(\lambda)
	&=& \frac{\diffd^i}{\diffd \lambda^i}(u \circ v_k)(\lambda) \\
	&=& \widetilde{\sum} \frac{i!}{m_1! \dotso m_i!} u^{(m_1 + \dotso + m_i)}(v_k(\lambda)) \prod_{j=1}^i \biggl( \frac{v_k^{(j)}(\lambda)}{j!} \biggr)^{m_j}
\end{IEEEeqnarray*}
where the summation $\widetilde{\sum}$, denoted with a tilde, is over all tuples of non-negative $(m_1,\dotsc,m_i) \in \mathbb{N}^i$ satisfying $1 \cdot m_1 + \dotso + i \cdot m_i = i$. The functions $u$ and $v_k$ are respectively defined as
\begin{subequations}
\begin{IEEEeqnarray}{rCl}
	u(\lambda)
	&=& \lambda^n \\
	v_k(\lambda)
	&=& \sum_{\ell=1}^K q_\ell e^{\lambda \tilde{W}_{\ell,k}}
\end{IEEEeqnarray}
\end{subequations}
and have respective derivatives
\begin{subequations}
\begin{IEEEeqnarray}{rCl}
	u^{(j)}(\lambda)
	&=&
	\begin{cases}
		\frac{n!}{(n-j)!}\lambda^{n-j} & \text{if $j \leq n$} \\
		0 & \text{if $j > n$}
	\end{cases} \\
	v_k^{(j)}(\lambda)
	&=& \sum_{\ell=1}^K q_\ell \tilde{W}_{\ell,k}^j e^{\lambda \tilde{W}_{\ell,k}}.
\end{IEEEeqnarray}
\end{subequations}
Noticing that $v_k(0)=1$ and assuming that $n$ is sufficiently large so as to ensure that $m_1+\dotsc+m_i \leq n$ for all values of the sum $m_1+\dotsc+m_i$ taken by summation indices of $\widetilde{\sum}$, we can evaluate the derivative $f_k^{(i)}(0)$ as follows:
\begin{equation}
	f_k^{(i)}(0)
	= \widetilde{\sum} \frac{i!}{m_1! \dotso m_i!} \frac{n!}{(n-m_1 - \dotso - m_i)!} \prod_{j=1}^i \frac{\bigl\langle \tilde{\Bw}_k^{\odot j},\Bq \bigr\rangle^{m_j}}{j!^{m_j}}   \label{f_i_0}
\end{equation}
where the superscript notation $\tilde{\Bw}_k^{\odot j} \triangleq (\tilde{W}_{1,k}^j,\dotsc,\tilde{W}_{K,k}^j)$ denotes entrywise exponentiation. Note that the inner product $\langle \tilde{\Bw}_k^{\odot j},\Bq \rangle$ appearing in the last expression is in fact nothing else than [cf.~\eqref{def:nu}]
\begin{IEEEeqnarray*}{rCl}
	\bigl\langle \tilde{\Bw}_k^{\odot j},\Bq \bigr\rangle
	&=& \Bp\BW \bigl( \underbrace{\BW^{-1} \odot \BW^{-1} \odot \dotso \odot \BW^{-1}}_{\text{$j$ factors}} \bigr) \Be_k^\Tr \\
	&=& \nu_{j,k}.   \IEEEyesnumber
\end{IEEEeqnarray*}
Combining \eqref{f_i_0} with \eqref{M_ell_nu} and noticing that $f_k(0)=g_k(0)=1$, we obtain
\begin{IEEEeqnarray*}{rCl}
	\mathcal{M}_k^{(\rho)}(0)
	&=& (-n \nu_{1,k})^\rho + \sum_{i=1}^\rho {\rho \choose i} f_k^{(i)}(0) (-n \nu_{1,k})^{\rho-i} \\
	&=& (-n \nu_{1,k})^\rho + \sum_{i=1}^\rho {\rho \choose i} \widetilde{\sum} \prod_{j=1}^i \frac{\nu_{j,k}^{m_j}}{j!^{m_j}} \frac{i!n! (-n \nu_{1,k})^{\rho-i}}{m_1! \dotso m_i!(n-m_1 - \dotso - m_i)!}.   \IEEEeqnarraynumspace\IEEEyesnumber\label{M_ell_nu_0}
\end{IEEEeqnarray*}
The second, third and fourth moments of the centered variable $n(\check{p}_{n,k}-p_k)$ can be evaluated from~\eqref{moments_from_MGF} and~\eqref{M_ell_nu_0} as
\begin{subequations}
\begin{IEEEeqnarray}{rCl}
	\check{\mu}_{n,k}^{(2)}
	&=& n \left( \nu_{2,k} - \nu_{1,k}^2 \right)   \label{second_moment_evaluation} \\
	\check{\mu}_{n,k}^{(3)}
	&=& n \left( 2 \nu_{1,k}^3 - 3 \nu_{1,k} \nu_{2,k} + \nu_{3,k} \right) \\
	\check{\mu}_{n,k}^{(4)}
	&=& n \bigl( (3n-6) \nu_{1,k}^4 + 3(n-1) \nu_{2,k}^2 + (12-6n) \nu_{1,k}^2 \nu_{2,k} - 4 \nu_{1,k} \nu_{3,k} + 3 \nu_{4,k} \bigr).   \IEEEeqnarraynumspace
\end{IEEEeqnarray}
\end{subequations}
These explicit evaluations may now be inserted in the Taylor expansion~\eqref{Taylor_expansion}. In the last part, what remains to be proven is that the remainder term of said Taylor expansion satisfies
\begin{equation}   \label{vanishing_remainder}
	\lim_{n \to \infty} n^2 R_{n,k} = 0.
\end{equation}
For this purpose we upper-bound the absolute value of $R_{n,k}$ as follows:
\begin{IEEEeqnarray*}{rCl}
	|R_{n,k}|
	&\leq& \Exp\left[ \Bigl| \varrho\Bigl( \Bigl| \tfrac{\hat{p}_{n,k}}{p_k} - 1 \Bigr|^{4+\gamma} \Bigr) \Bigr| \, \middle| \, \Bigl| \tfrac{\hat{p}_{n,k}}{p_k} - 1 \Bigr| < \delta \right] \\
	&& \quad {} + \Exp\left[ \Bigl| \varrho\Bigl( \Bigl| \tfrac{\hat{p}_{n,k}}{p_k} - 1 \Bigr|^{4+\gamma} \Bigr) \Bigr| \, \middle| \, \Bigl| \tfrac{\hat{p}_{n,k}}{p_k} - 1 \Bigr| \geq \delta \right] \Prob\left\{ \Bigl| \tfrac{\hat{p}_{n,k}}{p_k} - 1 \Bigr| \geq \delta \right\}.   \IEEEeqnarraynumspace\IEEEyesnumber\label{remainder_UB}
\end{IEEEeqnarray*}
This bound results from the convexity of the absolute value (Jensen's inequality), from expanding the expectation using the law of total expectation, and upper-bounding $\Prob\left\{ | \hat{p}_{n,k}/p_k - 1 | < \delta \right\}$ by one. The remaining three terms on the right-hand side of~\eqref{remainder_UB} are upper-bounded in the following.

For the first term, recall that $\varrho(x) = O(x)$ in the vicinity of $x=0$, hence there exists a value $\omega > 0$ such that for any sufficiently small $\delta > 0$, we have $|\varrho(\delta)| \leq \omega \delta$. Consequently, there exists a $\delta_0 > 0$ such that, so long as $\delta \in (0;\delta_0)$,
\begin{equation}
	\Exp\left[ \Bigl| \varrho\Bigl( \Bigl| \tfrac{\hat{p}_{n,k}}{p_k} - 1 \Bigr|^{4+\gamma} \Bigr) \Bigr| \, \middle| \, \Bigl| \tfrac{\hat{p}_{n,k}}{p_k} - 1 \Bigr| < \delta \right]
	\leq \omega \delta^{4+\gamma}.   \label{varrho_UB}
\end{equation}
For the second term in~\eqref{remainder_UB}, recall that $f(0)$ is bounded and $f(x)$ is well-defined for $x>0$, and as a consequence, $\varrho(x)$ must be well-defined for $x \geq 0$. In particular, it follows that $\varrho(|x-1|^{4+\gamma})$ is bounded on the closed interval $x \in [0;1/p_k]$. Therefore, the second term in~\eqref{remainder_UB} is upper-bounded by a constant
\begin{equation}
	\Exp\left[ \Bigl| \varrho\Bigl( \Bigl| \tfrac{\hat{p}_{n,k}}{p_k} - 1 \Bigr|^{4+\gamma} \Bigr) \Bigr| \, \middle| \, \Bigl| \tfrac{\hat{p}_{n,k}}{p_k} - 1 \Bigr| \geq \delta \right]
	\leq \max_{x \in [0;1/p_k]} \bigl\{ \bigl| \varrho\bigl( \left| x-1 \right|^{4+\gamma} \bigr) \bigr| \bigr\}
	\triangleq M   \label{varrho_UB_2}
\end{equation}
where $M$ will serve subsequently as an abbreviative notation.
For the third term in~\eqref{remainder_UB}, let us expand it using the law of total expectation:
\begin{IEEEeqnarray*}{rCl}
	\Prob\left\{ \Bigl| \tfrac{\hat{p}_{n,k}}{p_k} - 1 \Bigr| \geq \delta \right\}
	&=& \Prob\left\{ \Bigl| \tfrac{\check{p}_{n,k}}{p_k} - 1 \Bigr| \geq \delta \middle| \check{\Bp}_n \in \mathbb{P} \right\} \Prob\left\{ \check{\Bp}_n \in \mathbb{P} \right\} \\
	&& {} + \Prob\left\{ \Bigl| \tfrac{\hat{p}_{n,k}}{p_k} - 1 \Bigr| \geq \delta \middle| \check{\Bp}_n \notin \mathbb{P} \right\} \Prob\left\{ \check{\Bp}_n \notin \mathbb{P} \right\} \\
	&=& \Prob\left\{ \Bigl| \tfrac{\check{p}_{n,k}}{p_k} - 1 \Bigr| \geq \delta \right\} + \Bigl[ \Prob\left\{ \Bigl| \tfrac{\hat{p}_{n,k}}{p_k} - 1 \Bigr| \geq \delta \middle| \check{\Bp}_n \notin \mathbb{P} \right\} \\
	&& {} - \Prob\left\{ \Bigl| \tfrac{\check{p}_{n,k}}{p_k} - 1 \Bigr| \geq \delta \middle| \check{\Bp}_n \notin \mathbb{P} \right\} \Bigr] \Prob\left\{ \check{\Bp}_n \notin \mathbb{P} \right\} \\
	&\leq& \Prob\left\{ \Bigl| \tfrac{\check{p}_{n,k}}{p_k} - 1 \Bigr| \geq \delta \right\} + \Prob\left\{ \check{\Bp}_n \notin \mathbb{P} \right\} \\
	&\leq& \Prob\left\{ \Bigl| \tfrac{\check{p}_{n,k}}{p_k} - 1 \Bigr| \geq \delta \right\} + e^{-n D(\partial\mathbb{P}\BW \Vert \Bp\BW)}.   \IEEEeqnarraynumspace\IEEEyesnumber\label{tail_UB}
\end{IEEEeqnarray*}
Here, in the first equality, we have exploited the fact conditioned on $\check{\Bp}_n \in \mathbb{P}$, it holds that $\check{\Bp}_n = \hat{\Bp}_n$ (notice $\check{p}_{n,k}$ in the first probability term on the right-hand side). In the last bounding step, we have used Lemma~\ref{lem:concentration_inequality}.

To tightly bound the deviation probability
\begin{equation*}
	\Prob\left\{ \Bigl| \tfrac{\check{p}_{n,k}}{p_k} - 1 \Bigr| \geq \delta \right\}
	= \Prob\Bigl\{ \check{p}_{n,k} \notin \bigl( p_k(1-\delta);p_k(1+\delta) \bigr) \Bigr\}
\end{equation*}
remaining on the right-hand side of~\eqref{tail_UB}, recall that $\check{p}_{n,k}$ is equal to the inner product $\langle \tilde{\Bw}_k, \hat{\Bq}_n \rangle$. The entries of the type vector $\hat{\Bq}_n$ are jointly multinomially distributed, and their joint cumulant-generating function is given by
\begin{equation}
	\Blambda \mapsto \log\Exp\left[e^{\langle\Blambda,\hat{\Bq}_n\rangle}\right]
	= n\log\left( \sum_{\ell=1}^K q_\ell e^{\lambda_\ell/n} \right)
\end{equation}
with argument $\Blambda = (\lambda_1,\dotsc,\lambda_K) \in \mathbb{R}^K$, and therefore the cumulant-generating function of $\check{p}_{n,k} = \langle \tilde{\Bw}_k, \hat{\Bq}_n \rangle$ is given by
\begin{equation}
	\mathcal{K}_{n,k}(\lambda) \triangleq n\log\left( \sum_{\ell=1}^K q_\ell e^{\lambda \tilde{W}_{\ell,k}/n} \right).
\end{equation}
Since the limit
\begin{equation}
	\mathcal{K}_k(\lambda)
	\triangleq \lim_{n\to\infty} \frac{1}{n}\mathcal{K}_{n,k}(n\lambda)
	= \log\left( \sum_{\ell=1}^K q_\ell e^{\lambda \tilde{W}_{\ell,k}} \right)
\end{equation}
is a well-defined strictly convex function of $\lambda \in \mathbb{R}$, we can define its Legendre-Fenchel transform as
\begin{equation}
	\mathcal{K}_k^\star(x)
	\triangleq \sup_{\lambda \in \mathbb{R}} \left\{ \lambda x - \mathcal{K}_k(\lambda) \right\}
\end{equation}
which is a non-negative, quasi-convex function vanishing at $x = \mathcal{K}'(0) = \langle \tilde{\Bw}_{k'},\Bq \rangle = p_k \in (0;1)$, and apply the G\"artner--Ellis Theorem \citet[Theorem~2.3.6]{DeZe98} to obtain an upper bound on the large-deviation exponent:
\begin{IEEEeqnarray*}{rCl}
	\limsup_{n\to\infty} \frac{1}{n}\log\Prob\left\{ \Bigl| \tfrac{\check{p}_{n,k}}{p_k} - 1 \Bigr| \geq \delta \right\}
	&\leq& -\inf_{x \in (-\infty ; p_k(1-\delta) ] \cup [ p_k(1+\delta) ; \infty)} \mathcal{K}_k^\star(x) \\
	&=& -\min\bigl\{ \mathcal{K}_k^\star(p_k(1-\delta)), \mathcal{K}_k^\star(p_k(1+\delta)) \bigr\} \\
	&\leq& - \min\bigl\{ \lambda p_k(1-\delta) - \mathcal{K}_k(\lambda) , \lambda p_k(1+\delta) - \mathcal{K}_k(\lambda) \bigr\} \\
	&=& \mathcal{K}_k(\lambda) - p_k \bigl( \lambda + \delta |\lambda| \bigr)   \IEEEeqnarraynumspace\IEEEyesnumber\label{exponent_bound}
\end{IEEEeqnarray*}
for an arbitrary $\lambda \in \mathbb{R}$. Since we are free to choose $\lambda$, we may as well restrict $\lambda$ to being non-negative, thus rendering the absolute value $|\lambda|$ equal to $\lambda$. This gives us
\begin{equation}
	\limsup_{n\to\infty} \frac{1}{n}\log\Prob\left\{ \Bigl| \tfrac{\check{p}_{n,k}}{p_k} - 1 \Bigr| \geq \delta \right\}
	\leq \mathcal{K}_k(\lambda) - p_k(1+\delta)\lambda   \qquad \text{(for any $\lambda \geq 0$)}.   \label{exponent_bound_2}
\end{equation}
Given that $\mathcal{K}_k(\lambda)$ is infinitely differentiable, we can replace it in~\eqref{exponent_bound_2} with its second-order Taylor expansion around the origin
\begin{IEEEeqnarray*}{rCl}
	\mathcal{K}_k(\lambda)
	&=& \mathcal{K}_k'(0) \lambda + \frac{\mathcal{K}_k''(0)}{2} \lambda^2 + O(\lambda^3) \\
	&=& p_k \lambda + \frac{1}{2} \left( - p_k^2 + \textstyle\sum_{\ell=1}^K q_\ell \tilde{W}_{\ell,k}^2 \right) \lambda^2 + O(\lambda^3)   \IEEEeqnarraynumspace\IEEEyesnumber
\end{IEEEeqnarray*}
and, noting that $\mathcal{K}_k''(0) > 0$ (due to strict convexity of $\mathcal{K}_k$), we can further proceed with upper-bounding the large-deviation exponent as follows:
\begin{equation*}
	\limsup_{n\to\infty} \frac{1}{n}\log\Prob\left\{ \Bigl| \tfrac{\check{p}_{n,k}}{p_k} - 1 \Bigr| \geq \delta \right\}
	\leq \frac{\mathcal{K}_k''(0)}{2} \lambda^2 - p_k \delta \lambda + O(\lambda^3)
\end{equation*}
for arbitrary $\lambda \geq 0$. Setting $\lambda = p_k \delta / \mathcal{K}_k''(0)$ this bound becomes
\begin{equation*}
	\limsup_{n\to\infty} \frac{1}{n}\log\Prob\left\{ \Bigl| \tfrac{\check{p}_{n,k}}{p_k} - 1 \Bigr| \geq \delta \right\}
	\leq -\frac{p_k^2 \delta^2}{2 \mathcal{K}_k''(0)} + O\biggl( \frac{p_k^3 \delta^3}{\mathcal{K}_k''(0)^3} \biggr)
\end{equation*}
The later implies that there exist constants $C>0$ and $\delta_0' >0$ such that for every $\delta \in (0;\delta_0')$,
\begin{equation}
	\Prob\left\{ \Bigl| \tfrac{\check{p}_{n,k}}{p_k} - 1 \Bigr| \geq \delta \right\}
	\leq C \exp\biggl( -\frac{n p_k^2 \delta^2}{2 \mathcal{K}_k''(0)} \biggr).   \label{exponent_bound_3}
\end{equation}
Upon inserting inequalities~\eqref{varrho_UB}, \eqref{varrho_UB_2} and \eqref{exponent_bound_3} into~\eqref{remainder_UB}, we finally obtain that, for any $\delta \in (0;\min\{\delta_0,\delta_0'\})$,
\begin{equation}
	\left| R_{n,k} \right|
	\leq \omega \delta^{4+\gamma} + M C \exp\biggl( -\frac{n p_k^2 \delta^2}{2 \mathcal{K}_k''(0)} \biggr).
\end{equation}
Multiplying either side with $n^2$, and substituting $\delta$ with $n^{-\frac{1}{2}+\frac{\gamma}{4(4+\gamma)}}$, it holds for all $n$ sufficiently large so as to satisfy $n^{-\frac{1}{2}+\frac{\gamma}{4(4+\gamma)}} < \min\{\delta_0,\delta_0'\}$, we have
\begin{equation}
	n^2 \left| R_{n,k} \right|
	\leq \omega n^{-\frac{\gamma}{4}} + n^2 M C \exp\biggl( -\frac{p_k^2 n^{\frac{\gamma}{2(4+\gamma)}}}{2 \mathcal{K}_k''(0)} \biggr).
\end{equation}
Taking the limit as $n \to \infty$, we conclude~\eqref{vanishing_remainder}. This finishes the proof.

\section{Proof of Lemma~\ref{lem:exponential_decay}}   \label{app:proof:exponential_decay}

Using successively the Jensen inequality and the triangle inequality, we obtain the upper bound
\begin{IEEEeqnarray*}{rCl}
	\left| \hat{\mu}_{n,k}^{(\rho)} - \check{\mu}_{n,k}^{(\rho)} \right|
	&\leq& n^\rho \Exp\left[ \left| \bigl( \hat{p}_{n,k} - p_k \bigr)^{\rho} - \bigl( \check{p}_{n,k} - p_k \bigr)^\rho \right| \right] \\
	&=& n^\rho \Exp\Biggl[ \Biggl| \sum_{r=1}^\rho {\rho \choose r} \Bigl( \check{p}_{n,k} - p_k \Bigr)^{\rho-r} \Bigl( \hat{p}_{n,k} - \check{p}_{n,k} \Bigr)^r \Biggr| \Biggr] \\
	&\leq& n^\rho \Exp\Biggl[ \sum_{r=1}^\rho {\rho \choose r} \Bigl| \check{p}_{n,k} - p_k \Bigr|^{\rho-r} \Bigl| \hat{p}_{n,k} - \check{p}_{n,k} \Bigr|^r \Biggr].
\end{IEEEeqnarray*}
Next, by the Cauchy--Schwarz inequality,
\begin{IEEEeqnarray*}{rCl}
	\Bigl| \check{p}_{n,k} - p_k \Bigr|
	&=& \bigl| \bigl\langle \tilde{\Bw}_k , \hat{\Bq}_n - \Bq \bigr\rangle \bigr| \\
	&\leq& \bigl\lVert \tilde{\Bw}_k \bigr\rVert_2 \bigl\lVert \hat{\Bq}_n - \Bq \bigr\rVert_2 \\
	&\leq& \sqrt{K} \bigl\lVert \tilde{\Bw}_k \bigr\rVert_2
\IEEEeqnarraynumspace\IEEEyesnumber
\end{IEEEeqnarray*}
(where the last inequality is due to $\hat{\Bq}_n$ and $\Bq$ being probability vectors)
so we obtain
\begin{IEEEeqnarray*}{rCl}
	\left| \hat{\mu}_{n,k}^{(\rho)} - \check{\mu}_{n,k}^{(\rho)} \right|
	&\leq& n^\rho \sum_{r=1}^\rho {\rho \choose r} \left( \sqrt{K} \bigl\lVert \tilde{\Bw}_k \bigr\rVert_2 \right)^{\rho-r} \Exp\Bigl[ \bigl| \hat{p}_{n,k} - \check{p}_{n,k} \bigr|^r \Bigr]
\end{IEEEeqnarray*}
wherein the remaining expectation may be expressed as
\begin{equation}   \label{total_probability}
	\Exp\left[\Bigl| \hat{p}_{n,k} - \check{p}_{n,k} \Bigr|^{r}\right]
	= \Exp\left[\Bigl| \hat{p}_{n,k} - \check{p}_{n,k} \Bigr|^{r} \middle| \hat{p}_{n,k} \neq \check{p}_{n,k} \right]
	\Prob\left\{ \hat{p}_{n,k} \neq \check{p}_{n,k} \right\}.
\end{equation}
To upper-bound it, notice that the expectation on the right-hand side of \eqref{total_probability} can be upper-bounded by means of
\begin{IEEEeqnarray*}{rCl}
	\Bigl| \hat{p}_{n,k} - \check{p}_{n,k} \Bigr|
	&=& \Bigl| \hat{p}_{n,k} - \langle \hat{\Bq}_n, \tilde{\Bw}_k \rangle \Bigr| \\
	&\leq& \bigl| \hat{p}_{n,k} \bigr| + \bigl| \langle \hat{\Bq}_n, \tilde{\Bw}_k \rangle \bigr| \\
	&\leq& 1 + \bigl\lVert \hat{\Bq}_n \bigr\rVert_2 \bigl\lVert \tilde{\Bw}_k \bigr\rVert_2 \\
	&\leq& 1 + \bigl\lVert \tilde{\Bw}_k \bigr\rVert_2   \IEEEyesnumber\IEEEeqnarraynumspace\label{hat_check_inequality}
\end{IEEEeqnarray*}
to yield
\begin{equation}
	\left| \hat{\mu}_{n,k}^{(\rho)} - \mu_{n,k}^{(\rho)} \right|
	\leq n^\rho \sum_{r=1}^\rho {\rho \choose r} \left( \sqrt{K} \bigl\lVert \tilde{\Bw}_k \bigr\rVert_2 \right)^{\rho-r} \left(1 + \bigl\lVert \tilde{\Bw}_k \bigr\rVert_2 \right)^r \Prob\left\{ \hat{p}_{n,k} \neq \check{p}_{n,k} \right\}.   \label{tail_upper_bound}
\end{equation}
Finally, to upper-bound the probability $\Prob\bigl\{ \hat{p}_{n,k} \neq \check{p}_{n,k} \bigr\}$, notice that by definition of the projector $\mathrm{Proj}_\mathbb{P}(\cdot)$, for the event $\hat{p}_{n,k} \neq \check{p}_{n,k}$ to occur, it is necessary that $\check{\Bp}_n$ be no probability vector. However, the fact that the rows of $\BW^{-1}$ sum to one ensures that the entries of $\check{\Bp}_n$ will sum to one as well. Therefore $\check{\Bp}_n$ fails to be a probability vector whenever some of its entries lie outside of the unit interval. Hence,
\begin{IEEEeqnarray*}{rCl}
	\Prob\left\{ \hat{p}_{n,k} \neq \check{p}_{n,k} \right\}
	&\leq& \Prob\left\{ \hat{\Bp}_n \neq \check{\Bp}_n \right\} \\
	&\leq& e^{-n D(\partial\mathbb{P}\BW \Vert \Bp\BW)}   \IEEEeqnarraynumspace\IEEEyesnumber\label{projection_probability_UB}
\end{IEEEeqnarray*}
by Lemma~\ref{lem:concentration_inequality}.
Combining~\eqref{tail_upper_bound} and \eqref{projection_probability_UB}, we can conclude that
\begin{equation}
	\left| \hat{\mu}_{n,k}^{(\rho)} - \check{\mu}_{n,k}^{(\rho)} \right|
	\leq C n^\rho e^{-n D(\partial\mathbb{P}\BW \Vert \Bp\BW)}.
\end{equation}
for some positive constant $C>0$ which can be bounded for instance as
\begin{IEEEeqnarray*}{rCl}
	C
	&\leq& \sum_{r=1}^\rho {\rho \choose r} \left( \sqrt{K} \bigl\lVert \tilde{\Bw}_k \bigr\rVert_2 \right)^{\rho-r} \left(1 + \bigl\lVert \tilde{\Bw}_k \bigr\rVert_2 \right)^r \\
	&\leq& \left( 1 + \sqrt{K} \bigl( 1 + \bigl\lVert \tilde{\Bw}_k \bigr\rVert_2 \bigr) \right)^\rho.   \IEEEeqnarraynumspace\IEEEyesnumber\label{C_UB}
\end{IEEEeqnarray*}
This finishes the proof.

\section{Proof of Theorem~\ref{thm:MSE}}   \label{app:proof:MSE}

The MSE loss metric can be expressed as
\begin{IEEEeqnarray*}{rCl}
	\mathscr{L}_{\mathrm{MSE}}^{(n)}(\Bp,\BW)
	&=& \Exp\bigl[ \lVert \hat{\Bp}_n - \Bp \rVert_2^2 \bigr] \\
	&=& \Exp\bigl[ \lVert \hat{\Bp}_n - \check{\Bp}_n + \check{\Bp}_n - \Bp \rVert_2^2 \bigr] \\
	&=& \sum_{k \in [K]} \frac{\check{\mu}_{n,k}^{(2)}}{n^2} + \Exp\bigl[ \lVert \hat{\Bp}_n - \check{\Bp}_n \rVert_2^2 \bigr] + 2\Exp\bigl[ \langle \hat{\Bp}_n - \check{\Bp}_n , \check{\Bp}_n - \Bp \rangle \bigr].
\end{IEEEeqnarray*}
We shall prove that, in the last line, the first term [cf.~\eqref{def:mu_check}]
\begin{equation}
	\frac{1}{n^2} \sum_{k \in [K]} \check{\mu}_{n,k}^{(2)}
	= \frac{1}{n} \sum_{k \in [K]} (\nu_{2,k} - \nu_{1,k}^2)
\end{equation}
is the dominant term (as $n \to \infty$), whereas the two other terms decay exponentially in $n$. Recall from the proof of Lemma~\ref{lem:exponential_decay} that [cf.~\eqref{hat_check_inequality}]
\begin{equation}   \label{p_hat_check_gap}
	\bigl| \hat{p}_{n,k} - \check{p}_{n,k} \bigr|
	\leq 1 + \lVert \tilde{\Bw}_k \rVert_2
\end{equation}
where $\tilde{\Bw}_k$ is the $k$-th column of $\BW^{-1}$. Following the exact same steps as in~\eqref{hat_check_inequality} except for replacing $\hat{p}_{n,k}$ with $p_k$, one can show
\begin{equation}   \label{p_check_gap}
	\bigl| p_k - \check{p}_{n,k} \bigr|
	\leq 1 + \lVert \tilde{\Bw}_k \rVert_2.
\end{equation}
Hence, using~\eqref{p_hat_check_gap},
\begin{IEEEeqnarray*}{rCl}
	\Exp\bigl[ \lVert \hat{\Bp}_n - \check{\Bp}_n \rVert_2^2 \bigr]
	&=& \Exp\bigl[ \lVert \hat{\Bp}_n - \check{\Bp}_n \rVert_2^2 \big| \hat{\Bp}_n \neq \check{\Bp}_n \bigr] \Prob\{ \hat{\Bp}_n \neq \check{\Bp}_n \} \\
	&\leq& \sum_{k \in [K]} \bigl( 1 + \lVert \tilde{\Bw}_k \rVert_2 \bigr)^2 \Prob\{ \hat{\Bp}_n \neq \check{\Bp}_n \}.   \IEEEeqnarraynumspace\IEEEyesnumber
\end{IEEEeqnarray*}
Likewise, using~\eqref{p_hat_check_gap} and \eqref{p_check_gap},
\begin{IEEEeqnarray*}{rCl}
	\bigl| \Exp\bigl[ \langle \hat{\Bp}_n - \check{\Bp}_n , \check{\Bp}_n - \Bp \rangle \bigr] \bigr|
	&=& \bigl| \Exp\bigl[ \langle \hat{\Bp}_n - \check{\Bp}_n , \check{\Bp}_n - \Bp \rangle \big| \hat{\Bp}_n \neq \check{\Bp}_n \bigr] \bigr| \cdot \Prob\{ \hat{\Bp}_n \neq \check{\Bp}_n \} \\
	&\leq& \sum_{k \in [K]} \Exp\Bigl[ \bigl| \hat{p}_{n,k} - \check{p}_{n,k} \bigr| \cdot \bigl| \check{p}_{n,k} - p_k \bigr| \Big| \hat{\Bp}_n \neq \check{\Bp}_n \Bigr] \Prob\{ \hat{\Bp}_n \neq \check{\Bp}_n \} \\
	&\leq& \sum_{k \in [K]} \bigl( 1 + \lVert \tilde{\Bw}_k \rVert_2 \bigr)^2 \Prob\{ \hat{\Bp}_n \neq \check{\Bp}_n \}.   \IEEEeqnarraynumspace\IEEEyesnumber
\end{IEEEeqnarray*}
Invoking Lemma~\ref{lem:concentration_inequality}, we can conclude
\begin{equation*}
	\mathscr{L}_{\mathrm{MSE}}^{(n)}(\Bp,\BW)
	= \frac{1}{n} \sum_{k \in [K]} (\nu_{2,k} - \nu_{1,k}^2) + O\bigl(e^{-n D(\partial\mathbb{P}\BW \Vert \Bp\BW)}\bigr).
\end{equation*}

\section{Proof of Theorem~\ref{thm:TV}}   \label{app:proof:TV}

The TV loss metric is close to $\Exp\bigl[ \lVert \check{\Bp}_n - \Bp \rVert_1 \bigr]$, up to an error term which can be bounded by means of Jensen's inequality and the triangle inequality as
\begin{IEEEeqnarray*}{rCl}
	\Bigl| \mathscr{L}_{\mathrm{TV}}^{(n)}(\Bp,\BW) - \Exp\bigl[ \lVert \check{\Bp}_n - \Bp \rVert_1 \bigr] \Bigr|
	&=& \Bigl| \Exp\bigl[ \lVert \hat{\Bp}_n - \Bp \rVert_1 \bigr] - \Exp\bigl[ \lVert \check{\Bp}_n - \Bp \rVert_1 \bigr] \Bigr| \\
	&\leq& \Exp\Bigl[ \bigl| \lVert \hat{\Bp}_n - \Bp \rVert_1 - \lVert \check{\Bp}_n - \Bp \rVert_1 \bigr| \Bigr] \\
	&\leq& \Exp\bigl[ \lVert \hat{\Bp}_n - \check{\Bp}_n \rVert_1 \bigr] \\
	&=& \Exp\bigl[ \lVert \hat{\Bp}_n - \check{\Bp}_n \rVert_1 \bigm| \hat{\Bp}_n \neq \check{\Bp}_n \bigr] \bigr| \cdot \Prob\{ \hat{\Bp}_n \neq \check{\Bp}_n \}.   \IEEEeqnarraynumspace\IEEEyesnumber
\end{IEEEeqnarray*}
Using the bound $\bigl| \hat{p}_{n,k} - \check{p}_{n,k} \bigr| \leq 1 + \lVert \tilde{\Bw}_k \rVert_2$ from the proof of Lemma~\ref{lem:exponential_decay} [cf.~\eqref{hat_check_inequality}], we obtain
\begin{equation}
	\Bigl| \mathscr{L}_{\mathrm{TV}}^{(n)}(\Bp,\BW) - \Exp\bigl[ \lVert \check{\Bp}_n - \Bp \rVert_1 \bigr] \Bigr|
	\leq \left( K + \textstyle\sum_{k \in [K]} \lVert \tilde{\Bw}_k \rVert_2 \right) \Prob\{ \hat{\Bp}_n \neq \check{\Bp}_n \}.
\end{equation}
Finally, by Lemma~\ref{lem:concentration_inequality}, we conclude that
\begin{equation}
	\mathscr{L}_{\mathrm{TV}}^{(n)}(\Bp,\BW)
	= \Exp\bigl[ \lVert \check{\Bp}_n - \Bp \rVert_1 \bigr] + O\bigl(e^{-n D(\partial\mathbb{P}\BW \Vert \Bp\BW)}\bigr)
\end{equation}
meaning that we can focus on expanding $\Exp\bigl[ \lVert \check{\Bp}_n - \Bp \rVert_1 \bigr]$ instead of $\mathscr{L}_{\mathrm{TV}}^{(n)}(\Bp,\BW)$, since both quantities are exponentially close (hence asymptotically equivalent).
Next, let us define the centered unit-variance random variables
\begin{equation}
	Z_{n,k}
	\triangleq \frac{\check{p}_{n,k} - p_k}{\sqrt{\var(\check{p}_{n,k})}}
	= \frac{\sqrt{n}(\check{p}_{n,k} - p_k)}{\sqrt{\nu_{2,k} - \nu_{1,k}^2}}
\end{equation}
where the second equality holds because~[cf.~\eqref{def:mu_check}]
\begin{equation}
	n^2 \var(\check{p}_{n,k})
	= \check{\mu}_{n,k}^{(2)}
	= n \bigl( \nu_{2,k} - \nu_{1,k}^2 \bigr).
\end{equation}
Recall that the moment-generating function of $n(\check{p}_{n,k}-p_k)$ was introduced in~\eqref{MGF_k}, where it is denoted as $\mathcal{M}_k$. Hence, from~\eqref{MGF_k} we infer that the moment-generating function of $Z_{n,k}$ is
\begin{IEEEeqnarray*}{rCl}
	\Exp\bigl[ e^{\lambda Z_{n,k}} \bigr]
	&=& \mathcal{M}_k\left(\frac{\lambda}{\sqrt{n(\nu_{2,k} - \nu_{1,k}^2)}}\right) \\
	&=& \left( \sum_{\ell=1}^K q_\ell \exp\left( \frac{\lambda \bigl( \tilde{W}_{\ell,k} - \langle \tilde{\Bw}_k,\Bq \rangle \bigr)}{\sqrt{n(\nu_{2,k} - \nu_{1,k}^2)}} \right) \right)^n \\
	&=& \left( 1 - \frac{\lambda^2}{2n} + o(\lambda^2) \right)^n   \IEEEeqnarraynumspace\IEEEyesnumber\label{MGF}
\end{IEEEeqnarray*}
where the last equality follows from Taylor-expanding the exponential function to second order and from noticing that the first and second order terms of said expansion simplify due to
\begin{subequations}
\begin{IEEEeqnarray}{rCl}
	\sum_{\ell=1}^K q_\ell \tilde{W}_{\ell,k} - \langle \tilde{\Bw}_k,\Bq \rangle
	&=& 0 \\
	\sum_{\ell=1}^K q_\ell \left( \frac{\tilde{W}_{\ell,k} - \langle \tilde{\Bw}_k,\Bq \rangle}{\sqrt{\nu_{2,k} - \nu_{1,k}^2}} \right)^2
	&=& 1.
\end{IEEEeqnarray}
\end{subequations}
For the latter equality, recall that $\langle \tilde{\Bw}_k,\Bq \rangle = \sum_{\ell} q_\ell \tilde{W}_{\ell,k} = p_k = \nu_{1,k}$ and that [cf.~\eqref{second_moment_evaluation}]
\begin{equation}
	\sum_{\ell=1}^K q_\ell \left( \tilde{W}_{\ell,k} - \langle \tilde{\Bw}_k,\Bq \rangle \right)^2
	= \nu_{2,k} - \nu_{1,k}^2
	= \frac{\check{\mu}_{n,k}^{(2)}}{n}.
\end{equation}
Since $\lim_{n \to \infty} \left(1+\frac{x}{n}+o(n^{-1})\right)^n = e^x$, we have that for any $\lambda \in \mathbb{R}$, the limit of the moment-generating function of $Z_{n,k}$ as given in~\eqref{MGF} is
\begin{equation}
	\lim_{n \to \infty} \Exp\bigl[ e^{\lambda Z_{n,k}} \bigr]
	= e^{-\lambda^2/2}
\end{equation}
which is the moment-generating function of the normal distribution. By L\'evy's continuity theorem for moment-generating functions~\citep[Problem~30.4]{Bi95}, the pointwise convergence of the moment-generating functions of $Z_{n,k}$ to that of the normal distribution as $n \to \infty$ implies that $Z_{n,k}$ converges in law to a normal random variable $Z \sim \mathcal{N}(0,1)$. Therefore,
\begin{IEEEeqnarray*}{rCl}
	\lim_{n \to \infty}	\sqrt{n} \mathscr{L}_{\mathrm{TV}}^{(n)}(\Bp,\BW)
	&=& \lim_{n \to \infty} \sqrt{n} \sum_{k=1}^K \Exp\bigl[ \left| \check{p}_{n,k} - p_k \right| \bigr] \\
	&=& \sum_{k=1}^K \sqrt{\nu_{2,k} - \nu_{1,k}^2} \lim_{n \to \infty} \Exp\bigl[ \left| Z_{n,k} \right| \bigr] \\
	&=& \sum_{k=1}^K \sqrt{\nu_{2,k} - \nu_{1,k}^2} \Exp\bigl[ |Z| \bigr] \\
	&=& \sum_{k=1}^K \sqrt{\nu_{2,k} - \nu_{1,k}^2} \sqrt{\frac{2}{\pi}}.
\end{IEEEeqnarray*}
This concludes the proof.

\section{Proof of Theorem~\ref{thm:general_DPI}}   \label{app:proof:general_DPI}

Let $\mathscr{D}(\cdot,\cdot)$ denote a distance metric between probability vectors, which shall stand either for the $f$-divergence $D_f(\cdot \Vert \cdot)$ or for the MSE metric $\lVert \cdot - \cdot \rVert_2^2$. These metrics are jointly convex and continuous in both arguments. In particular, they are convex in the first argument, i.e.,
\begin{equation*}
	\mathscr{D}(\lambda\Bp+(1-\lambda)\Bp',\Bp'') \leq \lambda \mathscr{D}(\Bp,\Bp'') + (1-\lambda) \mathscr{D}(\Bp',\Bp'').
\end{equation*}
Denoting the probability simplex as $\mathbb{P}$, let us define
\begin{equation}
	\bar{\mathscr{D}}(\Bp)
	\triangleq \sup_{\Bp' \in \mathbb{P}} \mathscr{D}(\Bp',\Bp)
\end{equation}
which is always finite in the case of the MSE metric, and also finite in the case of the $f$-divergence metric as long as $\lim_{x \downarrow 0} x f(x) < +\infty$, which holds due to the assumption that $f(0)$ is finite [cf.~Theorem~\ref{thm:asymptotic_expansion}].

As usual, assume that $\Bx_n \sim \Bp^{\otimes n}$ is a vector of $n$ i.i.d.\ source samples, and let $\By_n \sim (\Bp\BW)^{\otimes n}$ be the corresponding outputs from $n$ copies of the channel $\BW$. In addition, assume that $\By'_n$ are the outputs from passing $\By_n$ through $n$ copies of another channel $\BW'$. Let us define the estimates
\begin{subequations}
\begin{IEEEeqnarray}{rCl}
	\check{\Bp}_n'
	&\triangleq& \Bt(\By'_n)(\BW\BW')^{-1} \\
	\hat{\Bp}_n'
	&\triangleq& \mathrm{Proj}_\mathbb{P}\bigl(\Bt(\By'_n)(\BW\BW')^{-1}\bigr)
\end{IEEEeqnarray}
\end{subequations}
based on the degraded observation $\By_n'$. In addition, let us define the estimates of the output distribution of the first channel $\BW$ as
\begin{subequations}
\begin{IEEEeqnarray}{rCl}
	\check{\Bq}_n
	&\triangleq& \Bt(\By'_n)(\BW')^{-1} \\
	\hat{\Bq}_n
	&\triangleq& \mathrm{Proj}_\mathbb{P}\bigl(\Bt(\By'_n)(\BW')^{-1}\bigr).
\end{IEEEeqnarray}
\end{subequations}
Using the law of total probability,
\begin{IEEEeqnarray*}{rCl}
	\mathscr{L}^{(n)}(\Bp,\BW\BW')
	&=& \Exp\bigl[\mathscr{D}\bigl(\hat{\Bp}_n',\Bp\bigr)\bigr] \\
	&\geq& \Exp\bigl[\mathscr{D}\bigl(\check{\Bp}_n',\Bp\bigr) \big| \bigl(\check{\Bq}_n,\check{\Bp}_n'\bigr) \in \mathbb{P}^2 \bigr] \Prob\bigl\{ \bigl(\check{\Bq}_n,\check{\Bp}_n'\bigr) \in \mathbb{P}^2 \bigr\}
	\IEEEeqnarraynumspace\IEEEyesnumber\label{degraded_LB}
\end{IEEEeqnarray*}
wherein we have exploited the fact that the event $\check{\Bp}_n' \in \mathbb{P}$ implies $\check{\Bp}_n' = \hat{\Bp}_n'$ by idempotence of the projection.
We can bound the probability factor on the right-hand side of~\eqref{degraded_LB} as
\begin{IEEEeqnarray*}{rCl}
	\Prob\bigl\{ \bigl(\check{\Bq}_n,\check{\Bp}_n'\bigr) \in \mathbb{P}^2 \bigr\}
	&=& \Prob\{ \check{\Bp}_n' \in \mathbb{P} \} - \Prob\{ \check{\Bp}_n' \in \mathbb{P} \bigm| \check{\Bq}_n \notin \mathbb{P} \} \Prob\{ \check{\Bp}_n \notin \mathbb{P} \} \\
	&\geq& \Prob\{ \check{\Bp}_n' \in \mathbb{P} \} - \Prob\{ \check{\Bq}_n \notin \mathbb{P} \} \\
	&\geq& 1 - e^{-n \mathscr{D}(\partial\mathbb{P}\BW\BW' \Vert \Bp\BW\BW')} - e^{-n \mathscr{D}(\partial\mathbb{P}\BW' \Vert \Bp\BW\BW')} \\
	&\geq& 1 - 2 e^{-n M}   \IEEEeqnarraynumspace\IEEEyesnumber\label{no_projection}
\end{IEEEeqnarray*}
by twice applying Lemma~\ref{lem:concentration_inequality}. Here, $M$ denotes the constant
where
\begin{equation}
	M
	= \min\bigl\{ \mathscr{D}(\partial\mathbb{P}\BW\BW' \Vert \Bp\BW\BW') , \mathscr{D}(\partial\mathbb{P}\BW' \Vert \Bp\BW\BW') \bigr\}.
\end{equation}
In order to bound the other factor in~\eqref{degraded_LB}, let us first define $\Btau_n$ and $\Btau_n'$ as being the respective types $\Bt(\By_n)$ and $\Bt(\By_n')$, conditionally on the event
\begin{equation}
	\bigl\{ \bigl(\check{\Bq}_n,\check{\Bp}_n'\bigr) \in \mathbb{P}^2 \bigr\}
	= \{ \Bt(\By_n)\BW^{-1} \in \mathbb{P}, \Bt(\By_n')(\BW\BW')^{-1} \in \mathbb{P} \}.
\end{equation}
The pair $(\Btau_n,\Btau_n')$ thus has a probability mass function
\begin{multline}
	\Prob\{ (\Btau_n,\Btau_n') = (\tilde{\Btau}_n,\tilde{\Btau}_n') \}
	= \Prob\{ (\Bt(\By_n),\Bt(\By_n')) = (\tilde{\Btau}_n,\tilde{\Btau}_n') \} \\
	\times \frac{\mathds{1}\bigl\{ \tilde{\Btau}_n'(\BW')^{-1} \in \mathbb{P}, \tilde{\Btau}_n'(\BW\BW')^{-1} \in \mathbb{P} \bigr\}}{\Prob\bigl\{ \Bt(\By_n')(\BW')^{-1} \in \mathbb{P}, \Bt(\By_n')(\BW\BW')^{-1} \in \mathbb{P} \bigr\}}
\end{multline}
where $\mathds{1}\{\cdot\}$ stands for the indicator function.
Hence, the first factor on the right-hand side of~\eqref{degraded_LB} can be lower-bounded as follows:
\begin{IEEEeqnarray*}{rCl}
	\Exp\bigl[\mathscr{D}\bigl(\check{\Bp}_n',\Bp\bigr) \big| \bigl(\check{\Bq}_n,\check{\Bp}_n'\bigr) \in \mathbb{P}^2 \bigr]
	&=& \Exp\bigl[ \mathscr{D}\bigl(\Btau_n'(\BW\BW')^{-1} , \Bp\bigr) \bigr] \\
	&=& \Exp\bigl[ \Exp\bigl[ \mathscr{D}\bigl(\Btau_n'(\BW\BW')^{-1} , \Bp\bigr) \big| \Btau_n \bigr] \bigr] \\
	&>& \Exp\bigl[ \mathscr{D}\bigl( \Exp[\Btau_n'\big|\Btau_n] (\BW\BW')^{-1} , \Bp\bigr) \bigr].
	\IEEEeqnarraynumspace\IEEEyesnumber\label{DPI_Jensen_lower_bound}
\end{IEEEeqnarray*}
The bounding step is due to the convexity of $\mathscr{D}(\cdot,\cdot)$ in the first argument (Jensen's inequality). The inequality is strict because the metrics $\mathscr{D}(\cdot,\cdot)$ are strictly convex in the first argument by assumption, and because $\Btau_n'$ conditioned on any $\Btau_n$ is non-deterministic.

Next, we show that $\Exp[\Btau_n'|\Btau_n](\BW')^{-1}$ is exponentially close to $\Btau_n$ as $n \to \infty$. By the law of total probability, 
\begin{IEEEeqnarray*}{rCl}
	\tilde{\Btau}_n \BW'
	&=& \Exp\bigl[ \Bt(\By_n') \big| \Bt(\By_n) = \tilde{\Btau}_n \bigr] \\
	&=& \Exp\bigl[ \Bt(\By_n') \big| \Bt(\By_n) = \tilde{\Btau}_n, \check{\Bq}_n \in \mathbb{P} \bigr] \Prob\bigl\{ \check{\Bq}_n \in \mathbb{P} \bigr\} + \Exp[\Bt(\By_n') | \Bt(\By_n) = \tilde{\Btau}_n, \check{\Bq}_n \notin \mathbb{P} ] \Prob\{ \check{\Bq}_n \notin \mathbb{P} \}.
\end{IEEEeqnarray*}
Note that the first expectation in the last line is equal to
\begin{equation}
	\Exp\bigl[ \Bt(\By_n') \big| \Bt(\By_n) = \tilde{\Btau}_n, \check{\Bq}_n \in \mathbb{P} \bigr]
	= \Exp[\Btau_n'|\Btau_n=\tilde{\Btau}_n]
\end{equation}
by definition of $(\Btau_n,\Btau_n')$. Since by Lemma~\eqref{lem:concentration_inequality}, we have
\begin{IEEEeqnarray*}{rCl}
	1 - \Prob\{ \check{\Bq}_n \in \mathbb{P} \}
	&=& \Prob\{ \check{\Bq}_n \notin \mathbb{P} \} \\
	&\leq& e^{-n \mathscr{D}(\partial\mathbb{P}\BW' \Vert \Bp\BW\BW')}
\end{IEEEeqnarray*}
it is straightforward to bound the difference between $\tilde{\Btau}_n \BW'$ and $\Exp[\Btau_n'|\Btau_n=\tilde{\Btau}_n]$ entrywise from below and from above as
\begin{subequations}
\begin{IEEEeqnarray*}{rCl}
	\tilde{\Btau}_n \BW' - \Exp[\Btau_n'|\Btau_n=\tilde{\Btau}_n]
	&\geq& - \Exp[\Btau_n'|\Btau_n=\tilde{\Btau}_n] e^{-n \mathscr{D}(\partial\mathbb{P}\BW' \Vert \Bp\BW\BW')} \\
	&\geq& - \myone e^{-n \mathscr{D}(\partial\mathbb{P}\BW' \Vert \Bp\BW\BW')}   \IEEEeqnarraynumspace\IEEEyesnumber\label{type_diff_LB} \\
	\tilde{\Btau}_n \BW' - \Exp[\Btau_n'|\Btau_n=\tilde{\Btau}_n]
	&\leq& \Exp[\Bt(\By_n') | \Bt(\By_n) = \tilde{\Btau}_n, \check{\Bq}_n \notin \mathbb{P} ] e^{-n \mathscr{D}(\partial\mathbb{P}\BW' \Vert \Bp\BW\BW')} \\
	&\leq& \myone e^{-n \mathscr{D}(\partial\mathbb{P}\BW' \Vert \Bp\BW\BW')}.   \IEEEeqnarraynumspace\IEEEyesnumber\label{type_diff_UB}
\end{IEEEeqnarray*}
\end{subequations}
The inequalities in~\eqref{type_diff_LB}--\eqref{type_diff_UB} hold entrywise in the sense that they stand for $K$ lines of simultaneously holding inequalities.
Combining~\eqref{type_diff_LB} and \eqref{type_diff_UB}, we can bound the Euclidean distance
\begin{equation}
	\bigl\lVert \tilde{\Btau}_n \BW' - \Exp[\Btau_n'|\Btau_n=\tilde{\Btau}_n] \bigr\rVert_2
	\leq \sqrt{K} e^{-n \mathscr{D}(\partial\mathbb{P}\BW' \Vert \Bp\BW\BW')}.
\end{equation}
Taking into account that the metric $\mathscr{D}(\cdot,\cdot)$ is continuous in the first argument, the mapping
\begin{equation}
	\Btau \mapsto \mathscr{D}(\Btau (\BW\BW')^{-1} \Vert \Bp)
	\quad \text{($\Btau \in \mathbb{P}$)}
\end{equation}
is continuous. Being defined on a compact set (the probability simplex), this mapping is also uniformly continuous. Hence, for any $\epsilon > 0$ there exists an $N_\epsilon$ such that for all $n \geq N_\epsilon$, 
\begin{equation}   \label{uniform_continuity}
	\Exp\bigl[ \mathscr{D}\bigl( \Exp[\Btau_n'\big|\Btau_n] (\BW\BW')^{-1} , \Bp\bigr) \bigr]
	\geq \Exp\bigl[ \mathscr{D}\bigl( \Btau_n\BW^{-1} , \Bp\bigr) \bigr] - \epsilon.
\end{equation}
On the other hand,
\begin{IEEEeqnarray*}{rCl}
	\mathscr{L}^{(n)}(\Bp,\BW)
	&=& \Exp\bigl[\mathscr{D}\bigl(\hat{\Bp}_n,\Bp\bigr)\bigr] \\
	&=& \Exp\bigl[\mathscr{D}\bigl(\check{\Bp}_n,\Bp\bigr) \big| \bigl(\check{\Bp}_n,\check{\Bp}_n'\bigr) \in \mathbb{P}^2 \bigr] \Prob\bigl\{ \bigl(\check{\Bp}_n,\check{\Bp}_n'\bigr) \in \mathbb{P}^2 \bigr\} \\
	&& \quad { } + \Exp\bigl[\mathscr{D}\bigl(\hat{\Bp}_n,\Bp\bigr) \big| \bigl(\check{\Bp}_n,\check{\Bp}_n'\bigr) \notin \mathbb{P}^2 \bigr] \Prob\bigl\{ \bigl(\check{\Bp}_n,\check{\Bp}_n'\bigr) \notin \mathbb{P}^2 \bigr\} \\
	&\leq& \Exp\bigl[\mathscr{D}\bigl(\check{\Bp}_n,\Bp\bigr) \big| \bigl(\check{\Bp}_n,\check{\Bp}_n'\bigr) \in \mathbb{P}^2 \bigr] \\
	&& \quad { } + \bar{\mathscr{D}}(\Bp) \bigl(e^{-n \mathscr{D}(\partial\mathbb{P}\BW \Vert \Bp\BW)} + e^{-n \mathscr{D}(\partial\mathbb{P}\BW\BW' \Vert \Bp\BW\BW')}\bigr) \\
	&\leq& \Exp\bigl[\mathscr{D}\bigl(\check{\Bp}_n,\Bp\bigr) \big| \bigl(\check{\Bp}_n,\check{\Bp}_n'\bigr) \in \mathbb{P}^2 \bigr] \\
	&& \quad { } + 2 \bar{\mathscr{D}}(\Bp) e^{-n \mathscr{D}(\partial\mathbb{P}\BW\BW' \Vert \Bp\BW\BW')} \\
	&=& \Exp\bigl[\mathscr{D}\bigl(\Btau_n\BW^{-1},\Bp\bigr) \bigr] + 2 \bar{\mathscr{D}}(\Bp) e^{-n \mathscr{D}(\partial\mathbb{P}\BW\BW' \Vert \Bp\BW\BW')} \\
	&\leq& \Exp\bigl[\mathscr{D}\bigl(\Btau_n\BW^{-1},\Bp\bigr) \bigr] + 2 \bar{\mathscr{D}}(\Bp) e^{-n M}
	\IEEEeqnarraynumspace\IEEEyesnumber\label{UB}
\end{IEEEeqnarray*}
where the first bounding step follows from the union of events, and from twice applying Lemma~\ref{lem:concentration_inequality}, while in the second bounding step we have used the data processing inequality $\mathscr{D}(\cdot \BW\BW' \Vert \Bp\BW\BW') \leq \mathscr{D}(\cdot \BW \Vert \Bp\BW)$.
Combining all inequalities~\eqref{degraded_LB}, \eqref{no_projection}, \eqref{DPI_Jensen_lower_bound}, \eqref{uniform_continuity} and~\eqref{UB}, we obtain
\begin{equation}
	\mathscr{L}^{(n)}(\Bp,\BW\BW')
	> \bigl( \mathscr{L}^{(n)}(\Bp,\BW) - 2 \bar{\mathscr{D}}(\Bp) e^{-n M} - \epsilon \bigr) ( 1 - 2 e^{-n M} )
\end{equation}
for arbitrarily small $\epsilon>0$ and sufficiently large $n$.
It follows that $\mathscr{L}^{(n)}(\Bp,\BW\BW') > \mathscr{L}^{(n)}(\Bp,\BW)$ for sufficiently large $n$, which concludes the proof.

\section{Proof of Theorem~\ref{thm:DPI}}   \label{app:proof:DPI}

Using the identity \eqref{alpha_KL_in_full}, the data-processing inequality \eqref{alpha_DIV_DPI} can be written as
\begin{equation}   \label{alpha_DIV_DPI_in_full}
	\sum_{(k,k') \in [K]^2} \frac{p_k}{p_{k'}} (\Phi_{k,k'}(\BW\BW')-\Phi_{k,k'}(\BW))
	\geq 0.
\end{equation}
Since this inequality holds for any probability vector $\Bp$, it holds in particular for the probability vector taking value $\varepsilon^2$ at the $i$-th coordinate, value $1-(K-2)\varepsilon-\varepsilon^2$ at the $j$-th coordinate, and value $\varepsilon$ on all other coordinates, where $\varepsilon > 0$ is assumed to be sufficiently small to ensure $1-(K-2)\varepsilon-\varepsilon^2 \geq 0$. Expanding the sum on the left-hand side of \eqref{alpha_DIV_DPI_in_full} for this probability vector, while abbreviating the matrix $\BPhi(\BW\BW')-\BPhi(\BW)$ as $\BDelta\BPhi$ for notational concision, we obtain the inequality
\begin{IEEEeqnarray*}{rCl}
	\IEEEeqnarraymulticol{3}{l}{
	\frac{\varepsilon^2}{1-(K-2)\varepsilon-\varepsilon^2} \Delta\Phi_{i,j} + \frac{1-(K-2)\varepsilon-\varepsilon^2}{\varepsilon^2} \Delta\Phi_{j,i}
	} \\ \quad
	&& { } + \varepsilon \sum_{k' \in [K]\setminus\{i,j\}} \Delta\Phi_{i,k'} + \frac{1}{\varepsilon} \sum_{k \in [K]\setminus\{i,j\}} \Delta\Phi_{k,i} \\
	&& { } + \frac{\varepsilon}{1-(K-2)\varepsilon-\varepsilon^2} \sum_{k \in [K]\setminus\{i,j\}} \Delta\Phi_{k,j} \\
	&& { } + \frac{1-(K-2)\varepsilon-\varepsilon^2}{\varepsilon} \sum_{k' \in [K]\setminus\{i,j\}} \Delta\Phi_{j,k'} \\
	&& { } + \sum_{(k,k') \in ([K]\setminus\{i,j\})^2 \cup \{i,i\} \cup \{j,j\}} \Delta\Phi_{k,k'} \geq 0.
	\IEEEeqnarraynumspace\IEEEyesnumber
\end{IEEEeqnarray*}
Multiplying both sides of this inequality by $\varepsilon^2$ and taking the limit as $\varepsilon \to 0$, we obtain the desired result $\Delta\Phi_{j,i} \geq 0$ which concludes the proof.

\section{Proof of Theorem~\ref{thm:optimal_mechanism}}   \label{app:proof:optimal_mechanism}

Let $\Bw = [w_1, \dotsc, w_K]$ denote the first row of a circulant mechanism, as defined in \eqref{circulant_mechanism}, which we assume to be full-rank. Owing to its circulant structure, $\BW$ has an eigendecomposition $\BW = \BF \BLambda \BF^\He$ with a symmetric Fourier eigenbasis $[\BF]_{i,j} = \frac{1}{\sqrt{K}}\xi_K^{(i-1)(j-1)}$ where $\xi_K = e^{-2\pi\jmath/K}$, and a diagonal matrix of eigenvalues $\BLambda = \diag(\Blambda)$. The vector $\Blambda = [\lambda_1,\dotsc,\lambda_K]$ of eigenvalues is the discrete Fourier transform of $\Bw$, i.e.,
\begin{equation}
	\Blambda
	= \Bw\sqrt{K}\BF.
\end{equation}
Note that the row-stochasticity of $\BW$ implies $\lambda_1 = 1$. Furthermore, since $\BW$ is real-valued, the remaining eigenvalues $(\lambda_2,\dotsc,\lambda_K)$ are skew-symmetric, in the sense that $\lambda_k^* = \lambda_{K-k+2}$.
Let $\Blambda \ast \Blambda$ denote the cyclic self-convolution of $\Blambda$, i.e.,
\begin{equation}
	[\Blambda \ast \Blambda]_k
	= \sum_{k'=1}^K \lambda_{k'} \bar{\lambda}_{k+1-k'}
\end{equation}
where $\bar{\lambda}$ denotes the $K$-periodic continuation of $(\lambda_1,\dotsc,\lambda_K)$. Then, by the (self-)convolution theorem of the discrete Fourier transform, for any transform pair $(\tilde{\Bw},\tilde{\Blambda})$, i.e., $\tilde{\Blambda} = \tilde{\Bw}\sqrt{K}\BF$, it holds that
\begin{equation}   \label{self_convolution}
	\frac{1}{K} \tilde{\Blambda} \ast \tilde{\Blambda}
	= (\tilde{\Bw} \odot \tilde{\Bw})\sqrt{K}\BF.
\end{equation}
Hence, for circulant $\BW = \BF\BLambda\BF^\He$, the quantity $\varphi(\BW)$ can be evaluated as follows: 
\begin{IEEEeqnarray*}{rCl}
	\varphi(\BW)
	&=& \myone\BF\BLambda\BF^\He(\BF\BLambda^{-1}\BF^\He \odot \BF\BLambda^{-1}\BF^\He)\myone^\Tr \\
	&\stackrel{\text{(a)}}{=}& \frac{1}{K} \myone\BF\BLambda\diag(\Blambda^{-1} \ast \Blambda^{-1})\BF^\He\myone^\Tr \\
	&\stackrel{\text{(b)}}{=}& \Be_1\BLambda\diag(\Blambda^{-1} \ast \Blambda^{-1})\Be_1^\Tr \\
	&\stackrel{\text{(c)}}{=}&  \Be_1\diag(\Blambda^{-1} \ast \Blambda^{-1})\Be_1^\Tr \\
	&=& \sum_{k=1}^K \frac{1}{\lambda_k \bar{\lambda}_{2-k}} \\
	&=& 1 + \sum_{k=2}^K \frac{1}{\lambda_k \lambda_{K-k+2}} \\
	&=& 1 + \sum_{k=2}^K \frac{1}{|\lambda_k|^2}.   \IEEEyesnumber\label{sum_Phi_as_function_of_eigenvalues}
\end{IEEEeqnarray*}
Here, $\text{(a)}$ results from the self-convolution identity~\eqref{self_convolution}; step $\text{(b)}$ is due to $\BF^\He\BF = \myid$ and $\myone\BF = \sqrt{K}\Be_1$; step $\text{(c)}$ follows from $\Be_1\BLambda = \Be_1$. By the Plancherel Theorem,
\begin{equation}
	1 + \sum_{k=2}^K |\lambda_k|^2
	= K \sum_{k=1}^K w_k^2
\end{equation}
and by the harmonic-arithmetic-mean inequality, we can lower-bound the quantity~\eqref{sum_Phi_as_function_of_eigenvalues} as
\begin{equation}   \label{harmonic_arithmetic_mean_inequality}
	\varphi(\BW)
	\geq 1 + \frac{K^2}{-1 + K \sum_{k=1}^K w_k^2}.
\end{equation}
Note that the denominator in the last expression is positive, hence minimizing the fraction (subject to an $\epsilon$-privacy constraint) amounts to maximizing the sum of squares $\sum_{k=1}^K w_k^2$, which by Appendix~\ref{app:lem:sum_of_square}, Lemma~\ref{lem:sum_of_squares}, is achieved (up to permutation) by a vector
\begin{equation}   \label{optimal_profile}
	\Bw_\star
	= \frac{1}{e^\epsilon + K - 1}
	\begin{bmatrix}
		e^\epsilon & 1 & 1 & \hdots & 1
	\end{bmatrix}.
\end{equation}
It now suffices to show that the harmonic-arithmetic-mean inequality~\eqref{harmonic_arithmetic_mean_inequality} is indeed satisfied with equality for a circulant mechanism generated by $\Bw_\star$. Said inequality is tight for $|\lambda_2| = \dotso = |\lambda_K|$, and the choice~\eqref{optimal_profile} yields eigenvalues
\begin{IEEEeqnarray*}{rCl}
	\lambda_k
	&=& \sum_{\ell=1}^K w_\ell \xi_K^{(k-1)(\ell-1)} \\
	&=& \frac{1}{e^\epsilon+K-1} \left( e^\epsilon + \sum_{\ell=2}^K \xi_K^{(k-1)(\ell-1)} \right) \\
	&=& \frac{e^\epsilon-1}{e^\epsilon+K-1},
	\quad (k=2,\dotsc,K)   \IEEEyesnumber\label{optimal_lambda_k}
\end{IEEEeqnarray*}
since for $k=2,\dotsc,K$, we have
\begin{equation}
	\sum_{\ell=1}^K \xi_K^{(k-1)(\ell-1)}
	= 0.
\end{equation}
Given that the right-hand side of \eqref{optimal_lambda_k} does not depend on $k$, we have indeed $|\lambda_2| = \dotso = |\lambda_K|$, which implies that the harmonic-arithmetic-mean inequality is tight, and thus concludes the proof.

\section{Proof of Lemma~\ref{lem:bivariate_convexity}}   \label{app:proof:bivariate_convexity}

It suffices to prove that the function $a \colon \Bp \mapsto \Bp\BPhi(\BW)\Bp^{-\Tr}$ satisfies that the restriction
\begin{equation}
	[0;1] \to \mathbb{R}_+, \quad \lambda \mapsto a(\lambda\Bp + (1-\lambda)\Bp\BPi_{[i,j]},\BW)
\end{equation}
is convex for any $\Bp$, $\{i,j\}$ and $\BW$.
Singling out two arbitrary variables $p_i$ and $p_j$ out of the coordinates of $\Bp$, we can write $a(\Bp)$ in the following way:
\begin{IEEEeqnarray*}{rCl}
	a(\Bp)
	&=& \sum_{(k,\ell) \in [K]^2} \frac{p_k}{p_\ell} \Phi_{k,\ell}(\BW) \\
	&=& \frac{p_i}{p_j} \Phi_{i,j}(\BW) + \frac{p_j}{p_i} \Phi_{j,i}(\BW) \\
	&& { } + \sum_{\ell \in [K]\setminus\{i,j\}} \frac{p_i}{p_\ell} \Phi_{i,\ell}(\BW)
	+ \sum_{k \in [K]\setminus\{i,j\}} \frac{p_k}{p_j} \Phi_{k,j}(\BW) \\
	&& { } + \sum_{\ell \in [K]\setminus\{i,j\}} \frac{p_j}{p_\ell} \Phi_{j,\ell}(\BW)
	+ \sum_{k \in [K]\setminus\{i,j\}} \frac{p_k}{p_i} \Phi_{k,i}(\BW) \\
	&& { } + \sum_{(k,\ell) \in ([K]\setminus\{i,j\})^2} \frac{p_k}{p_\ell} \Phi_{k,\ell}(\BW).
	\IEEEeqnarraynumspace\IEEEyesnumber\label{alpha_KL_expanded}
\end{IEEEeqnarray*}
Then, $a(\lambda\Bp + (1-\lambda)\Bp\BPi_{[i,j]},\BW)$ can be written out similarly, by replacing all occurrences of $p_i$ and $p_j$ on the right-hand side of~\eqref{alpha_KL_expanded} with $\lambda p_i + (1-\lambda) p_j$ and $\lambda p_j + (1-\lambda) p_i$, respectively.
It is then easy to see that $a(\lambda\Bp + (1-\lambda)\Bp\BPi_{[i,j]},\BW)$ is convex in $\lambda$, since all entries of $\BPhi(\BW)$ are non-negative. In fact, every summand on the right-hand side of~\eqref{alpha_KL_expanded} is (weakly or strongly) convex in $\lambda$.

\section{Auxiliary lemma}   \label{app:lem:sum_of_square}

\begin{lemma}   \label{lem:sum_of_squares}
The solution to the maximization problem
\begin{equation}   \label{x_optimization}
	\Bx_\star
	= \argmax_{\mathclap{\substack{\Bx \in \mathbb{R}_+^K \colon \\ \lVert \Bx \rVert_1 = 1 \\ \forall k,k' \colon x_{k}/x_{k'} \leq e^{\epsilon} }}} \ \lVert \Bx \rVert_2^2
\end{equation}
is given (up to an arbitrary permutation) by
\begin{equation}
	\Bx_\star
	= \frac{1}{e^\epsilon+K-1} \begin{bmatrix}
		e^\epsilon & 1 & \dotsc & 1
	\end{bmatrix}.
\end{equation}
\end{lemma}
\begin{proof}
Let us first ascertain that the optimization domain is convex. Consider any two vectors $\Bx^{(1)}$ and $\Bx^{(2)}$ that belong to said domain. Any convex combination $\lambda \Bx^{(1)} + (1-\lambda) \Bx^{(2)}$ will also belong to this domain, because
\begin{equation}
	\frac{\lambda x_k^{(1)} + (1-\lambda) x_k^{(2)}}{\lambda x_{k'}^{(1)} + (1-\lambda) x_{k'}^{(2)}}
	\leq \max \left\{ \frac{x_k^{(1)}}{x_{k'}^{(1)}} , \frac{x_k^{(2)}}{x_{k'}^{(2)}} \right\}
	\leq e^\epsilon.
\end{equation}
Since problem \eqref{x_optimization} is the maximization of a convex objective over a convex optimization domain, we infer that the maximizer lies on the domain boundary, that is, the privacy constraint is satisfied with equality. This means that we can restrict the privacy constraint to retain only those vectors satisfying
\begin{equation}
	\max_{k,k'} \frac{x_k}{x_{k'}} = e^\epsilon.
\end{equation}
Given this equality constraint and the symmetry (permutation invariance) of the objective function and optimization constraints, we can now conveniently parametrize the optimization domain as follows, without loss of generality:
\begin{equation}   \label{parametrized_X}
	\Bx
	= \frac{\begin{bmatrix} e^\epsilon & 1 & e^{\lambda_3\epsilon} & \dotsc & e^{\lambda_K\epsilon} \end{bmatrix}}{e^{\epsilon} + 1 + e^{\lambda_3\epsilon} + \dotsc + e^{\lambda_K\epsilon}}.
\end{equation}
where $\lambda_3,\dotsc,\lambda_K \in [0;1]^{K-2}$ are the parameters left to optimize. Let us consider any single one of them, with index $k \in \{3,\dotsc,K\}$, and study the maximum of the squared Euclidean norm of \eqref{parametrized_X} as a function of $\lambda_k$, which we define as a function\footnote{The function $f$ depends on the other parameters $\lambda_{k'}$, but we omit this in notation}
\begin{IEEEeqnarray*}{rCl}
	f(e^{\lambda_k\epsilon})
	&\triangleq& \frac{e^{2\epsilon} + 1 + e^{2\lambda_3\epsilon} + \dotsc + e^{2\lambda_K\epsilon}}{\left(e^{\epsilon} + 1 + e^{\lambda_3\epsilon} + \dotsc + e^{\lambda_K\epsilon}\right)^2} \\
	&=& \frac{A + e^{2\lambda_k\epsilon}}{\left(B + e^{\lambda_k\epsilon}\right)^2}   \IEEEyesnumber
\end{IEEEeqnarray*}
where the constants $A$ and $B$ are defined as
\begin{subequations}
\begin{IEEEeqnarray}{rCl}
	A &\triangleq& e^{2\epsilon} + 1 + \sum_{\substack{k'=3\\k'\neq k}}^K e^{2\lambda_{k'}\epsilon} \\
	B &\triangleq& e^{\epsilon} + 1 + \sum_{\substack{k'=3\\k'\neq k}}^K e^{\lambda_{k'}\epsilon}
\end{IEEEeqnarray}
\end{subequations}
and have a ratio $A/B$ upper and lower-bounded as
\begin{equation}
	1
	\leq \frac{A}{B}
	\leq e^{\epsilon}.
\end{equation}
Note that we also have $A \leq B^2$.
The function $f$ is differentiable on $\mathbb{R}^+$ and has a derivative
\begin{equation}
	f'(x)
	= \frac{\diffd}{\diffd x} \left\{ \frac{A + x^2}{(B + x)^2} \right\}
	= \frac{2(Bx-A)}{(B+x)^3}
\end{equation}
which is negative for $x < A/B$ and positive for $x > A/B$. It follows that $f$ is quasi-convex on $\mathbb{R}_+$, and thus in particular on $[1;e^{\epsilon}]$, so its maximum is attained at either one of the boundary points, i.e.,
\begin{equation}
	\max_{1 \leq x \leq e^{\epsilon}} f(x) = \max\{f(1),f(e^\epsilon)\}.
\end{equation}
It follows that the maximizer $\Bx_\star$ of \eqref{x_optimization} can be represented in the parametric form \eqref{parametrized_X} in which some number $\kappa$ of parameters $\lambda_k$ are set to zero, while the remaining $K-2-\kappa$ parameters are set to one. That is, up to a permutation, the optimal $\Bx_\star$ has the form
\begin{equation}
	\Bx_\star(\kappa)
	= \frac{\begin{bmatrix} e^\epsilon & e^\epsilon & \dotsc & e^\epsilon & 1 & 1 & \dotsc & 1 \end{bmatrix}}{(K-\kappa-1) e^{\epsilon} + \kappa + 1}
\end{equation}
where the vector in the numerator contains $K-\kappa-1$ entries equal to $e^\epsilon$ and $\kappa+1$ entries equal to one, and the optimal value of $\kappa \in \{0,\dotsc,K-2\}$ is yet to be determined. Hence, the optimum of~\eqref{x_optimization} is given by
\begin{IEEEeqnarray*}{rCl}
	\quad
	\smash[b]{\max_{\mathclap{\substack{\Bx \in \mathbb{R}_+^K \colon \\ \lVert \Bx \rVert_1 = 1 \\ \forall k,k' \colon \frac{x_k}{x_{k'}} \leq e^{\epsilon}}}}} \ \lVert \Bx \rVert_2^2
	&=& \max_{\kappa \in \{0,\dotsc,K-2\}} \lVert \Bx_\star(\kappa) \rVert_2^2 \\
	&=& \max_{\kappa \in \{0,\dotsc,K-2\}} \frac{\kappa + 1 + (K-\kappa-1) e^{2\epsilon}}{(\kappa + 1 + (K-\kappa-1) e^{\epsilon})^2} \\
	&=& \max_{\kappa' \in \{1,\dotsc,K-1\}} \zeta(\kappa').   \IEEEyesnumber\label{max_zeta}
\end{IEEEeqnarray*}
where the function $\zeta$ is defined as
\begin{equation}
	\zeta(\kappa')
	= \frac{\kappa' + (K-\kappa') e^{2\epsilon}}{(\kappa' + (K-\kappa') e^{\epsilon})^2}.
\end{equation}
By considering the continuous extension of $\zeta$ to the interval $[0;K]$ and studying the sign of its derivative
\begin{equation}
	\frac{\diffd \zeta}{\diffd \kappa'}
	= \frac{(e^\epsilon-1)^2 ((K-\kappa')e^\epsilon - \kappa')}{(\kappa' + (K-\kappa')e^\epsilon)^3}
\end{equation}
we conclude that $\zeta$ is quasi-concave on the interval $[0;K]$ with a maximum located at
\begin{equation}
	\frac{e^\epsilon K}{e^\epsilon + 1} \in (0;K).
\end{equation}
Consequently, the maximum~\eqref{max_zeta} is attained at either $\kappa'=1$ or $\kappa'=K-1$, i.e.,
\begin{IEEEeqnarray*}{rCl}
	\qquad
	\smash[b]{\max_{\mathclap{\substack{\Bx \in \mathbb{R}_+^K \colon \\ \lVert \Bx \rVert_1 = 1 \\ \forall k,k' \colon \frac{x_k}{x_{k'}} \leq e^{\epsilon}}}}} \; \lVert \Bx \rVert_2^2
	&=& \max\left\{ \zeta(1), \zeta(K-1) \right\} \\[2ex]
	&=& \max\left\{ \frac{1 + (K-1) e^{2\epsilon}}{(1 + (K-1) e^{\epsilon})^2}, \frac{K-1 + e^{2\epsilon}}{(K-1 + e^{\epsilon})^2} \right\}.
\end{IEEEeqnarray*}
To compute this maximum of two fractions, consider the cross-product of numerators and denominators, which can be factorized as follows:
\begin{multline}
	\left( K-1 + e^{2\epsilon} \right) \left( 1 + (K-1) e^{\epsilon} \right)^2
	{} - \left( 1 + (K-1) e^{2\epsilon} \right) \left( K-1 + e^{\epsilon} \right)^2 \\
	= (K-1)(K-2)(e^{2\epsilon}-1)(e^\epsilon-1)^2.
\end{multline}
From the right-hand side of the last equality, it appears that this expression is non-negative, hence we conclude that
\begin{IEEEeqnarray*}{rCl}
	\max_{\mathclap{\substack{\Bx \in \mathbb{R}_+^K \colon \\ \lVert \Bx \rVert_1 = 1 \\ \forall k,k' \colon \frac{x_k}{x_{k'}} \leq e^{\epsilon}}}} \ \lVert \Bx \rVert_2^2
	&=& \frac{K-1 + e^{2\epsilon}}{(K-1 + e^{\epsilon})^2}
	= \zeta(K-1)
\end{IEEEeqnarray*}
which, up to an arbitrary permutation, is attained by
\begin{equation}
	\Bx^\star
	= \Bx^\star(K-2)
	= \frac{\begin{bmatrix} e^\epsilon & 1 & 1 & \dotsc & 1 \end{bmatrix}}{e^{\epsilon}+K-1}.
\end{equation}
This concludes the proof.
\end{proof}

\section{Proof of Lemma~\ref{lem:Phi_weighted_sum}}   \label{app:proof:Phi_weighted_sum}

Let the columns of $\BW$ be denoted as (row vectors) $\Bw_1,\dotsc,\Bw_K$ and the rows of its inverse $\BW^{-1}$ be denoted as $\tilde{\Bw}_1,\dotsc,\tilde{\Bw}_K$. By definition of the inverse, we have
\begin{equation}   \label{orthogonality_relations}
	\bigl\langle \Bw_{k'} , \tilde{\Bw}_k \bigr\rangle
	=
	\begin{cases}
		0, & \text{if $k \neq k'$} \\
		1, & \text{if $k = k'$.}
	\end{cases}
\end{equation}
Let
\begin{equation}
	\cos(\Ba,\Bb) = \frac{\langle \Ba , \Bb \rangle}{\lVert \Ba \rVert_2 \lVert \Bb \rVert_2}
\end{equation}
denote the cosine of the angle enclosed by vectors $\Ba$ and $\Bb$. For any orthogonal basis $(\Bb_1,\dotsc,\Bb_K)$ of $\mathbb{R}^K$ and an arbitrary vector $\Ba$, we have the relationship
\begin{equation}
	\sum_{k=1}^K \cos(\Ba,\Bb_k)^2 = 1.
\end{equation}
For any $k \neq k'$, since $\Bw_{k'}$ and $\tilde{\Bw}_k$ are orthogonal by \eqref{orthogonality_relations}, it follows that
\begin{equation}   \label{cosine_inequality}
	\cos(\Bw_k,\Bw_{k'})^2 + \cos(\Bw_k,\tilde{\Bw}_k)^2 \leq 1.
\end{equation}
The first squared cosine in \eqref{cosine_inequality} can be lower-bounded as follows: Due to the privacy constraint $\BW \in \mathcal{W}_{\epsilon}$ and the row-stochasticity of $\BW$, the vectors $\Bw_k$ can be represented in parametric form as
\begin{equation}   \label{parametric_W}
	\Bw_k
	= \frac{\lVert \Bw_k \rVert_2}{\sqrt{\sum_{\ell=1}^K e^{2\epsilon\lambda_{\ell,k}}}}
	\begin{bmatrix}
		e^{\epsilon \lambda_{1,k}} &
		\hdots &
		e^{\epsilon \lambda_{K,k}}
	\end{bmatrix}
\end{equation}
where the coefficients $\lambda_{\ell,k}$ belong to the unit interval $[0;1]$. Any two distinct column vectors $\Bw_k$ and $\Bw_{k'}$ span an angle of cosine at least
\begin{IEEEeqnarray*}{rCl}
	\cos(\Bw_k,\Bw_{k'})
	&=& \frac{\langle \Bw_k,\Bw_{k'} \rangle}{\lVert \Bw_k \rVert_2 \lVert \Bw_{k'} \rVert_2} \\
	&=& \frac{\sum_{\ell=1}^K e^{\epsilon(\lambda_{\ell,k}+\lambda_{\ell,k'})}}{\sqrt{\sum_{\ell=1}^K e^{2\epsilon\lambda_{\ell,k}} \sum_{\ell'=1}^K e^{2\epsilon\lambda_{\ell',k'}}}} \\
	&\geq& e^{-2\epsilon}.   \IEEEeqnarraynumspace\IEEEyesnumber\label{cosine_LB}
\end{IEEEeqnarray*}
Combining \eqref{cosine_inequality} and \eqref{cosine_LB}, we obtain
\begin{IEEEeqnarray*}{rCl}
	\cos\bigl(\Bw_k , \tilde{\Bw}_k\bigr)^2
	&=& \frac{1}{\lVert \Bw_k \rVert_2^2 \lVert \tilde{\Bw}_k \rVert_2^2} \\
	&\leq& 1 - e^{-4\epsilon}.
\end{IEEEeqnarray*}
This inequality allows us to establish the following lower bound:
\begin{IEEEeqnarray*}{rCl}
	\sum_{(k,\ell) \in [K]^2} \Phi_{k,\ell}(\BW)
	&=& \sum_{k=1}^K \sum_{\ell=1}^K \sum_{m=1}^K W_{k,m} \tilde{W}_{m,\ell}^2 \\
	&=& \sum_{m=1}^K \sum_{k=1}^K W_{k,m} \sum_{\ell=1}^K \tilde{W}_{m,\ell}^2 \\
	&=& \sum_{m=1}^K \rVert \Bw_m \rVert_1 \lVert \tilde{\Bw}_m \rVert_2^2 \\
	&\geq& \frac{1}{1-e^{-4\epsilon}} \sum_{m=1}^K \frac{\rVert \Bw_m \rVert_1}{\lVert \Bw_m \rVert_2^2}.   \IEEEeqnarraynumspace\IEEEyesnumber\label{converse_bound_1}
\end{IEEEeqnarray*}
As a final step, when minimizing the right-hand side of \eqref{converse_bound_1} over privatization channels $\BW \in \mathcal{W}_\epsilon$, we will show in the following that the minimum is attained for a step-circulant mechanism
\begin{equation}
	\frac{1}{e^\epsilon + K - 1}
	\begin{bmatrix}
		e^\epsilon & 1 & \hdots & 1 \\
		1 & e^\epsilon & \ddots & \vdots \\
		\vdots & \ddots & \ddots & 1 \\
		1 & \hdots & 1 & e^\epsilon
	\end{bmatrix}.
\end{equation}
Hence,
\begin{equation}
	\min_{\BW \in \mathcal{W}_\epsilon} \sum_{m=1}^K \frac{\rVert \Bw_m \rVert_1}{\lVert \Bw_m \rVert_2^2}
	= K \frac{(e^{\epsilon} + K-1)^2}{e^{2\epsilon} + K-1}   \label{quod_est}
\end{equation}
which upon recombining with \eqref{converse_bound_1} yields the desired statement of Lemma~\ref{lem:Phi_weighted_sum}.

The statement \eqref{quod_est} will be proven in what follows. To begin with, consider the relaxation of problem \eqref{quod_est} that one obtains when minimizing over a superset $\overline{\mathcal{W}}_{\epsilon} \supset \mathcal{W}_\epsilon$ defined as
\begin{equation}
	\overline{\mathcal{W}}_\epsilon
	\triangleq
	\Bigl\{ \BW \in \mathbb{R}_+^{K \times K} \colon
	\textstyle\sum_{k,\ell} W_{k,\ell} = K \ \text{and} \ \forall i,j \colon W_{i,j,j'} \leq e^{\epsilon} W_{i,j'} \Bigr\}.
\end{equation}
In other words, we relax the row-stochasticity constraint while keeping the sum of all entries equal to $K$.
This relaxed problem can be rewritten as follows:
\begin{equation}   \label{min_problem_decomposed}
	\min_{\BW \in \overline{\mathcal{W}}_\epsilon} \sum_{m=1}^K \frac{\lVert \Bw_m \rVert_1}{\lVert \Bw_m \rVert_2^2}
	= \min_{\substack{(N_1,\dotsc,N_K) \in \mathbb{R}_+^K \colon \\ N_1 + \dotsc + N_K = K}} \sum_{m=1}^K \min_{\substack{\Bw_m \in \mathbb{R}_+^K \colon \\ \lVert \Bw_m \rVert_2 = N_m \\ \mathclap{ \forall k,k' \colon W_{k,m}/W_{k',m} \leq e^{\epsilon} }}} \frac{N_m}{\lVert \Bw_m \rVert_2^2}.
\end{equation}
We now need to solve the inner minimization problem (for any given $m$) on the right-hand side of \eqref{min_problem_decomposed}, whose minimizer can be equivalently expressed as the maximizer
\begin{equation}
	\Bw_\star
	= \argmax_{\mathclap{\substack{\Bw \in \mathbb{R}_+^K \colon \\ \lVert \Bw \rVert_1 = N_m \\ \forall k,k' \colon w_{k}/w_{k'} \leq e^{\epsilon} }}} \ \lVert \Bw \rVert_2^2
\end{equation}
which according to Lemma~\ref{lem:sum_of_squares} in Appendix~\ref{app:lem:sum_of_square} admits the solution (up to a permutation)
\begin{equation}
	\Bw_\star
	= N_m \frac{\begin{bmatrix} e^\epsilon & 1 & 1 & \dotsc & 1 \end{bmatrix}}{e^{\epsilon}+K-1}.
\end{equation}
The reciprocal of the squared Euclidean norm of $\Bw_\star$ equals
\begin{equation}
	\frac{1}{\lVert\Bw_\star\rVert_2^2}
	= \frac{1}{N_m^2} \frac{(e^{\epsilon}+K-1)^2}{e^{2\epsilon}+K-1}.
\end{equation}
Hence, the relaxed optimization problem considered further above can be written as
\begin{IEEEeqnarray*}{rCl}
	\min_{\BW \in \overline{\mathcal{W}}_\epsilon} \sum_{m=1}^K \frac{1}{\lVert \Bw_m \rVert_2}
	&=& \min_{\substack{(N_1,\dotsc,N_K) \in \mathbb{R}_+^K \colon \\ N_1 + \dotsc + N_K = K}} \sum_{m=1}^K \frac{1}{N_m} \frac{K-1 + e^{\epsilon}}{\sqrt{K-1 + e^{2\epsilon}}} \\
	&=& K \frac{K-1 + e^{\epsilon}}{\sqrt{K-1 + e^{2\epsilon}}}.
\end{IEEEeqnarray*}
The latter expression is also the minimum of the original optimization problem (before relaxation) and can be achieved by picking $\Bw_m$ to be the rows of the step-circulant matrix
\begin{equation}
	\BW_{\epsilon,\star} = \frac{1}{e^\epsilon+K-1}
	\begin{bmatrix}
		e^\epsilon & 1 & \hdots & 1 \\
		1 & e^\epsilon & \ddots & \vdots \\
		\vdots & \ddots & \ddots & 1 \\
		1 & \hdots & 1 & e^\epsilon
	\end{bmatrix}.
\end{equation}
as defined in \eqref{W_star}, which establishes \eqref{quod_est} and thus finalizes the proof of Lemma~\ref{lem:Phi_weighted_sum}.

\vskip 1.2in
\bibliographystyle{sample}
\bibliography{references}

\begin{thebibliography}{28}
\providecommand{\natexlab}[1]{#1}
\providecommand{\url}[1]{\texttt{#1}}
\expandafter\ifx\csname urlstyle\endcsname\relax
  \providecommand{\doi}[1]{doi: #1}\else
  \providecommand{\doi}{doi: \begingroup \urlstyle{rm}\Url}\fi

\bibitem[Abe(1996)]{Ab96}
Syuuji Abe.
\newblock Expected relative entropy between a finite distribution and its
  empirical distribution.
\newblock \emph{SUT Journal of Mathematics}, 32\penalty0 (2), 1996.

\bibitem[Agrawal and Srikant(2000)]{AgSr00}
Rakesh Agrawal and Ramakrishnan Srikant.
\newblock Privacy-preserving data mining.
\newblock In \emph{Proceedings of the ACM International Conference on
  Management of Data (SIGMOD)}, pages 439--450, New York, NY, USA, 2000.

\bibitem[Billingsley(1995)]{Bi95}
Patrick Billingsley.
\newblock \emph{Probability and Measure}.
\newblock John Wiley and Sons, third edition, 1995.

\bibitem[Csisz\'{a}r(1984)]{Cs84}
Imre Csisz\'{a}r.
\newblock Sanov property, generalized {I}-projection and a conditional limit
  theorem.
\newblock \emph{The Annals of Probability}, 12\penalty0 (3):\penalty0 768--793,
  1984.

\bibitem[Dembo and Zeitouni(1998)]{DeZe98}
Amir Dembo and Ofer Zeitouni.
\newblock \emph{Large Deviations Techniques and Applications}.
\newblock Applications of mathematics. Springer, 1998.

\bibitem[Duchi et~al.(2013)Duchi, Wainwright, and Jordan]{DuWaJo13}
John~C. Duchi, Martin~J. Wainwright, and Michael~I. Jordan.
\newblock Local privacy and minimax bounds: Sharp rates for probability
  estimation.
\newblock In \emph{Advances in Neural Information Processing Systems (NIPS)
  26}, pages 1529--1537. 2013.

\bibitem[Duchi et~al.(2014)Duchi, Jordan, and Wainwright]{DuJoWa14}
John~C. Duchi, Michael~I. Jordan, and Martin~J. Wainwright.
\newblock Privacy aware learning.
\newblock \emph{Journal of the ACM}, 61\penalty0 (6):\penalty0 38:1--38:57,
  December 2014.

\bibitem[Dwork(2006)]{Dw06}
Cynthia Dwork.
\newblock Differential privacy.
\newblock In \emph{Automata, Languages and Programming}, Lecture Notes in
  Computer Science, pages 1--12. Springer, Berlin, Heidelberg, July 2006.

\bibitem[Dwork(2008)]{Dw08}
Cynthia Dwork.
\newblock Differential privacy: A survey of results.
\newblock In \emph{Theory and Applications of Models of Computation}, Lecture
  Notes in Computer Science, pages 1--19. Springer Berlin Heidelberg, April
  2008.

\bibitem[Dwork et~al.(2006)Dwork, McSherry, Nissim, and Smith]{DwMcNiSm06}
Cynthia Dwork, Frank McSherry, Kobbi Nissim, and Adam Smith.
\newblock Calibrating noise to sensitivity in private data analysis.
\newblock In \emph{Theory of Cryptography}, Lecture Notes in Computer Science,
  pages 265--284. Springer, Berlin, Heidelberg, March 2006.

\bibitem[Erlingsson et~al.(2014)Erlingsson, Pihur, and Korolova]{ErPiKo14}
\'{U}lfar Erlingsson, Vasyl Pihur, and Aleksandra Korolova.
\newblock {RAPPOR}: Randomized aggregatable privacy-preserving ordinal
  response.
\newblock In \emph{Proceedings of the ACM SIGSAC Conference on Computer and
  Communications Security}, pages 1054--1067, New York, NY, USA, 2014.

\bibitem[Gaboardi and Rogers(2017)]{GaRo17}
Marco Gaboardi and Ryan Rogers.
\newblock Local private hypothesis testing: Chi-square tests.
\newblock \emph{arXiv:1709.07155}, September 2017.
\newblock URL \url{http://arxiv.org/abs/1709.07155}.
\newblock arXiv: 1709.07155.

\bibitem[Huang et~al.(2017)Huang, Kairouz, Chen, Sankar, and
  Rajagopal]{HuKaChSaRa17}
Chong Huang, Peter Kairouz, Xiao Chen, Lalitha Sankar, and Ram Rajagopal.
\newblock Context-aware generative adversarial privacy.
\newblock \emph{Entropy}, 19\penalty0 (12):\penalty0 656, December 2017.

\bibitem[Ji et~al.(2014)Ji, Lipton, and Elkan]{JiLiEl14}
Zhanglong Ji, Zachary~C. Lipton, and Charles Elkan.
\newblock Differential privacy and machine learning: a survey and review.
\newblock \emph{arXiv:1412.7584}, December 2014.
\newblock URL \url{http://arxiv.org/abs/1412.7584}.

\bibitem[Kairouz et~al.(2014)Kairouz, Oh, and Viswanath]{KaOhVi14}
Peter Kairouz, Sewoong Oh, and Pramod Viswanath.
\newblock Extremal mechanisms for local differential privacy.
\newblock In \emph{Advances in Neural Information Processing Systems (NIPS)
  27}, pages 2879--2887. 2014.

\bibitem[Kairouz et~al.(2015)Kairouz, Oh, and Viswanath]{KaOhVi15}
Peter Kairouz, Sewoong Oh, and Pramod Viswanath.
\newblock Secure multi-party differential privacy.
\newblock In \emph{Advances in Neural Information Processing Systems (NIPS)
  28}, pages 2008--2016. 2015.

\bibitem[Kairouz et~al.(2016{\natexlab{a}})Kairouz, Bonawitz, and
  Ramage]{KaBoRa16}
Peter Kairouz, Keith Bonawitz, and Daniel Ramage.
\newblock Discrete distribution estimation under local privacy.
\newblock In \emph{Proceedings of the 33rd International Conference on Machine
  Learning}, volume~48, pages 2436--2444, New York, USA, June
  2016{\natexlab{a}}.

\bibitem[Kairouz et~al.(2016{\natexlab{b}})Kairouz, Bonawitz, and
  Ramage]{KaBoRa16arxiv}
Peter Kairouz, Keith Bonawitz, and Daniel Ramage.
\newblock Discrete distribution estimation under local privacy.
\newblock \emph{arXiv:1602.07387}, February 2016{\natexlab{b}}.
\newblock URL \url{http://arxiv.org/abs/1602.07387}.

\bibitem[Kasiviswanathan et~al.(2011)Kasiviswanathan, Lee, Nissim,
  Raskhodnikova, and Smith]{KaLeNiRaSm11}
Shiva~Prasad Kasiviswanathan, Homin~K. Lee, Kobbi Nissim, Sofya Raskhodnikova,
  and Adam Smith.
\newblock What can we learn privately?
\newblock \emph{SIAM Journal on Computing}, 40\penalty0 (3):\penalty0 793--826,
  January 2011.

\bibitem[Leoni(2012)]{Le12}
David Leoni.
\newblock Non-interactive differential privacy: A survey.
\newblock In \emph{Proceedings of the First International Workshop on Open Data
  (WOD)}, pages 40--52, New York, NY, USA, 2012. ACM.

\bibitem[Sarwate and Sankar(2014)]{SaSa14}
Anand~D. Sarwate and Lalitha Sankar.
\newblock A rate-distortion perspective on local differential privacy.
\newblock In \emph{52nd Annual Allerton Conference on Communication, Control,
  and Computing}, pages 903--908, September 2014.

\bibitem[Wainwright et~al.(2012)Wainwright, Jordan, and Duchi]{DuJoWa12}
Martin~J. Wainwright, Michael~I. Jordan, and John~C. Duchi.
\newblock Privacy aware learning.
\newblock In \emph{Advances in Neural Information Processing Systems (NIPS)
  25}, pages 1430--1438. 2012.

\bibitem[Wang et~al.(2016)Wang, Huang, Wang, Nie, Xu, Yang, Li, and
  Qiao]{WaHuWaNiXuYaLiQi16}
Shaowei Wang, Liusheng Huang, Pengzhan Wang, Yiwen Nie, Hongli Xu, Wei Yang,
  Xiang-Yang Li, and Chunming Qiao.
\newblock Mutual information optimally local private discrete distribution
  estimation.
\newblock \emph{arXiv:1607.08025}, July 2016.
\newblock URL \url{http://arxiv.org/abs/1607.08025}.

\bibitem[Warner(1965)]{Wa65}
Stanley~L. Warner.
\newblock Randomized response: A survey technique for eliminating evasive
  answer bias.
\newblock \emph{Journal of the American Statistical Association}, 60\penalty0
  (309):\penalty0 63--69, 1965.

\bibitem[Wasserman and Zhou(2010)]{WaZh10}
Larry Wasserman and Shuheng Zhou.
\newblock A statistical framework for differential privacy.
\newblock \emph{Journal of the American Statistical Association}, 105\penalty0
  (489):\penalty0 375--389, 2010.

\bibitem[Ye and Barg(2017)]{YeBa17}
Min Ye and Alexander Barg.
\newblock Asymptotically optimal private estimation under mean square loss.
\newblock \emph{arXiv:1708.00059}, July 2017.
\newblock URL \url{http://arxiv.org/abs/1708.00059}.

\bibitem[Ye and Barg(2018{\natexlab{a}})]{YeBa18}
Min Ye and Alexander Barg.
\newblock Optimal schemes for discrete distribution estimation under locally
  differential privacy.
\newblock \emph{IEEE Transactions on Information Theory}, 64\penalty0
  (8):\penalty0 5662--5676, August 2018{\natexlab{a}}.

\bibitem[Ye and Barg(2018{\natexlab{b}})]{YeBa18b}
Min Ye and Alexander Barg.
\newblock Optimal locally private estimation under $\ell_p$ loss for $1\le p\le
  2$.
\newblock \emph{arXiv:1810.07283}, October 2018{\natexlab{b}}.
\newblock URL \url{http://arxiv.org/abs/1810.07283}.

\end{thebibliography}

\end{document}